\documentclass{article}
\PassOptionsToPackage{numbers}{natbib}

\usepackage[preprint]{neurips_2026}

\usepackage[utf8]{inputenc} 
\usepackage[T1]{fontenc}    
\usepackage{hyperref}       
\usepackage{url}            
\usepackage{booktabs}       
\usepackage{amsfonts}       
\usepackage{nicefrac}       
\usepackage{microtype}      
\usepackage{xcolor}         
\usepackage{amsmath,amssymb,amsthm}
\usepackage{graphicx}
\usepackage[capitalise,nameinlink]{cleveref}
\usepackage{listings} 
\usepackage{algorithm}
\usepackage{placeins}
\usepackage[noend]{algpseudocode}
\definecolor{codegreen}{rgb}{0,0.6,0}
\definecolor{codegray}{rgb}{0.5,0.5,0.5}
\definecolor{codepurple}{rgb}{0.58,0,0.82}
\definecolor{backcolour}{rgb}{0.95,0.95,0.92}
\lstdefinestyle{mystyle}{
    language=Python,
    backgroundcolor=\color{backcolour},   
    commentstyle=\color{codegreen},
    keywordstyle=\color{magenta},
    numberstyle=\tiny\color{codegray},
    stringstyle=\color{codepurple},
    basicstyle=\ttfamily\footnotesize,
    breakatwhitespace=false,         
    breaklines=true,                 
    captionpos=b,                    
    keepspaces=true,                 
    numbers=left,                    
    numbersep=5pt,                  
    showspaces=false,                
    showstringspaces=false,
    showtabs=false,                  
    tabsize=2
}
\newtheorem{theorem}{Theorem}[section]

\newtheorem{proposition}[theorem]{Proposition}
\newtheorem{conjecture}[theorem]{Conjecture}
\newtheorem{corollary}[theorem]{Corollary}
\theoremstyle{definition}

\newtheorem{example}[theorem]{Example}
\theoremstyle{remark}
\newtheorem{remark}[theorem]{Remark}
\title{Double-Scoring: Reliable Extraction of Strong Lottery Tickets}

\author{%
  Bryce A. Christopherson \\
  Department of Mathematics \& Statistics\\
  University of North Dakota\\
  Grand Forks, ND 58202-8376 \\
  \texttt{bryce.christopherson@UND.edu} \\
  \And
  Jack Baretz \\
  Department of Mathematics \& Statistics\\
  University of North Dakota\\
  Grand Forks, ND 58202-8376 \\
  \texttt{jack.baretz@UND.edu} \\
  \AND
  Darian Colgrove \\
  Department of Mathematics \& Statistics\\
  University of North Dakota\\
  Grand Forks, ND 58202-8376 \\
  \texttt{darian.colgrove@UND.edu} \\
  \And
  Salah Dandan \\
  Department of Mathematics \& Statistics\\
  University of North Dakota\\
  Grand Forks, ND 58202-8376 \\
  \texttt{salah.dandan@UND.edu} \\
}

\begin{document}

\maketitle

\begin{abstract}
  The lottery ticket hypothesis proposes that large random neural networks contain sparse subnetworks that can match the performance of dense models after comparable training. A stronger version asserts that sufficiently overparameterized random networks contain subnetworks that are already accurate before any weight training. Existing theory establishes that such strong lottery tickets exist, but reliable extraction remains difficult. We revisit \texttt{edge-popup}, a frozen-weight score-training method for extracting strong tickets, and identify layerwise sparsity selection as a central bottleneck. We introduce \texttt{double-scoring}, an augmented score-space parameterization that replaces a layerwise sparsity search with optimization over enlarged score tensors. We prove that fixed-density masking in an augmented score space preserves access to all original-coordinate masks, and we show that the resulting method can be interpreted as \texttt{edge-popup} on a zero-augmented network. In controlled experiments, \texttt{double-scoring} substantially improves strong-ticket extraction over fixed-density \texttt{edge-popup} and pruning-at-initialization baselines, improves on the performance of rewound sparse-training topologies, and exhibits markedly lower sensitivity to sparsity hyperparameters. Ablations show that the gain is not merely due to additional trainable score parameters, but is tied to the augmented score-space competition that induces the effective original sparsity.
\end{abstract}

\section{Introduction}

    In 2019, Frankle and Carbin introduced the lottery ticket hypothesis \cite{lotterytickethypothesis}, which posits that dense, randomly initialized neural networks contain sparse subnetworks (``winning tickets'') that, when trained in isolation, can match the performance of the original network. Shortly thereafter, Ramanujan et al.~\cite{ramanujan} proposed a stronger variant: sufficiently overparameterized random networks should contain subnetworks that already approximate a target network \emph{without any training}. These subnetworks are referred to as \emph{strong lottery tickets}, whereas the subnetworks proposed in \cite{lotterytickethypothesis} are referred to as \textit{weak lottery tickets}.
    
    This strong form was subsequently rigorously established in several works, precisely quantifying the necessary overparameterization. Malach et al.~\cite{pruningisallyouneed} proved existence under polynomial overparameterization, and later results \cite{orseau2020logarithmic,pensia2020optimal} significantly improved these bounds, showing that under appropriate weight distributions, only logarithmic growth in width is required. These results collectively demonstrate that the existence of strong lottery tickets is no longer in doubt.  Thus, the central difficulty is not existence, but \emph{extraction}. That is, given a randomly initialized network known to contain a strong lottery ticket, how can one efficiently identify it?
    
    A variety of pruning and masking strategies have been proposed for this purpose.  Some methods, including SNIP~\cite{lee2018snip}, GraSP~\cite{wang2020picking}, and SynFlow~\cite{tanaka2020pruning}, identify sparse subnetworks at initialization using saliency criteria.  These methods are computationally attractive because they produce masks before full training. Other sparse training methods such as Iterative Magnitude Pruning (IMP) use repeated train-prune-rewind cycles to identify sparse subnetworks that could train effectively from an early or original initialization \cite{pmlr-v119-frankle20a}.  Likewise, SET~\cite{mocanu2018scalable}, RigL~\cite{evci2020rigging}, and movement pruning~\cite{sanh2020movement} learn sparse topologies while training weights.  These methods are not native strong-ticket extraction algorithms, since their masks are discovered during weight optimization.  Nevertheless, they provide useful comparison points: after learning a topology, one can rewind the weights to the original initialization and test whether the topology itself defines a strong ticket. 

    In this work, we revisit one of the earliest methods for strong ticket extraction: the \texttt{edge-popup} algorithm of Ramanujan et al.~\cite{ramanujan}, which freezes random weights and trains auxiliary scores whose top-ranked entries determine the active mask.  We argue that the primary obstacle in \texttt{edge-popup} is not its optimization dynamics, but rather its \emph{parameterization of sparsity}. The algorithm requires choosing layerwise sparsity levels in advance, yet existing theory provides no principled way to choose them.

    We show that this difficulty can be removed entirely. First, we prove that for sufficiently sparse layers, the full class of masked subnetworks can be represented using a fixed mask density of $\frac{1}{2}$. We then extend this idea to dense networks by doubling the number of trainable score parameters---that is, by enlarging the \emph{score space} rather than the weight space. This leads to a simple modification of \texttt{edge-popup}, which we call \texttt{double-scoring}, that eliminates the need to tune layerwise sparsity parameters while preserving expressive power. Empirically, this modification allows us to reliably extract strong lottery tickets using a procedure with essentially the same computational cost as standard training. Moreover, the resulting subnetworks appear to be competitive with other weak lottery ticket extraction methods after training as well.

\paragraph{Contributions.}
        \begin{enumerate}
            \item We identify layerwise sparsity selection as a practical bottleneck in \texttt{edge-popup}-style strong ticket extraction.
            \item We introduce \texttt{double-scoring}, an augmented score-space parameterization that removes the need to tune a separate sparsity level for each layer.
            \item We prove that augmented score-space masking preserves representational access to all original-coordinate masks and is equivalent to \texttt{edge-popup} on a zero-augmented network under straight-through gradients.
            \item In controlled FashionMNIST MLP experiments, \texttt{double-scoring} substantially improves strong-ticket extraction relative to fixed-density \texttt{edge-popup}, random masks, SNIP, and GraSP, and remains competitive with rewound sparse-training topologies.
            \item We provide stability and ablation studies showing that the gain is tied to augmented score-space competition rather than simply to adding trainable score parameters.
        \end{enumerate}

\section{Background and problem setup}
\label{background and problem setup section}

    We consider feed-forward neural networks with $\ell$ layers, widths $(n_0,\dots,n_\ell)$, and activation functions $(\sigma_1,\dots,\sigma_{\ell-1})$. Such a network is a function $f_W:\mathbb{R}^{n_0}\to\mathbb{R}^{n_\ell}$ of the form
    $$
    f_W(x)=W_\ell \sigma_{\ell-1}\big(\cdots \sigma_1(W_1 x + b_1)\cdots\big) + b_\ell,
    $$
    where $W_i \in \mathbb{R}^{n_i \times n_{i-1}}$ and $b_i \in \mathbb{R}^{n_i}$ for each $i$.
    
    The lottery ticket hypothesis \cite{lotterytickethypothesis} asserts that, for any $\delta \in (0,1)$, there exists $N$ such that, with probability at least $\delta$, a randomly initialized feed-forward network with $\min \{n_i : 1 \leq i \leq \ell-1\} \geq N$ contains a mask $H$ such that the subnetwork $f_{W \odot H}$, where $H$ remains fixed, can achieve accuracy comparable to that of the original network after each undergoes comparable weight training.  The strong lottery ticket hypothesis \cite{ramanujan} is effectively the same, but dispenses with the subsequent training requirement and instead asserts the mask $H$ satisfies $\|f_{W \odot H} - g\|_{K,\infty} < \epsilon$ for any chosen continuous function $g$ on a compact subset $K$ of the domain.  Although a strong ticket is already performant at initialization, whether such a subnetwork also remains trainable as a weak ticket is an empirical question we examine in \Cref{sparse training and weak tickets subsection} and again in \Cref{extra capacity section}.

    In the same paper in which Ramanujan et al.~initially posed the strong lottery ticket hypothesis, they also introduced \texttt{edge-popup}, a frozen-weight score-training method.  For a layer with fixed weights \(W_t\), score tensor \(S_t\), and density \(k_t\), define \(H_t=\operatorname{TopKMask}(S_t;k_t)\) to retain the \(\lfloor k_t d_t\rfloor\) entries of largest \(|S_t|\), where \(d_t\) is the number of weights.  The forward pass uses \(W_t\odot H_t\), while the backward pass uses a straight-through estimator for the hard mask; gradients through the magnitude parameterization use the derivative of \(|S_t|\) away from zero.  Thus, \texttt{edge-popup} optimizes masks while keeping weights fixed.  
    
    The difficulty in using \texttt{edge-popup} to extract strong lottery tickets from a network is that the algorithm assumes the correct choice of layerwise densities $k_1,\ldots,k_\ell$ are already known or somehow otherwise obtainable. This is the bottleneck that prevents \texttt{edge-popup} from extracting strong lottery tickets.
    
\section{The sparsity selection bottleneck}
\label{the sparsity selection bottleneck section}
    
    Existence theorems for strong lottery tickets imply only that some good choice of layerwise densities $k_1,\ldots,k_\ell$ exists (i.e. the densities of the masks for a strong ticket itself); they do not identify that choice. In practice, different layerwise density patterns can produce dramatically different performance, and there is no robust rule that predicts the correct tuple in advance.  Worse still, the correct choices provided by existence theorems are only viable for score initializations sufficiently near the desired target.  In many cases, it appears unlikely for \textit{any} choice of densities to produce a strong lottery ticket with \texttt{edge-popup} from an unfortunate score initialization, since density selection and score initialization interact nontrivially, creating sensitivities not only to global sparsity but to the full layerwise density vector and the score initialization.  Since the correct vector is unknown, extracting a true strong ticket generally requires an expensive grid search over $k$-vectors and often some degree of luck.  This sensitivity is illustrated in a toy sine-regression experiment in Appendix~\ref{proofs appendix}, where exhaustive sweeps over layerwise densities produce highly nonuniform loss landscapes.

    One can make this bottleneck explicit even at the level of counting. The $i$th weight matrix has $n_{i-1}n_i$ entries, and hence there are only $n_{i-1}n_i+1$ possible density values of the form \(0,\frac{1}{n_{i-1}n_i},\frac{2}{n_{i-1}n_i},\dots,\frac{n_{i-1}n_i-1}{n_{i-1}n_i},1\).  Thus, if one wished to determine with certainty which layerwise density tuple is optimal, then, in principle, one would have to consider all combinations of these values across layers.  Without bias terms, the total number of possible layerwise density tuples is $\prod_{i=1}^{\ell}(n_{i-1}n_i+1)$.  Some of these tuples are obviously redundant—for instance, if $k_i=0$, then the later values $k_{i+1},\dots,k_\ell$ become irrelevant—but the search space still grows extremely rapidly.  For example, even a three-layer network with weight matrices of shapes $10\times 20$, $20\times 30$, and $30\times 5$ (only \(950\) weights in total) already has over \(18\) million possible layerwise density tuples.

\section{Double-scoring method}
\label{double-scoring method section}

    The \texttt{double-scoring} method (\Cref{double-scoring algo}) is a modification to \texttt{edge-popup} that functions by enlarging the \emph{score space} of the model, allowing a dense layer to be treated as though it were embedded inside a half-sparse augmented layer, without ever modifying the underlying weights (the motivation for doing this is described along with the summary of our theoretical results in \Cref{theory summary section}). The idea is to introduce two auxiliary score tensors $S_t$, $T_t$ of the same shape as each weight tensor $W_t$, and to apply the masking operation to the concatenation of the two score tensors $\widehat{S}_t=(S_t,T_t)$ at fixed density $\frac{1}{2}$ to produce a double-width mask $\widehat{H}_t=\textrm{TopKMask}(\hat{S}_t; 0.5)$, then yield the final mask by restricting the result back to the original coordinates; i.e. $H_t=\widehat{H}_t|_{\textrm{orig}}$.  Thus, one obtains the expressive flexibility of a half-sparse augmented system without introducing any additional weight parameters.

    \begin{algorithm}[ht]
        \caption{\texttt{double-scoring}}
        \label{double-scoring algo}
        \begin{algorithmic}[1]\small
            \Require Frozen randomly initialized weights $\{W_t,b_t\}_{t=1}^{L}$, training data, loss function $\mathcal{L}$, learning rate $\eta$
            \For{$t=1,\dots,L$}
                \State Initialize two score tensors $S_t,T_t$ and bias score tensors $f_t,g_t$ with the same shapes as $W_t$ and $b_t$
            \EndFor
            \For{each training iteration}
                \For{$t=1,\dots,L$}
                    \State $\widehat{S}_t \gets (S_t,T_t),\enskip \widehat{f}_t \gets (f_t,g_t)$
                    \State $\widehat{H}_t \gets \operatorname{TopKMask}(\widehat{S}_t;1/2), \enskip \widehat{h}_t \gets \operatorname{TopKMask}(\widehat{f}_t;1/2)$ 
                    \State $H_t \gets \widehat{H}_t|_{\mathrm{orig}}, \enskip h_t \gets \widehat{h}_t|_{\mathrm{orig}}$
                    \State $\widetilde{W}_t \gets W_t \odot H_t, \enskip \widetilde{b}_t \gets b_t \odot h_t$
                \EndFor
                \State Compute the network output using $\{(\widetilde{W}_t,\widetilde{b}_t)\}_{t=1}^{L}$,  Compute loss $\mathcal{L}$
                \For{$t=1,\dots,L$}
                    \State Update $S_t,T_t,f_t,g_t$ by gradient descent with straight-through gradients
                \EndFor
            \EndFor
            \State \Return Final masks $\{H_t,h_t\}_{t=1}^{L}$
        \end{algorithmic}
    \end{algorithm}
    \begin{remark}
        Although the auxiliary scores $T_t$ and $g_t$ participate in the top-\(k\) competition with $S_t$, their gradients vanish because the auxiliary coordinates are multiplied by zero weights when interpreted as being embedded in a large sparse model. Thus, the auxiliary scores act as a competitive reservoir; in Section~\ref{ablations subsection}, freezing them yields nearly identical performance.  Likewise, the initialization distribution for the score tensors in \Cref{double-scoring algo} can potentially be varied (further explored with preliminary experiments in \Cref{initialization learning}), though we use a standard normal distribution in experiments.  The weight and bias tensors are initialized from a signed Kaiming distribution, in line with the original \texttt{edge-popup} algorithm.
    \end{remark}
    This construction fixes the density globally at $\frac{1}{2}$ while still allowing the effective density on the original weights to be learned implicitly. As we will show in \Cref{theory summary section}, the resulting search space is expressive enough to represent all subnetworks obtainable via explicit layerwise density selection. In this sense, the combinatorial search over density parameters is replaced by a continuous optimization over an enlarged score space, thereby eliminating the need to choose the brittle layerwise density hyperparameter during optimization without any detriment to the masked model's expressive power. We emphasize:  these results are representational; they remove the expressive obstruction created by fixing density, but they do not by themselves guarantee that straight-through score optimization will find the desired mask.
    \subsection{Computational complexity}
    
        The proposed \texttt{double-scoring} algorithm increases the number of score parameters by a constant factor (specifically, a factor of two per layer), but does not alter the size of the weight tensors or the structure of the forward and backward passes through the network. In particular, the dominant computational cost remains the evaluation of the network and its gradients with respect to the masked weights.  The additional overhead arises from computing $\operatorname{TopKMask}$ on the augmented score tensors and updating the auxiliary scores. Both operations scale linearly (or log-linearly, depending on the implementation of the top-$k$ operation) in the number of parameters, and therefore introduce only a constant-factor increase in runtime.
        
        Consequently, the overall asymptotic time complexity of the \texttt{double-scoring} procedure matches that of classical \texttt{edge-popup} and standard network training while eliminating the need for repeated runs over different sparsity configurations. In practice, this leads to substantial computational savings relative to approaches that rely on exhaustive or heuristic search over layerwise density parameters, or methods that involve partial training and rewinding.

\section{Theory summary}
\label{theory summary section}

    The mechanism of \texttt{double-scoring} relies on an observation that is elementary but useful. If a weight is already equal to zero, then changing the corresponding mask entry has no effect on the represented subnetwork. Consequently, on a sufficiently sparse layer, one may alter mask entries on zero-weight locations without changing the resulting masked tensor. On a sufficiently sparse tensor, the apparent freedom in the density parameter in \texttt{edge-popup} is largely illusory: one may trade a search over densities for a suitable choice of scored coordinates.

    \begin{proposition}\label{prop:half-density}
        Let $W \in \mathbb{R}^{m\times n}$ and suppose at least half of the entries of $W$ are zero. Let $d=mn$ and let $r=\lfloor \frac{d}{2} \rfloor$. Then, for every binary mask $M \in \{0,1\}^{m\times n}$, there exists a binary mask $M^\ast \in \{0,1\}^{m\times n}$ with exactly $r$ ones such that
        \[
        W \odot M = W \odot M^\ast.
        \]
    \end{proposition}
    \begin{corollary}\label{cor:network-half-density}
        Let $N$ be a feed-forward network, and suppose each layer of $N$ contains at least half zero weights. Then every masked subnetwork of $N$ can be represented by choosing, in each layer $t$, a mask with exactly $\lfloor \frac{d_t}{2} \rfloor$ ones, where $d_t$ is the number of weights in layer $t$.
    \end{corollary}
    Applied layerwise, \Cref{cor:network-half-density} says that once each layer is at least half sparse, the entire search space of masked subnetworks can be represented using a single fixed density near $\frac12$ in every layer. This does \emph{not} mean that \texttt{edge-popup} at $k=\frac12$ will automatically converge to the best possible mask; optimization issues such as local minima may still remain, just as they do in ordinary training. What it does show, however, is that the principal \emph{expressive} obstruction disappears: in the half-sparse regime, fixing the density at $\frac{1}{2}$ is already rich enough to represent every masked subnetwork.
    
    If half-sparse layers admit a fixed-density search, then one can attempt to embed a dense layer into a larger layer whose additional entries are permanently zero. Running \texttt{edge-popup} at fixed density $1/2$ on this enlarged layer would then be expressive enough by \Cref{cor:network-half-density}.  Naively enlarging the weight tensors themselves would be wasteful. However, the additional zero entries do not require weights---they only require \emph{scores}. This leads to a simple but crucial observation: one can enlarge the \emph{score space} without enlarging the weight space.
    \begin{proposition}\label{prop:augmented-space}
        Let $w \in \mathbb{R}^d$ and define the augmented vector $\widehat{w}=(w,0)\in \mathbb{R}^{2d}$ by adjoining $d$ zero coordinates. Then for every binary mask $m \in \{0,1\}^d$ there exists a binary mask $\widehat{m} \in \{0,1\}^{2d}$ with exactly $d$ ones such that the restriction of $\widehat{w}\odot \widehat{m}$ to the first $d$ coordinates equals $w\odot m$.
    \end{proposition}
    That is, \texttt{double-scoring} is not merely representationally equivalent to \texttt{edge-popup} on an augmented half-sparse system; the optimization dynamics agree exactly as well (if this is not obvious, \cref{proofs appendix} contains formal verification of this fact).  These results show that \texttt{double-scoring} does not merely add parameters; it changes the parameterization so that fixed-density score optimization can represent every effective sparsity pattern.
\section{Experiments}
\label{experiments section}
    In this section, we verify the performance of \texttt{double-scoring} experimentally.  Note that the theoretical construction fixes the augmented score density at $\frac{1}{2}$, which shows that no layerwise density tuple is needed for representability. In experiments, we also consider a targetable augmented variant in which a single global augmented-space density is supplied to probe high-sparsity regimes and an iterated variant which aggregates the half-density masks over multiple runs. We focus on the one-shot augmented variant. The iterated variant can reach different sparsity regimes but incurs a multiplicative runtime cost proportional to the number of rounds.  In both cases, the method removes the need to tune a separate density for each layer; the effective original-coordinate sparsity is induced by competition with auxiliary score coordinates and is reported explicitly. Throughout, it should be kept in mind that \texttt{double-scoring} should be interpreted as an effective-sparsity method: the requested augmented-space density determines a competition rule, while the original-coordinate sparsity is induced by the learned score geometry. Therefore, the relevant quantities are both accuracy and achieved sparsity.

    Unless otherwise stated, experiments use a bias-free ReLU MLP with three width-256 hidden layers on FashionMNIST.  For each of three seeds, all methods share the same random initialization: \(5000\) extraction/training examples, \(5000\) validation examples, and the full test set.  Strong-ticket accuracy is measured by extracting a mask, rewinding weights to the shared initialization and evaluating the masked network without further weight training.  Sparse-training methods are included as rewound-topology baselines.  As an additional architecture sanity check, \Cref{app:cifar-convnet-matched} repeats the strong-ticket comparison on CIFAR-10 using a VGG-style ConvNet without BatchNorm and with baselines matched to the achieved sparsity of DoubleScore-Augmented.  
    \subsection{Strong ticket extraction}
    \label{strong ticket extraction subsection}

        We first evaluate whether each method can extract a useful subnetwork while the weights remain frozen at their random initialization.  For each seed, we initialize a common linear MLP, allow each method to produce a binary mask, then rewind the weights to the shared initialization before evaluating the masked network.  This ensures that the reported accuracy measures strong-ticket extraction rather than sparse training performance.  
        
        \begin{table}[h]
            \centering
            \caption{
                High-sparsity strong-ticket extraction results.
                Entries report test accuracy of the extracted untrained subnetwork, with achieved sparsity shown separately.
                Sparse-training methods are evaluated as rewound-topology baselines: their masks are learned during sparse training but evaluated after rewinding weights to the shared initialization.
            }
            \begin{tabular}{lrrrr}
                \toprule
                Method & 90\% acc. & 90\% achieved & 95\% acc. & 95\% achieved \\
                \midrule
                Random & 8.92 $\pm$ 0.93 & 90.0 & 12.20 $\pm$ 2.22 & 95.0 \\
                SNIP & 9.53 $\pm$ 0.80 & 90.0 & 10.06 $\pm$ 0.11 & 95.0 \\
                GraSP & 11.29 $\pm$ 1.31 & 90.0 & 13.28 $\pm$ 5.53 & 95.0 \\
                EdgePopup & 77.73 $\pm$ 0.38 & 90.0 & 60.61 $\pm$ 0.97 & 95.0 \\
                IMP & 64.20 $\pm$ 6.01 & 90.0 & 62.11 $\pm$ 4.14 & 95.0 \\
                SET & 33.97 $\pm$ 6.64 & 90.0 & 27.40 $\pm$ 9.53 & 95.0 \\
                RigL & 17.69 $\pm$ 10.02 & 90.0 & 16.32 $\pm$ 8.11 & 95.0 \\
                Movement & 76.34 $\pm$ 3.24 & 90.0 & 71.88 $\pm$ 7.08 & 95.0 \\
                DoubleScore-Aug & 82.74 $\pm$ 0.04 & 81.3 & 78.03 $\pm$ 0.70 & 90.0 \\
                DoubleScore-Iter & 82.09 $\pm$ 0.40 & 86.6 & 77.06 $\pm$ 0.76 & 90.7 \\
                \bottomrule
            \end{tabular}
            \label{tab:strong ticket extraction table}
        \end{table}
        
        \Cref{tab:strong ticket extraction table} summarizes the high-sparsity regime on FashionMNIST.  These comparisons are not intended to claim same-sparsity dominance in every row.  The double-score variants induce an effective sparsity on the original coordinates, which can differ from the nominal requested sparsity.  We therefore report achieved sparsity throughout.  In \Cref{app:cifar-convnet-matched}, we include a matched-sparsity CIFAR-10 ConvNet sanity check in which baselines are run at the achieved sparsity induced by DoubleScore-Augmented.  Fixed-density \texttt{edge-popup} is a strong direct baseline, achieving \(77.7\%\) accuracy at \(90.0\%\) achieved sparsity, but its performance drops sharply at the more extreme \(95\%\) requested sparsity, where it obtains \(60.6\%\).  In contrast, DoubleScore-Augmented obtains \(82.7\%\) accuracy in the \(90\%\) requested-sparsity setting and \(78.0\%\) in the \(95\%\) setting, achieving sparsities of \(81.3\%\) and \(90\%\) respectively.  DoubleScore-Iterated is similarly strong, achieving \(82.1\%\) and \(77.1\%\) accuracy, respectively, at achieved sparsities of \(86.6\%\) and \(90.7\%\).
        
        The comparison to pruning-at-initialization methods is especially stark.  SNIP and GraSP remain near chance in this benchmark, indicating that simple saliency-at-initialization criteria do not identify high-performing strong tickets in this setting.  Rewound sparse-training baselines provide a stronger comparison: movement pruning, in particular, discovers topologies that remain useful after rewinding.  Nevertheless, the double-score variants outperform these rewound topologies in the high-sparsity regime, while extracting masks with frozen weights.
        
        As in all experiments, we report achieved sparsity in addition to requested sparsity.  The augmented score-space method does not enforce the requested original-coordinate density exactly; rather, the effective sparsity emerges through competition with auxiliary score coordinates.  This is part of the mechanism studied in \Cref{hyperparameter stability subsection} and \Cref{ablations subsection}.  The key conclusion from the baseline comparison is that \texttt{double-scoring} produces high-accuracy strong tickets in regimes where fixed-density \texttt{edge-popup} and standard pruning-at-initialization baselines degrade substantially.  These results suggest that the main obstacle is not the existence of strong tickets, but their extraction: once the sparsity-selection bottleneck is relaxed through augmented score-space optimization, high-performing strong tickets can be found reliably.
    \subsection{Sparse training and weak tickets}
    \label{sparse training and weak tickets subsection}

        As shown in the previous section, \texttt{double-scoring} dominates in the strong-ticket regime.  Here, we show that it also remains competitive after training.  For each method, we first extract a binary mask, rewind the weights to the shared random initialization, and then train only the surviving weights while keeping the mask fixed.  Thus this experiment separates the quality of the extracted topology at initialization from its trainability under standard sparse training.
        
        \begin{figure}[h]
            \centering
            \includegraphics[width=1.0\linewidth]{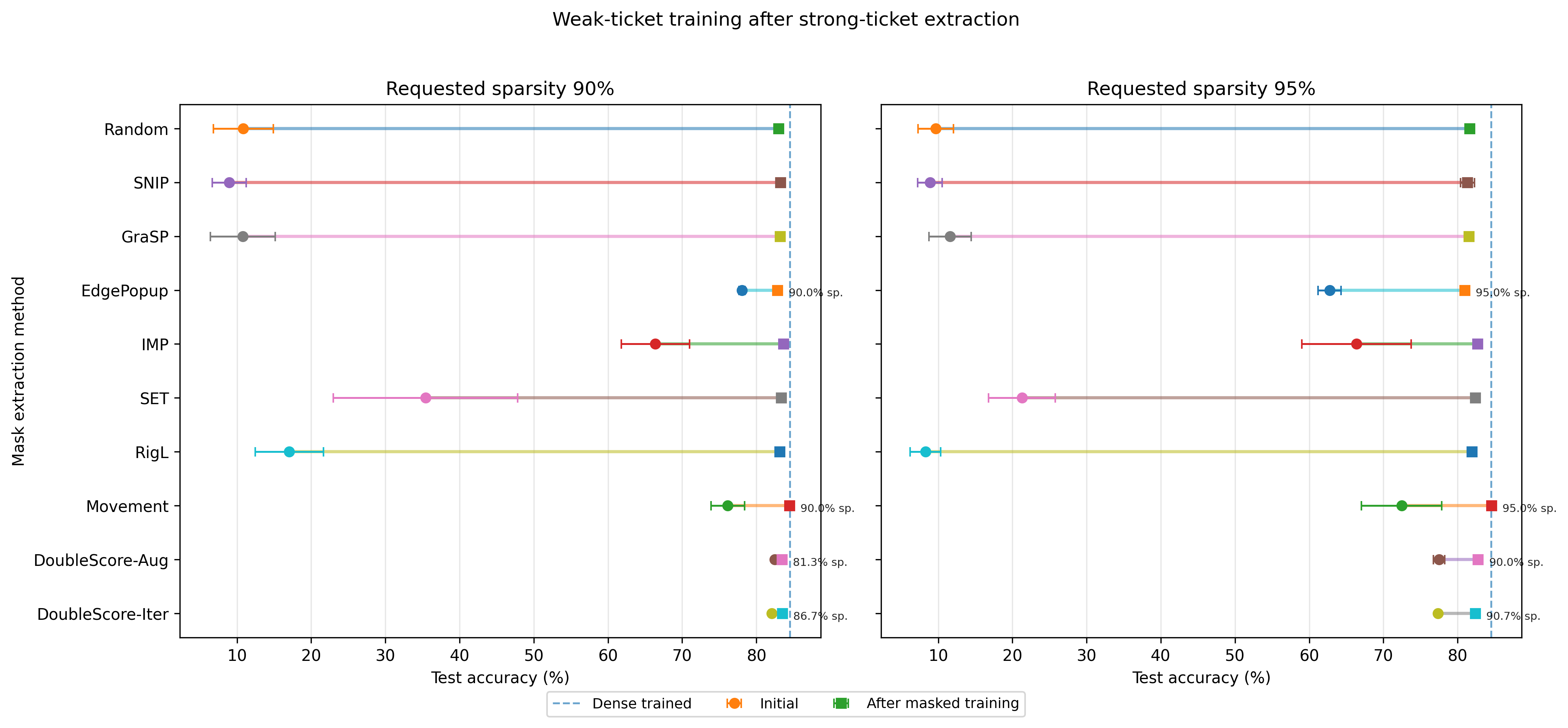}
            \caption{
                Weak-ticket training after strong-ticket extraction.
                Each method first extracts a binary mask, after which weights are rewound to the shared random initialization and trained with the mask fixed.
                Points show test accuracy before and after masked weight training.
                DoubleScore masks begin with high accuracy and remain competitive after training, whereas several baselines only become accurate after weight optimization.
                Sparse-training methods such as Movement are included as rewound-topology baselines.
            }
            \label{fig:sparse training and weak tickets figure}
        \end{figure}
        
        The results indicate that double-score masks are not only strong at initialization, but also remain highly trainable.  At requested \(90\%\) sparsity, DoubleScore-Augmented begins at \(82.5\%\) test accuracy and reaches \(83.4\%\) after masked training, while DoubleScore-Iterated begins at \(82.1\%\) and reaches \(83.5\%\).  In contrast, pruning-at-initialization baselines such as SNIP and GraSP begin near chance but train to above \(83\%\).  Thus these baselines can produce trainable sparse networks, but they do not identify strong tickets.
        
        The comparison with sparse-training methods further clarifies the distinction.  Movement pruning achieves the highest final trained accuracy, reaching \(84.4\%\) at requested \(90\%\) sparsity and \(84.6\%\) at requested \(95\%\) sparsity.  However, movement pruning learns its topology while training weights, whereas double-score methods extract their masks with frozen weights.  The role of \texttt{double-scoring} is therefore not to replace all sparse-training procedures, but to produce subnetworks that are already accurate at initialization and remain competitive after subsequent training.

    \subsection{Hyperparameter stability}
    \label{hyperparameter stability subsection}

        A central motivation for \texttt{double-scoring} is that \texttt{edge-popup} requires the user to specify a sparsity level, and in multilayer networks this choice is effectively a layerwise hyperparameter.  To test whether this tuning problem is merely cosmetic or whether it materially affects extraction quality, we performed a sparsity-sensitivity experiment on the controlled FashionMNIST MLP benchmark.  We fixed the requested sparsity at \(90\%\), trained \texttt{edge-popup} over a collection of scalar keep densities and randomly sampled layerwise keep-density tuples, and evaluated each extracted mask after rewinding the weights to the same random initialization.  We compare this sweep to fixed target-density \texttt{edge-popup}, random masks, and the two double-score variants.  For the \texttt{edge-popup} sweep, the reported range is taken over all tested scalar and layerwise keep-density configurations across seeds.  We also report two validation-oracle \texttt{edge-popup} results, obtained by selecting the \texttt{edge-popup} configuration with the best validation accuracy for each seed both with and without a sparsity constraint.  The constrained oracle is the relevant same-regime comparison. The unconstrained oracle is included only to show that validation tuning over \(k\) can drift to a substantially denser regime.

        \begin{figure}[t]
            \centering
            \includegraphics[width=1.0\linewidth]{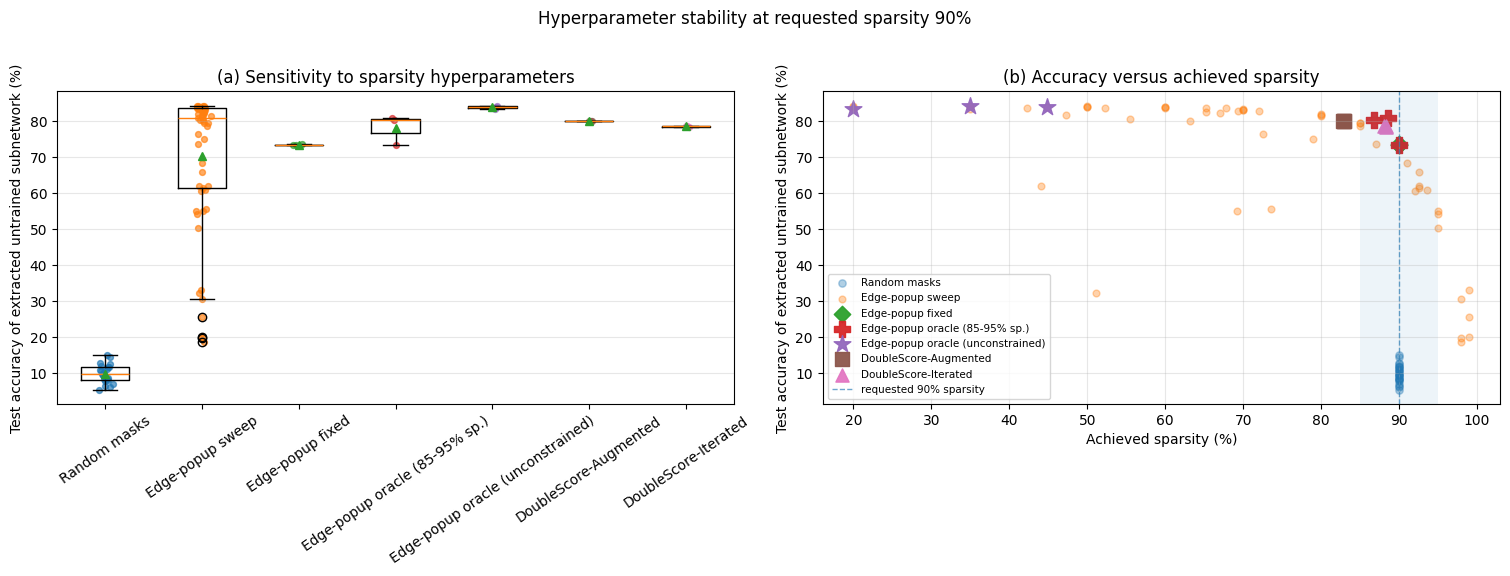}
            \caption{
                Hyperparameter stability at requested \(90\%\) sparsity on FashionMNIST with a linear MLP.
                Left: \texttt{edge-popup} exhibits large variance across scalar and layerwise sparsity choices, while the double-score variants are stable across seeds.
                Right: the constrained oracle provides a same-regime high-sparsity reference, while the unconstrained validation-oracle \texttt{edge-popup} configuration achieves high accuracy only by selecting substantially denser masks, whereas the double-score variants remain in the high-sparsity regime.
                All masks are evaluated after rewinding weights to the shared random initialization.
            }
            \label{fig:stability experiment figure}
        \end{figure}
        
        The results show that \texttt{edge-popup} is highly sensitive to the sparsity choice.  Across the sweep, \texttt{edge-popup} test accuracy ranges from \(18.6\%\) to \(84.4\%\), with a standard deviation of \(19.6\) percentage points.  The fixed target-density \texttt{edge-popup} baseline is stable but substantially lower, achieving \(73.5\%\) accuracy at \(90.0\%\) achieved sparsity.  In contrast, DoubleScore-Augmented achieves \(80.1\%\) accuracy with standard deviation below \(0.2\) percentage points, and DoubleScore-Iterated achieves \(78.6\%\) accuracy with similarly low variance.  Thus the double-score variants provide high-performing strong-ticket extraction without requiring a sparsity sweep.

        To separate hyperparameter sensitivity from the trivial capacity advantage of denser subnetworks, we report two validation-selected \texttt{edge-popup} references. The first is a constrained oracle, which selects the best validation configuration among \texttt{edge-popup} masks whose achieved sparsity lies in the high-sparsity band \(85\%-95\%\). This separates high-sparsity tuning from the capacity advantage of denser subnetworks and is the relevant same-regime reference.  The unconstrained oracle is included only as a diagnostic: it can achieve higher accuracy by selecting substantially denser masks, and therefore is not a same-sparsity comparison.  Thus the advantage of \texttt{double-scoring} is not that it dominates every possible \texttt{edge-popup} configuration, but that it avoids a fragile layerwise sparsity search while remaining in a high-sparsity regime.
        
    \subsection{Ablations}
    \label{ablations subsection}

        The previous experiments show that the augmented score-space parameterization improves extraction and stabilizes sparsity selection. We next isolate which components of the augmented score-space parameterization are responsible for the observed gains.  A natural critique of the \texttt{double-scoring} algorithm is that the gains it produces are the result of doubling the number of trainable parameters.  Via a series of ablations, we will show that the gain is not merely from more trainable score parameters; the critical change is the minimal doubled score-space parameterization.  We compare ordinary single-score \texttt{edge-popup}, signed-score variants without the absolute-value parameterization, augmented score spaces with different auxiliary widths, a frozen-auxiliary variant, and a projected-final variant.  In all cases, weights are frozen during mask extraction and the resulting mask is evaluated after rewinding to the shared random initialization.

        \begin{table}
            \centering
            \caption{
                Ablation study at requested \(90\%\) sparsity.
                All variants are evaluated as strong-ticket extractors after rewinding weights to the shared initialization.
                Score parameters are reported as trainable/total multiples relative to the original weight count.
            }
            \begin{tabular}{lrrrr}
                \toprule
                Variant & Test acc. (\%) & Achieved       & Score             & Interpretation \\
                        &                & sparsity (\%)  & params            &                \\
                        &                &                & (train/total) &                \\
                \midrule
                Random & 8.93 $\pm$ 4.39 & 90.0 $\pm$ 0.0 & 0.0$\times$/0.0$\times$ & chance baseline \\
                SingleScore-Abs & 73.88 $\pm$ 0.31 & 90.0 $\pm$ 0.0 & 1.0$\times$/1.0$\times$ & \texttt{edge-popup} baseline \\
                SingleScore-NoAbs & 72.60 $\pm$ 0.52 & 90.0 $\pm$ 0.0 & 1.0$\times$/1.0$\times$ & signed-score ranking \\
                Aug-x1-Abs & 80.34 $\pm$ 0.42 & 82.9 $\pm$ 0.2 & 2.0$\times$/2.0$\times$ & canonical augmented method \\
                Aug-x1-FrozenAux & 79.87 $\pm$ 0.42 & 82.9 $\pm$ 0.0 & 1.0$\times$/2.0$\times$ & tests trainable auxiliary capacity \\
                Aug-x1-NoAbs & 78.49 $\pm$ 0.20 & 83.3 $\pm$ 0.2 & 2.0$\times$/2.0$\times$ & tests abs-score parameterization \\
                Aug-x1-Projected & 70.16 $\pm$ 6.48 & 90.0 $\pm$ 0.0 & 2.0$\times$/2.0$\times$ & exact final original sparsity \\
                Aug-x2-Abs & 80.62 $\pm$ 0.26 & 80.8 $\pm$ 0.1 & 3.0$\times$/3.0$\times$ & larger auxiliary space \\
                Aug-x4-Abs & 80.93 $\pm$ 0.45 & 79.1 $\pm$ 0.2 & 5.0$\times$/5.0$\times$ & much larger auxiliary space \\
                \bottomrule
            \end{tabular}
            \label{tab:ablations table}
        \end{table}
        
        The ablations show that the augmented score space is the critical ingredient for inducing high-performing effective sparsity patterns.  At requested \(90\%\) sparsity, ordinary single-score \texttt{edge-popup} with \(\mathrm{abs}{S}\)-ranking achieves \(73.9\%\) accuracy at \(90.0\%\) achieved sparsity.  The canonical augmented variant, Aug-x1-Abs, improves to \(80.3\%\) accuracy, while achieving \(82.9\%\) sparsity.  Thus the augmented method substantially improves the extracted strong ticket, but it does so by allowing the effective original-coordinate sparsity to emerge from competition with auxiliary score coordinates rather than by enforcing the requested density exactly.  SingleScore-Abs is the \texttt{edge-popup} baseline using magnitude-based score ranking.  The deviation in the result here from that of \Cref{tab:strong ticket extraction table} highlights the sensitivity of classical \texttt{edge-popup} to score initializations, as demonstrated in \Cref{fig:stability experiment figure}.
        
        The frozen-auxiliary ablation is especially informative.  Freezing the auxiliary scores barely changes performance: Aug-x1-Abs obtains \(80.3\%\) accuracy, while Aug-x1-Abs-FrozenAux obtains \(79.9\%\).  This indicates that the gain is not simply due to adding more trainable score parameters.  Instead, the auxiliary coordinates act as a learned or fixed competitive reservoir that changes the thresholding dynamics in the original coordinates.  The absolute-value score parameterization is also consistently beneficial.  Removing it reduces performance for both single-score and augmented variants.  At requested \(90\%\) sparsity, SingleScore-Abs outperforms SingleScore-NoAbs, and Aug-x1-Abs outperforms Aug-x1-NoAbs.  This supports the use of magnitude-based score ranking in the main method.
        
        Finally, the projected-final variant clarifies the role of effective sparsity.  When the augmented method is trained normally but then projected back to exactly the requested original-coordinate sparsity, performance drops substantially.  This suggests that the advantage of \texttt{double-scoring} is not simply that it learns better original-coordinate saliency scores for a fixed target density.  Rather, the augmented score space improves extraction by jointly shaping the mask and its effective sparsity.  Consequently, all main comparisons report achieved sparsity alongside accuracy.
    
\section{Limitations}
\label{limitations section}

    Several limitations remain.  First, the theoretical results are representational and do not guarantee that straight-through score optimization will find globally optimal masks.  Second, \texttt{double-scoring} induces an effective original-coordinate sparsity rather than enforcing an exact sparsity budget.  This is useful when the goal is high-performing sparse extraction, but applications requiring exact sparsity may require calibration of the nominal augmented density or a projection procedure, and our simple projected-final ablation shows that naive projection can substantially reduce accuracy.  Third, the main experiments are controlled FashionMNIST MLP benchmarks.  We include a no-BatchNorm CIFAR-10 ConvNet sanity check, but large-scale architectures, transformers, and modern high-resolution benchmarks remain open due to compute constraints.  Finally, the iterated variant incurs substantially higher runtime than the one-shot augmented variant, so we treat it as an exploratory extension rather than the main practical method.
\section{Conclusion}
\label{conclusion section}

    We introduced \texttt{double-scoring}, an augmented score-space parameterization for \texttt{edge-popup}-style strong lottery ticket extraction. The method replaces a brittle layerwise sparsity search with competition in an enlarged score space, while leaving the weight tensors unchanged. Theoretical results show that the augmented parameterization preserves representational access to original-coordinate masks, and controlled experiments show improved strong-ticket extraction, strong hyperparameter stability, and robust trainability after mask extraction. These results suggest that the practical barrier to strong lottery ticket extraction is not merely the existence of suitable subnetworks, but the parameterization used to search for them.
    
\begin{ack}

    Use unnumbered first level headings for the acknowledgments. All acknowledgments
    go at the end of the paper before the list of references. Moreover, you are required to declare funding (financial activities supporting the submitted work) and competing interests (related financial activities outside the submitted work).
    More information about this disclosure can be found at: \url{https://neurips.cc/Conferences/2026/PaperInformation/FundingDisclosure}.
\end{ack}

\bibliographystyle{plainnat}
\bibliography{references}

@article{lotterytickethypothesis,
  doi = {10.48550/ARXIV.1803.03635},
  
  url = {https://arxiv.org/abs/1803.03635},
  
  author = {Frankle, Jonathan and Carbin, Michael},
  
  title = {The Lottery Ticket Hypothesis: Finding Sparse, Trainable Neural Networks},
  
  publisher = {arXiv},
  
  year = {2018},
  
  copyright = {arXiv.org perpetual, non-exclusive license}
}

@misc{ramanujan,
  doi = {10.48550/ARXIV.1911.13299},
  
  url = {https://arxiv.org/abs/1911.13299},
  
  author = {Ramanujan, Vivek and Wortsman, Mitchell and Kembhavi, Aniruddha and Farhadi, Ali and Rastegari, Mohammad},
  
  title = {What's Hidden in a Randomly Weighted Neural Network?},
  
  publisher = {arXiv},
  
  year = {2019},
  
  copyright = {arXiv.org perpetual, non-exclusive license}
}

@misc{pruningisallyouneed,
  doi = {10.48550/ARXIV.2002.00585},
  
  url = {https://arxiv.org/abs/2002.00585},
  
  author = {Malach, Eran and Yehudai, Gilad and Shalev-Shwartz, Shai and Shamir, Ohad},
  
  title = {Proving the Lottery Ticket Hypothesis: Pruning is All You Need},
  
  publisher = {arXiv},
  
  year = {2020},
  
  copyright = {arXiv.org perpetual, non-exclusive license}
}

@article{orseau2020logarithmic,
  title={Logarithmic pruning is all you need},
  author={Orseau, Laurent and Hutter, Marcus and Rivasplata, Omar},
  journal={Advances in Neural Information Processing Systems},
  volume={33},
  pages={2925--2934},
  year={2020}
}

@article{lee2018snip,
  title={Snip: Single-shot network pruning based on connection sensitivity},
  author={Lee, Namhoon and Ajanthan, Thalaiyasingam and Torr, Philip HS},
  journal={arXiv preprint arXiv:1810.02340},
  year={2018}
}

@article{wang2020picking,
  title={Picking winning tickets before training by preserving gradient flow},
  author={Wang, Chaoqi and Zhang, Guodong and Grosse, Roger},
  journal={arXiv preprint arXiv:2002.07376},
  year={2020}
}

@article{tanaka2020pruning,
  title={Pruning neural networks without any data by iteratively conserving synaptic flow},
  author={Tanaka, Hidenori and Kunin, Daniel and Yamins, Daniel L and Ganguli, Surya},
  journal={Advances in Neural Information Processing Systems},
  volume={33},
  pages={6377--6389},
  year={2020}
}

@article{pensia2020optimal,
  title={Optimal lottery tickets via subset sum: Logarithmic over-parameterization is sufficient},
  author={Pensia, Ankit and Rajput, Shashank and Nagle, Alliot and Vishwakarma, Harit and Papailiopoulos, Dimitris},
  journal={Advances in neural information processing systems},
  volume={33},
  pages={2599--2610},
  year={2020}
}

@InProceedings{pmlr-v119-frankle20a,
  title = 	 {Linear Mode Connectivity and the Lottery Ticket Hypothesis},
  author =       {Frankle, Jonathan and Dziugaite, Gintare Karolina and Roy, Daniel and Carbin, Michael},
  booktitle = 	 {Proceedings of the 37th International Conference on Machine Learning},
  pages = 	 {3259--3269},
  year = 	 {2020},
  editor = 	 {III, Hal Daumé and Singh, Aarti},
  volume = 	 {119},
  series = 	 {Proceedings of Machine Learning Research},
  month = 	 {13--18 Jul},
  publisher =    {PMLR},
  url = 	 {https://proceedings.mlr.press/v119/frankle20a.html}
}

@article{yang2023medmnist,
  title={Medmnist v2-a large-scale lightweight benchmark for 2d and 3d biomedical image classification},
  author={Yang, Jiancheng and Shi, Rui and Wei, Donglai and Liu, Zequan and Zhao, Lin and Ke, Bilian and Pfister, Hanspeter and Ni, Bingbing},
  journal={Scientific Data},
  volume={10},
  number={1},
  pages={41},
  year={2023},
  publisher={Nature Publishing Group UK London}
}

@inproceedings{yang2021medmnist,
  title={Medmnist classification decathlon: A lightweight automl benchmark for medical image analysis},
  author={Yang, Jiancheng and Shi, Rui and Ni, Bingbing},
  booktitle={2021 IEEE 18th International Symposium on Biomedical Imaging (ISBI)},
  pages={191--195},
  year={2021},
  organization={IEEE}
}

@article{masoudnia2014mixture,
  title={Mixture of experts: a literature survey},
  author={Masoudnia, Saeed and Ebrahimpour, Reza},
  journal={Artificial Intelligence Review},
  volume={42},
  pages={275--293},
  year={2014},
  publisher={Springer}
}

@article{yuksel2012twenty,
  title={Twenty years of mixture of experts},
  author={Yuksel, Seniha Esen and Wilson, Joseph N and Gader, Paul D},
  journal={IEEE transactions on neural networks and learning systems},
  volume={23},
  number={8},
  pages={1177--1193},
  year={2012},
  publisher={IEEE}
}

@article{cai2024survey,
  title={A survey on mixture of experts},
  author={Cai, Weilin and Jiang, Juyong and Wang, Fan and Tang, Jing and Kim, Sunghun and Huang, Jiayi},
  journal={arXiv preprint arXiv:2407.06204},
  year={2024}
}

@article{nguyen2018practical,
  title={Practical and theoretical aspects of mixture-of-experts modeling: An overview},
  author={Nguyen, Hien D and Chamroukhi, Faicel},
  journal={Wiley Interdisciplinary Reviews: Data Mining and Knowledge Discovery},
  volume={8},
  number={4},
  pages={e1246},
  year={2018},
  publisher={Wiley Online Library}
}

@incollection{gormley2019mixture,
  title={Mixture of experts models},
  author={Gormley, Isobel Claire and Fr{\"u}hwirth-Schnatter, Sylvia},
  booktitle={Handbook of mixture analysis},
  pages={271--307},
  year={2019},
  publisher={Chapman and Hall/CRC}
}

@article{chen2022towards,
  title={Towards understanding the mixture-of-experts layer in deep learning},
  author={Chen, Zixiang and Deng, Yihe and Wu, Yue and Gu, Quanquan and Li, Yuanzhi},
  journal={Advances in neural information processing systems},
  volume={35},
  pages={23049--23062},
  year={2022}
}

@article{weiss2016survey,
  title={A survey of transfer learning},
  author={Weiss, Karl and Khoshgoftaar, Taghi M and Wang, DingDing},
  journal={Journal of Big data},
  volume={3},
  pages={1--40},
  year={2016},
  publisher={Springer}
}

@article{zhuang2020comprehensive,
  title={A comprehensive survey on transfer learning},
  author={Zhuang, Fuzhen and Qi, Zhiyuan and Duan, Keyu and Xi, Dongbo and Zhu, Yongchun and Zhu, Hengshu and Xiong, Hui and He, Qing},
  journal={Proceedings of the IEEE},
  volume={109},
  number={1},
  pages={43--76},
  year={2020},
  publisher={Ieee}
}

@article{hosna2022transfer,
  title={Transfer learning: a friendly introduction},
  author={Hosna, Asmaul and Merry, Ethel and Gyalmo, Jigmey and Alom, Zulfikar and Aung, Zeyar and Azim, Mohammad Abdul},
  journal={Journal of Big Data},
  volume={9},
  number={1},
  pages={102},
  year={2022},
  publisher={Springer}
}

@article{sutton2019bitter,
  title={The bitter lesson, 2019},
  author={Sutton, Rich},
  journal={URL http://www.incompleteideas.net/IncIdeas/BitterLesson.html},
  year={2019}
}

@article{yu2024super,
  title={The super weight in large language models},
  author={Yu, Mengxia and Wang, De and Shan, Qi and Reed, Colorado J and Wan, Alvin},
  journal={arXiv preprint arXiv:2411.07191},
  year={2024}
}

@article{mocanu2018scalable,
  title={Scalable training of artificial neural networks with adaptive sparse connectivity inspired by network science},
  author={Mocanu, Decebal Constantin and Mocanu, Elena and Stone, Peter and Nguyen, Phuong H and Gibescu, Madeleine and Liotta, Antonio},
  journal={Nature communications},
  volume={9},
  number={1},
  pages={2383},
  year={2018},
  publisher={Nature Publishing Group UK London}
}

@inproceedings{evci2020rigging,
  title={Rigging the lottery: Making all tickets winners},
  author={Evci, Utku and Gale, Trevor and Menick, Jacob and Castro, Pablo Samuel and Elsen, Erich},
  booktitle={International conference on machine learning},
  pages={2943--2952},
  year={2020},
  organization={PMLR}
}

@article{sanh2020movement,
  title={Movement pruning: Adaptive sparsity by fine-tuning},
  author={Sanh, Victor and Wolf, Thomas and Rush, Alexander},
  journal={Advances in neural information processing systems},
  volume={33},
  pages={20378--20389},
  year={2020}
}

\appendix
    
\section{Reproducibility and code}
\label{reproducibility and code appendix}

    All code necessary to reproduce the experiments is included in the anonymous supplementary material.

\section{Experimental details}
\label{experimental details appendix}

    All experiments in the main paper, with the exception of the toy experiment illustrating the $k$-selection bottleneck in \Cref{fig:k-bottleneck}, are conducted on FashionMNIST using a fully connected ReLU network.  The input images are flattened to dimension \(784\), and the network has three hidden layers of width \(256\), followed by a \(10\)-class output layer.  Unless otherwise stated, all linear layers are bias-free.  We normalize FashionMNIST using mean \(0.2860\) and standard deviation \(0.3530\).  For each random seed, we randomly shuffle the FashionMNIST training set and use \(5{,}000\) examples for mask extraction or sparse training and \(5{,}000\) disjoint examples for validation.  All reported test accuracies are computed on the full FashionMNIST test set.
    
    Weights are initialized once per seed and shared across methods.  We use the signed Kaiming uniform initialization
    \[
    W_{ij} \sim \operatorname{Unif}\!\left[-\frac{1}{\sqrt{k}}\frac{\operatorname{gain}}{\sqrt{\operatorname{fan\_in}}},
    \frac{1}{\sqrt{k}}\frac{\operatorname{gain}}{\sqrt{\operatorname{fan\_in}}}\right],
    \]
    with \(k=1/2\) and ReLU gain.  After each method extracts a mask, we reload the original initialization before evaluating the masked network.  Thus the strong-ticket experiments measure the quality of the extracted mask rather than the quality of trained weights.
    
    All methods use Adam unless otherwise stated, with \(\beta=(0.9,0.999)\), \(\epsilon=10^{-8}\), and weight decay \(0\).  Score variables are trained with learning rate \(10^{-2}\), while weight-training baselines use learning rate \(10^{-3}\).  The batch size is \(512\).  We set Python, NumPy, and PyTorch random seeds for each run and disable cuDNN benchmarking.

    \paragraph{Datasets and assets used.}
    \begin{table}[t]
        \centering
        \small
        \caption{
            External datasets and software assets used in the experiments.
        }
        \label{tab:assets-licenses}
        \begin{tabular}{llll}
            \toprule
            Asset &bCredit/source & License / terms \\
            \midrule
            FashionMNIST  & Zalando Research; Xiao et al. & MIT License \\
            CIFAR-10 & Krizhevsky, Nair, Hinton & Public academic benchmark \\
            PyTorch  & PyTorch contributors & BSD-style / BSD-3-Clause \\
            torchvision  & torchvision contributors & BSD-3-Clause \\
            NumPy, pandas, matplotlib & package contributors & open-source scientific Python packages \\
            \bottomrule
        \end{tabular}
    \end{table}
    
    \paragraph{Mask densities and achieved sparsity.}
        For a requested sparsity \(s\), the corresponding keep density is \(k=1-s\).  For methods that directly impose a density, masks are constructed by retaining the top \(k\)-fraction of scores or saliencies.  For double-score methods, the effective sparsity on the original weight coordinates is induced by the augmented score-space competition and need not equal the requested sparsity.  We therefore report the achieved sparsity
        \[
        1-\frac{\#\{\text{active mask entries}\}}{\#\{\text{total mask entries}\}}
        \]
        for every method.
    
    \paragraph{Edge-popup.}
        For fixed \texttt{edge-popup}, weights are frozen and each weight tensor has an associated trainable score tensor of the same shape.  At each forward pass, we form a binary mask by retaining the entries with largest score magnitudes, i.e.
        \[
        H_t=\operatorname{TopKMask}(|S_t|;k),
        \]
        and train only the scores using a straight-through estimator for the hard top-\(k\) operation.  The final mask is extracted from the learned score magnitudes.
    
    \paragraph{DoubleScore-Augmented.}
        For each layer with weight matrix \(W_t\in \mathbb{R}^{m\times n}\), the augmented variant introduces a score matrix
        \[
        \widehat S_t\in \mathbb{R}^{m\times (n+n_{\mathrm{aux}})},
        \qquad n_{\mathrm{aux}}=n
        \]
        in the main experiments.  A hard top-\(k\) mask is computed in the augmented score space using \(|\widehat S_t|\), and the mask is then restricted to the original \(n\) coordinates.  The auxiliary score coordinates therefore compete with original coordinates for inclusion, allowing the effective original-coordinate sparsity to emerge from optimization.  No additional weight parameters are introduced.
    
    \paragraph{DoubleScore-Iterated.}
        The iterated variant maintains a cumulative mask.  At each round, it applies a doubled score-space mask only to currently active weights, multiplies the result into the cumulative mask, and stops once the achieved sparsity is within \(0.02\) of the requested sparsity or after at most \(8\) rounds.  The per-round keep density is chosen as the ratio between the target keep density and the current keep density, capped at \(1/2\).
    
    \paragraph{Pruning-at-initialization baselines.}
        SNIP is implemented by introducing differentiable mask parameters initialized to one and ranking weights by the magnitude of the gradient of the loss with respect to these mask parameters.  We compute SNIP saliencies using \(5\) minibatches and then apply a global top-\(k\) threshold.  GraSP is implemented using the standard Hessian-gradient saliency structure.  We compute two independent collections of \(5\) minibatches, use temperature \(200\), form the Hessian-gradient proxy, and rank weights by the signed score \(-\theta(Hg)_\theta\), followed by global top-\(k\) thresholding.
    
    \paragraph{Sparse-training baselines.}
        IMP, SET, RigL, and Movement pruning are included as rewound-topology baselines.  These methods are allowed to train weights while discovering a sparse topology, but the resulting mask is evaluated only after rewinding the weights to the shared random initialization.  Thus their reported strong-ticket accuracy measures whether the learned topology itself transfers back to initialization.
        
        IMP uses \(5\) pruning rounds and \(20\) training epochs per round.  At each round, the model is reset to the original initialization, trained under the current mask, and globally magnitude-pruned to the next density level.
        
        SET and RigL initialize a random global sparse mask at the requested keep density and train weights under the mask.  Every \(100\) optimization steps, \(30\%\) of active weights are pruned by magnitude and the same number of inactive weights are regrown.  SET regrows weights randomly.  RigL regrows using dense gradients computed on the current minibatch.
        
        Movement pruning trains weights and signed movement scores jointly.  It uses separate Adam parameter groups for weights and scores, with learning rates \(10^{-3}\) and \(10^{-2}\), respectively.  The keep density follows a cubic schedule from dense to the requested density, with \(10\%\) warmup and \(10\%\) cooldown.  Movement scores are initialized to zero with \(10^{-6}\) Gaussian noise for tie-breaking, and the final mask is obtained from the signed movement scores without taking absolute values.
    
    \paragraph{Strong-ticket extraction benchmark.}
        For the main strong-ticket comparison, we use seeds \(0,1,2\), requested sparsities \(50\%\), \(90\%\), and \(95\%\), and train mask-extraction methods for \(100\) epochs.  The main paper reports the high-sparsity \(90\%\) and \(95\%\) results, with the \(50\%\) results included in the appendix.
    
    \paragraph{Weak-ticket training benchmark.}
        For the weak-ticket experiment, each method first extracts a mask using the same extraction protocol.  We then rewind the weights to the shared initialization and train the surviving weights for \(40\) epochs while keeping the mask fixed.  We report both the initial accuracy before masked weight training and the final accuracy after masked weight training.  A dense baseline is trained once per seed under the same optimizer settings.
    
    \paragraph{Hyperparameter-stability benchmark.}
        For the hyperparameter-stability experiment, we fix requested sparsity at \(90\%\).  We compare fixed target-density \texttt{edge-popup}, random masks, the double-score variants, and a sweep of \texttt{edge-popup} configurations.  The \texttt{edge-popup} sweep includes scalar keep densities
        \[
        0.01,0.02,0.05,0.075,0.10,0.15,0.20,0.30,0.40,0.50,0.65,0.80
        \]
        and \(8\) randomly sampled layerwise keep-density tuples per seed drawn from the same pool.  We also report a validation-oracle \texttt{edge-popup} result, which selects the \texttt{edge-popup} configuration with highest validation accuracy for each seed.  The oracle is included only as an unconstrained tuning reference; it is not a same-sparsity baseline.
    
    \paragraph{Ablation benchmark.}
        For the ablation study, we use requested sparsities \(90\%\) and \(95\%\) with seeds \(0,1,2\).  We compare ordinary single-score \texttt{edge-popup} with and without magnitude ranking, augmented score spaces with auxiliary-width multipliers \(1,2,\) and \(4\), a frozen-auxiliary variant, and a projected-final variant.  In the frozen-auxiliary variant, the auxiliary score coordinates are randomly initialized and held fixed, while the original score coordinates are trained.  In the projected-final variant, the augmented method is trained normally, but the final mask is projected onto the original coordinates at the exact requested sparsity.  This tests whether the augmented method's gains come from improved original-coordinate saliency at fixed sparsity or from the effective sparsity induced by auxiliary score competition.
    
    \paragraph{Reported statistics.}
        All tables report mean \(\pm\) standard deviation over the stated random seeds.  We report test accuracy, validation accuracy when used for selection, achieved sparsity, and wall-clock runtime.  Validation accuracy is used only for the validation-oracle \texttt{edge-popup} configuration in the hyperparameter-stability experiment; all main comparisons are based on test accuracy after mask extraction and rewinding.

    \paragraph{Compute resources.}
        All reported experiments were run on Google Colab GPU instances.  Runs used a single GPU when available, with PyTorch automatically falling back to CPU otherwise.  The main FashionMNIST MLP experiments require modest compute and can be reproduced on a single consumer GPU.  The CIFAR-10 ConvNet sanity-check experiments are more expensive but were also run on a single Colab GPU.  Each results table reports average wall-clock time per method when relevant.  Because Colab assigns GPU types dynamically, the exact GPU model varied across runs; wall-clock times are therefore reported as approximate reproduction guidance rather than hardware-normalized benchmarks.  The total compute for the reported experiments was dominated by score-training methods that require full extraction runs, especially \texttt{edge-popup} sweeps, dynamic sparse-training baselines, and iterated\texttt{double-scoring}variants.
        
        The reported experiments represent the final controlled runs used in the paper.  During development, additional exploratory and failed runs were performed to debug implementations, choose stable protocols, and test variants not included in the final paper.  These preliminary runs required additional compute but are not used to support the paper's main claims.

\section{Additional baseline experiments}
\label{additional baselines appendix}

    This appendix contains additional experimental results omitted from the main paper for space.  Unless otherwise stated, all experiments use the setup described in Appendix~\ref{experimental details appendix}: a bias-free ReLU MLP with three hidden layers of width \(256\) on FashionMNIST, shared random initializations across methods, \(5{,}000\) extraction/training examples, \(5{,}000\) validation examples, and the full test set.  All reported values are mean \(\pm\) standard deviation over the stated random seeds.

    \subsection{Full strong-ticket extraction results}
    \label{app:full-strong-ticket}
    
        The main paper reports the high-sparsity \(90\%\) and \(95\%\) strong-ticket extraction results.  \Cref{tab:strong-ticket-appendix-all} gives the full strong-ticket table, including the \(50\%\) requested-sparsity setting and wall-clock time.  Each method first extracts a mask, after which the weights are rewound to the shared random initialization before evaluation.  Sparse-training methods are therefore evaluated as rewound-topology baselines.
        
        \begin{table}[h]
            \centering
            \scriptsize
            \setlength{\tabcolsep}{3.5pt}
            \caption{
                Full strong-ticket extraction results.  Each entry reports test accuracy of the extracted untrained subnetwork, achieved sparsity, and average runtime.  All masks are evaluated after rewinding weights to the shared random initialization.
            }
            \label{tab:strong-ticket-appendix-all}
            \begin{tabular}{lllll}
                \toprule
                Requested sparsity & Method & Test acc. (\%) & Achieved sparsity (\%) & Seconds \\
                \midrule
                    50\% & Random & 11.34 $\pm$ 1.78 & 50.0 $\pm$ 0.0 & 2.9 \\
                    50\% & SNIP & 9.04 $\pm$ 1.29 & 50.0 $\pm$ 0.0 & 3.4 \\
                    50\% & GraSP & 10.47 $\pm$ 0.81 & 50.0 $\pm$ 0.0 & 3.9 \\
                    50\% & EdgePopup & 84.03 $\pm$ 0.19 & 50.0 $\pm$ 0.0 & 100.5 \\
                    50\% & IMP & 74.31 $\pm$ 3.59 & 50.0 $\pm$ 0.0 & 99.1 \\
                    50\% & SET & 43.68 $\pm$ 8.87 & 50.0 $\pm$ 0.0 & 99.5 \\
                    50\% & RigL & 39.50 $\pm$ 19.50 & 50.0 $\pm$ 0.0 & 99.2 \\
                    50\% & Movement & 79.44 $\pm$ 1.03 & 50.0 $\pm$ 0.0 & 101.0 \\
                    50\% & DoubleScore-Aug & 84.10 $\pm$ 0.36 & 47.6 $\pm$ 0.0 & 101.3 \\
                    50\% & DoubleScore-Iter & 83.75 $\pm$ 0.27 & 68.1 $\pm$ 0.0 & 200.9 \\
                    90\% & Random & 8.92 $\pm$ 0.93 & 90.0 $\pm$ 0.0 & 2.8 \\
                    90\% & SNIP & 9.53 $\pm$ 0.80 & 90.0 $\pm$ 0.0 & 3.5 \\
                    90\% & GraSP & 11.29 $\pm$ 1.31 & 90.0 $\pm$ 0.0 & 3.9 \\
                    90\% & EdgePopup & 77.73 $\pm$ 0.38 & 90.0 $\pm$ 0.0 & 101.0 \\
                    90\% & IMP & 64.20 $\pm$ 6.01 & 90.0 $\pm$ 0.0 & 100.2 \\
                    90\% & SET & 33.97 $\pm$ 6.64 & 90.0 $\pm$ 0.0 & 100.0 \\
                    90\% & RigL & 17.69 $\pm$ 10.02 & 90.0 $\pm$ 0.0 & 99.3 \\
                    90\% & Movement & 76.34 $\pm$ 3.24 & 90.0 $\pm$ 0.0 & 100.9 \\
                    90\% & DoubleScore-Aug & 82.74 $\pm$ 0.04 & 81.3 $\pm$ 0.1 & 101.4 \\
                    90\% & DoubleScore-Iter & 82.09 $\pm$ 0.40 & 86.6 $\pm$ 0.1 & 794.7 \\
                    95\% & Random & 12.20 $\pm$ 2.22 & 95.0 $\pm$ 0.0 & 2.9 \\
                    95\% & SNIP & 10.06 $\pm$ 0.11 & 95.0 $\pm$ 0.0 & 3.5 \\
                    95\% & GraSP & 13.28 $\pm$ 5.53 & 95.0 $\pm$ 0.0 & 3.9 \\
                    95\% & EdgePopup & 60.61 $\pm$ 0.97 & 95.0 $\pm$ 0.0 & 101.0 \\
                    95\% & IMP & 62.11 $\pm$ 4.14 & 95.0 $\pm$ 0.0 & 99.7 \\
                    95\% & SET & 27.40 $\pm$ 9.53 & 95.0 $\pm$ 0.0 & 99.8 \\
                    95\% & RigL & 16.32 $\pm$ 8.11 & 95.0 $\pm$ 0.0 & 100.2 \\
                    95\% & Movement & 71.88 $\pm$ 7.08 & 95.0 $\pm$ 0.0 & 101.2 \\
                    95\% & DoubleScore-Aug & 78.03 $\pm$ 0.70 & 90.0 $\pm$ 0.0 & 101.5 \\
                    95\% & DoubleScore-Iter & 77.06 $\pm$ 0.76 & 90.7 $\pm$ 0.0 & 794.6 \\
                \bottomrule
            \end{tabular}
        \end{table}
        
        \begin{figure}[h]
            \centering
            \includegraphics[width=0.85\linewidth]{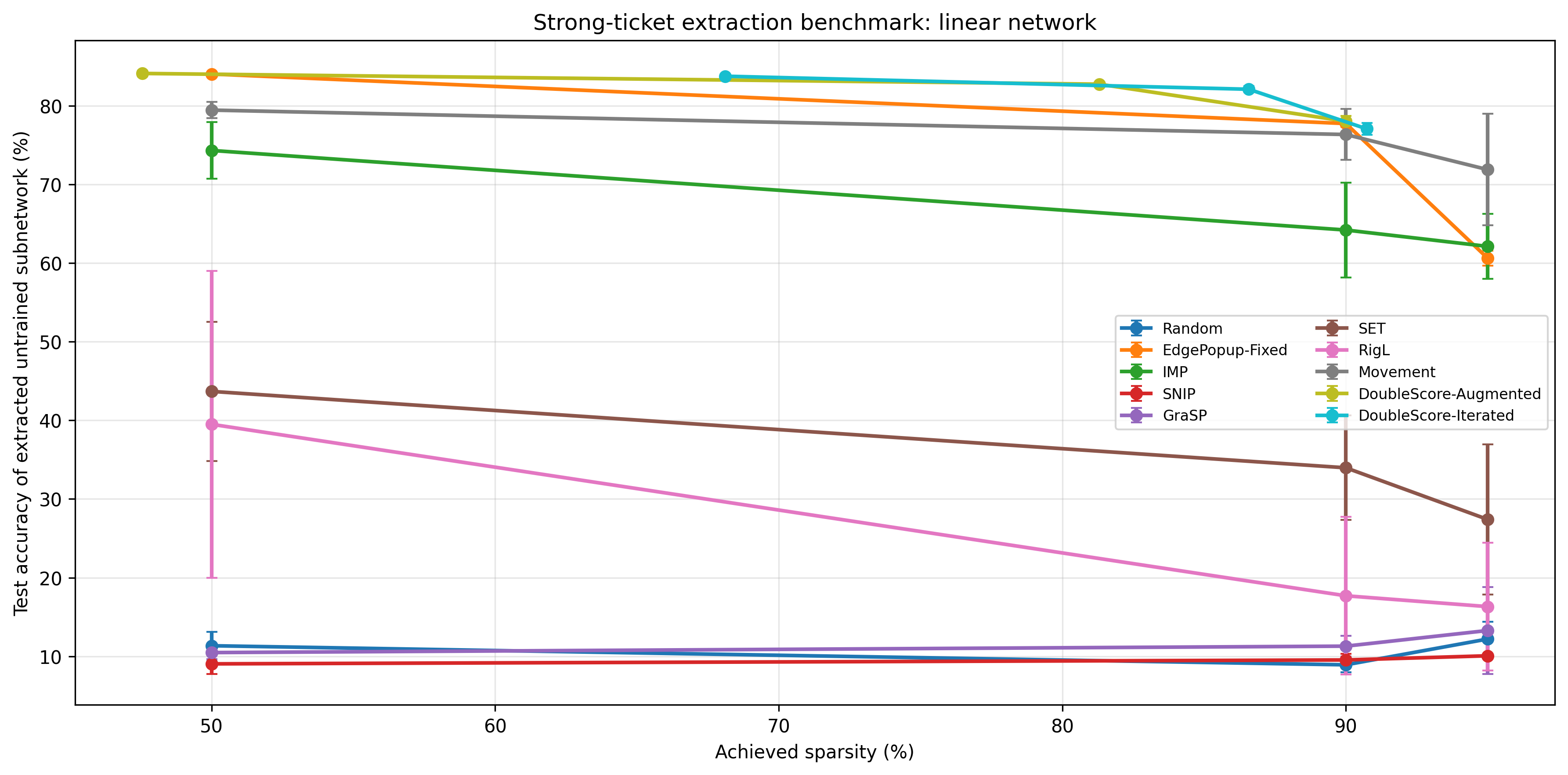}
            \caption{
                Strong-ticket extraction accuracy across requested sparsity levels.  DoubleScore-Augmented and DoubleScore-Iterated remain accurate in the high-sparsity regime where fixed-density \texttt{edge-popup} and pruning-at-initialization baselines degrade.
            }
            \label{fig:strong-ticket-full-accuracy}
        \end{figure}
        
        \begin{figure}[h]
            \centering
            \includegraphics[width=0.85\linewidth]{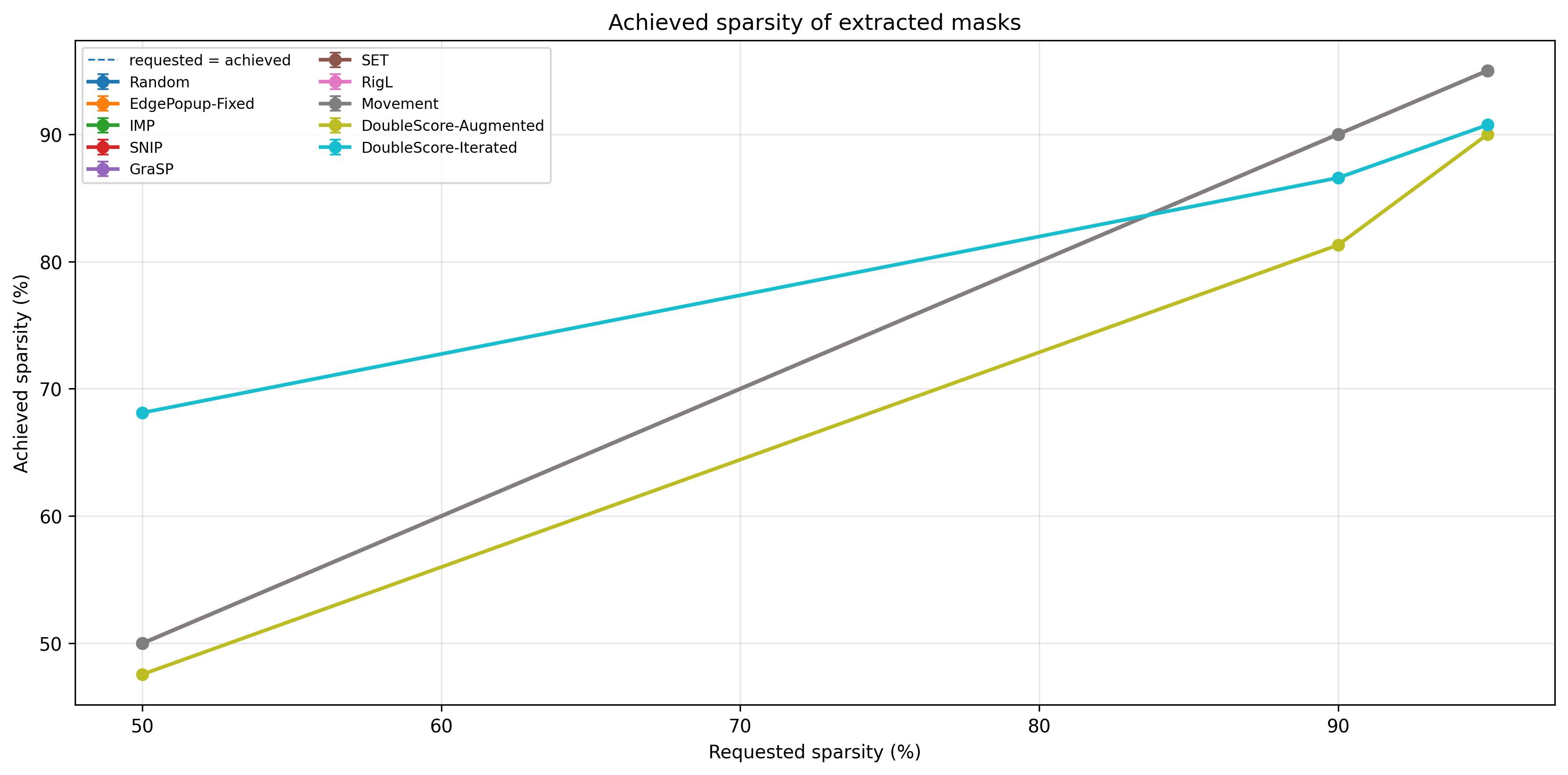}
            \caption{
                Achieved sparsity of extracted masks.  Methods that impose a fixed global density match the requested sparsity by construction. The double-score variants induce an effective original-coordinate sparsity through competition with auxiliary score coordinates, so their achieved sparsity can differ from the requested value.
            }
            \label{fig:strong-ticket-achieved-sparsity}
        \end{figure}
        
        These full results reinforce the main-paper conclusion.  At \(50\%\) requested sparsity, fixed-density \texttt{edge-popup} and DoubleScore-Augmented perform similarly.  The separation becomes more pronounced in the high-sparsity regime: at \(90\%\) and \(95\%\) requested sparsity, the double-score variants retain substantially higher strong-ticket accuracy than fixed-density \texttt{edge-popup}, random masks, SNIP, GraSP, SET, RigL, and IMP.  Movement pruning is the strongest rewound sparse-training baseline, but the double-score methods obtain higher strong-ticket accuracy in the high-sparsity settings while extracting masks with frozen weights.
    
    \section{Weak ticket training details}
    \label{weak ticket training appendix}
    
        The weak-ticket experiment asks whether the masks extracted in the strong-ticket setting remain trainable after ordinary masked weight training.  For each method, we first extract a mask, rewind the weights to the shared initialization, and then train only the surviving weights while keeping the mask fixed.  The main paper reports a dumbbell plot at \(90\%\) and \(95\%\) requested sparsity.  Here we provide the full final-accuracy, achieved-sparsity, and training-gain tables.
        
        \begin{table}[H]
        \centering
        \scriptsize
        \setlength{\tabcolsep}{3.5pt}
        \caption{
            Final test accuracy after masked weight training.  Dense denotes the fully dense baseline trained under the same optimizer settings.  Several masks that are poor strong tickets, such as random, SNIP, and GraSP masks, become trainable weak tickets after weight optimization.
        }
        \label{tab:weak-final-accuracy}
        \begin{tabular}{llllll}
            \toprule
             & 0\% requested sparse & 50\% requested sparse & 80\% requested sparse & 90\% requested sparse & 95\% requested sparse \\
            method &  &  &  &  &  \\
            \midrule
            Dense & 84.52 $\pm$ 0.42 & & & & \\
            Random & & 83.86 $\pm$ 0.78 & 83.92 $\pm$ 0.30 & 83.00 $\pm$ 0.27 & 81.61 $\pm$ 0.36 \\
            EdgePopup-Fixed & & 83.73 $\pm$ 0.21 & 83.17 $\pm$ 0.15 & 82.85 $\pm$ 0.19 & 80.96 $\pm$ 0.11 \\
            IMP & & 84.56 $\pm$ 0.20 & 83.98 $\pm$ 0.13 & 83.62 $\pm$ 0.33 & 82.67 $\pm$ 0.29 \\
            SNIP & & 84.02 $\pm$ 0.39 & 83.85 $\pm$ 0.10 & 83.22 $\pm$ 0.31 & 81.34 $\pm$ 0.93 \\
            GraSP & & 84.16 $\pm$ 0.46 & 83.81 $\pm$ 0.18 & 83.17 $\pm$ 0.36 & 81.52 $\pm$ 0.21 \\
            SET & & 84.20 $\pm$ 0.18 & 83.43 $\pm$ 0.42 & 83.35 $\pm$ 0.30 & 82.38 $\pm$ 0.27 \\
            RigL & & 84.28 $\pm$ 0.40 & 83.60 $\pm$ 0.60 & 83.15 $\pm$ 0.23 & 81.94 $\pm$ 0.23 \\
            Movement & & 84.74 $\pm$ 0.35 & 84.58 $\pm$ 0.27 & 84.44 $\pm$ 0.22 & 84.55 $\pm$ 0.18 \\
            DoubleScore-Augmented & & 84.08 $\pm$ 0.23 & 83.10 $\pm$ 0.60 & 83.44 $\pm$ 0.11 & 82.75 $\pm$ 0.09 \\
            DoubleScore-Iterated & & 83.92 $\pm$ 0.32 & 83.31 $\pm$ 0.48 & 83.50 $\pm$ 0.20 & 82.36 $\pm$ 0.24 \\
            \bottomrule
        \end{tabular}
        \end{table}
        
        \begin{table}[H]
        \centering
        \scriptsize
        \setlength{\tabcolsep}{3.5pt}
        \caption{
            Achieved sparsity for the weak-ticket experiment.  The double-score variants do not necessarily match the requested original-coordinate sparsity, since their effective sparsity is induced by auxiliary score competition.
        }
        \label{tab:weak-achieved-sparsity}
        \begin{tabular}{llllll}
            \toprule
             & 0\% requested sparse & 50\% requested sparse & 80\% requested sparse & 90\% requested sparse & 95\% requested sparse \\
            method &  &  &  &  &  \\
            \midrule
            Dense & 0.0 $\pm$ 0.0 & & & & \\
            Random & & 50.0 $\pm$ 0.0 & 80.0 $\pm$ 0.0 & 90.0 $\pm$ 0.0 & 95.0 $\pm$ 0.0 \\
            EdgePopup-Fixed & & 50.0 $\pm$ 0.0 & 80.0 $\pm$ 0.0 & 90.0 $\pm$ 0.0 & 95.0 $\pm$ 0.0 \\
            IMP & & 50.0 $\pm$ 0.0 & 80.0 $\pm$ 0.0 & 90.0 $\pm$ 0.0 & 95.0 $\pm$ 0.0 \\
            SNIP & & 50.0 $\pm$ 0.0 & 80.0 $\pm$ 0.0 & 90.0 $\pm$ 0.0 & 95.0 $\pm$ 0.0 \\
            GraSP & & 50.0 $\pm$ 0.0 & 80.0 $\pm$ 0.0 & 90.0 $\pm$ 0.0 & 95.0 $\pm$ 0.0 \\
            SET & & 50.0 $\pm$ 0.0 & 80.0 $\pm$ 0.0 & 90.0 $\pm$ 0.0 & 95.0 $\pm$ 0.0 \\
            RigL & & 50.0 $\pm$ 0.0 & 80.0 $\pm$ 0.0 & 90.0 $\pm$ 0.0 & 95.0 $\pm$ 0.0 \\
            Movement & & 50.0 $\pm$ 0.0 & 80.0 $\pm$ 0.0 & 90.0 $\pm$ 0.0 & 95.0 $\pm$ 0.0 \\
            DoubleScore-Augmented & & 47.6 $\pm$ 0.0 & 72.6 $\pm$ 0.1 & 81.3 $\pm$ 0.1 & 90.0 $\pm$ 0.0 \\
            DoubleScore-Iterated & & 68.2 $\pm$ 0.1 & 79.2 $\pm$ 0.1 & 86.7 $\pm$ 0.1 & 90.7 $\pm$ 0.0 \\
            \bottomrule
        \end{tabular}
        \end{table}
        
        \begin{table}[H]
        \centering
        \scriptsize
        \setlength{\tabcolsep}{3.5pt}
        \caption{
        Training gain in the weak-ticket experiment, measured as final test accuracy after masked training minus initial test accuracy immediately after mask extraction and rewinding.
        Large gains for random, SNIP, and GraSP masks show that these masks can be trainable weak tickets despite not being strong tickets.
        The double-score variants exhibit small gains because they already begin with high accuracy at initialization.
        }
        \label{tab:weak-training-gain}
        \begin{tabular}{llllll}
            \toprule
             & 0\% requested sparse & 50\% requested sparse & 80\% requested sparse & 90\% requested sparse & 95\% requested sparse \\
            method &  &  &  &  &  \\
            \midrule
            Dense & 74.53 $\pm$ 1.85 & & & & \\
            Random & & 72.63 $\pm$ 1.94 & 74.79 $\pm$ 2.14 & 72.15 $\pm$ 3.90 & 71.95 $\pm$ 2.52 \\
            EdgePopup-Fixed & & -0.17 $\pm$ 0.45 & -0.20 $\pm$ 0.43 & 4.80 $\pm$ 0.38 & 18.22 $\pm$ 1.59 \\
            IMP & & 8.15 $\pm$ 2.57 & 32.53 $\pm$ 6.57 & 17.23 $\pm$ 4.49 & 16.30 $\pm$ 7.18 \\
            SNIP & & 73.59 $\pm$ 2.89 & 75.55 $\pm$ 2.36 & 74.27 $\pm$ 2.27 & 72.44 $\pm$ 1.46 \\
            GraSP & & 74.35 $\pm$ 4.07 & 73.67 $\pm$ 1.62 & 72.39 $\pm$ 4.61 & 69.93 $\pm$ 2.77 \\
            SET & & 44.48 $\pm$ 4.83 & 32.85 $\pm$ 7.99 & 47.95 $\pm$ 12.62 & 61.10 $\pm$ 4.65 \\
            RigL & & 41.10 $\pm$ 10.31 & 52.31 $\pm$ 8.61 & 66.11 $\pm$ 4.43 & 73.65 $\pm$ 2.06 \\
            Movement & & 6.35 $\pm$ 1.83 & 5.69 $\pm$ 1.72 & 8.31 $\pm$ 2.15 & 12.09 $\pm$ 5.37 \\
            DoubleScore-Augmented & & 0.24 $\pm$ 0.26 & -0.52 $\pm$ 0.61 & 0.94 $\pm$ 0.29 & 5.25 $\pm$ 0.73 \\
            DoubleScore-Iterated & & -0.22 $\pm$ 0.12 & -0.50 $\pm$ 0.41 & 1.42 $\pm$ 0.26 & 5.03 $\pm$ 0.28 \\
            \bottomrule
        \end{tabular}
        \end{table}
        
        \begin{figure}[H]
            \centering
            \includegraphics[width=0.85\linewidth]{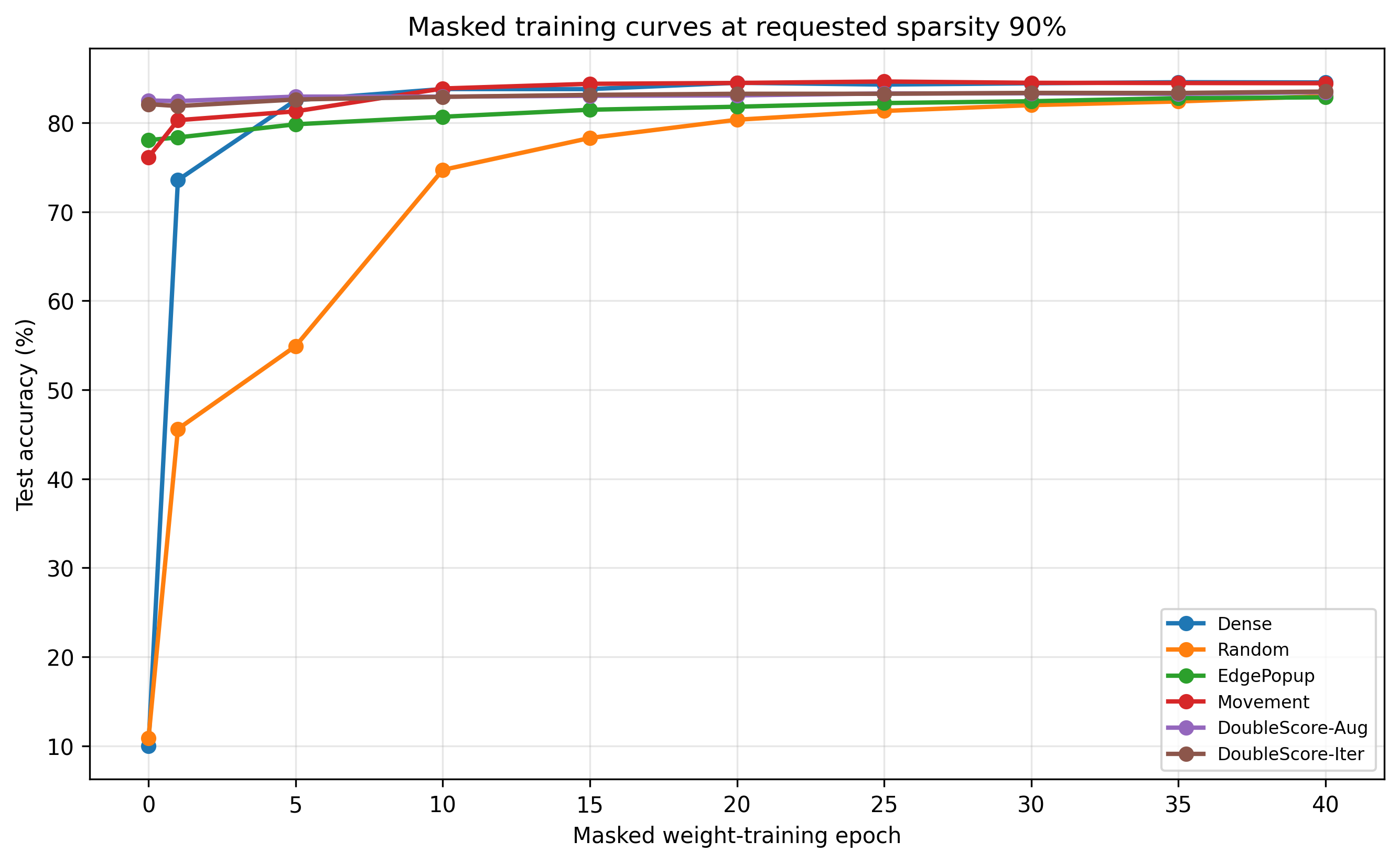}
            \caption{
            Masked weight-training curves at requested \(90\%\) sparsity.
            Double-score masks start near their final trained accuracy, while several baselines begin near chance and only become competitive after weight optimization.
            }
            \label{fig:weak-training-curves-90}
        \end{figure}
        
        The full weak-ticket results clarify the distinction between strong-ticket extraction and weak-ticket trainability.  Random, SNIP, and GraSP masks are generally poor strong tickets, but after masked weight training they can achieve final test accuracies comparable to other sparse networks.  In contrast, the double-score masks begin with high accuracy and therefore require little additional optimization to reach their final trained performance.  Movement pruning achieves the highest final trained accuracy in several settings, consistent with its role as a sparse-training method rather than a frozen-weight strong-ticket extractor.
    \section{Hyperparameter stability details}
    \label{hyperparameter stability appendix}

        The main text reports the full hyperparameter-stability comparison, including both constrained and unconstrained validation-selected \texttt{edge-popup} references.  Here we show the underlying scalar keep-density sweep used as part of that experiment.

        \begin{figure}[H]
            \centering
            \includegraphics[width=0.85\linewidth]{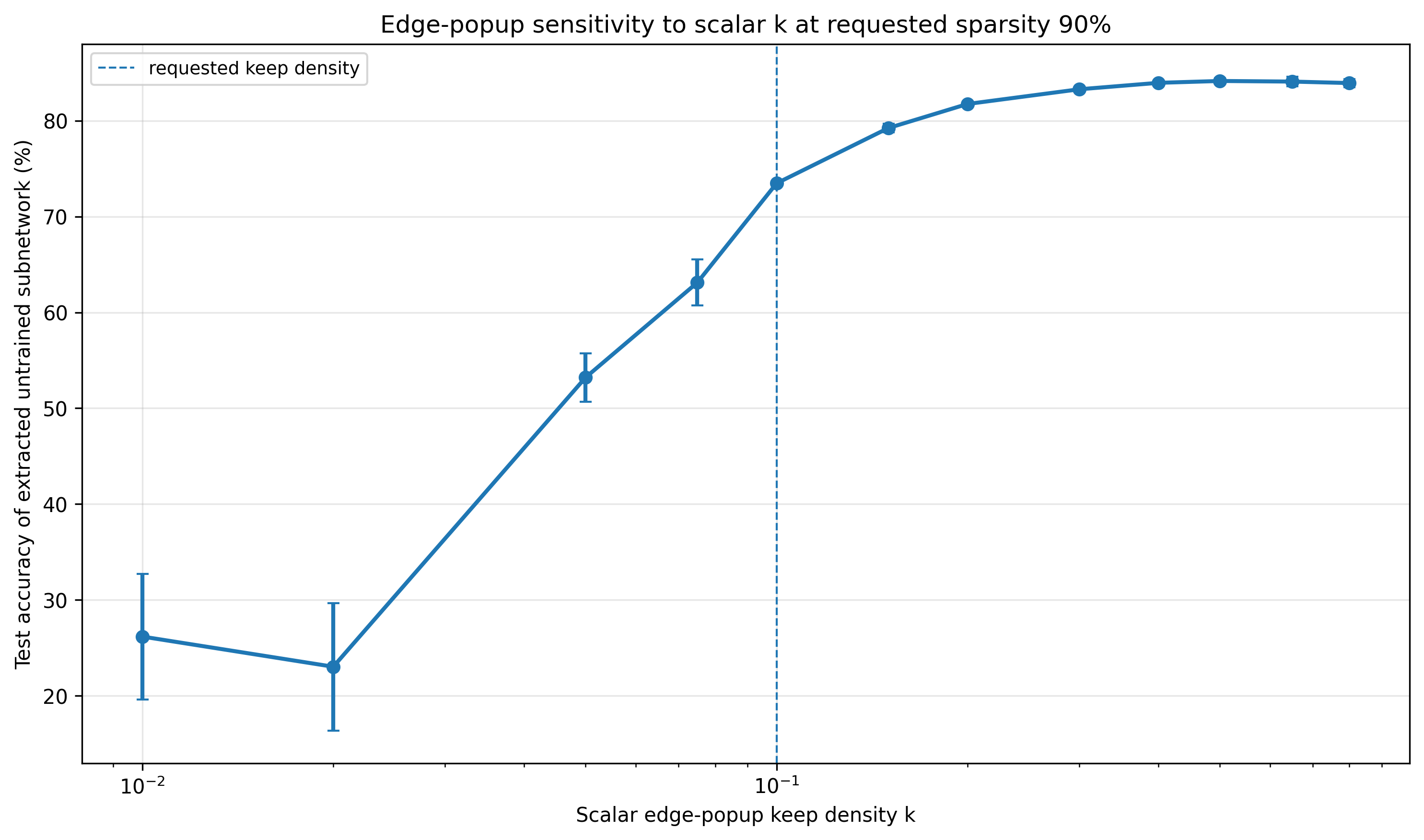}
            \caption{
                Scalar keep-density sweep for \texttt{edge-popup} at requested \(90\%\) sparsity.
                Performance varies substantially with the density parameter, illustrating that fixed-density \texttt{edge-popup} requires nontrivial sparsity selection even in this controlled benchmark.
                The full hyperparameter-stability comparison, including constrained and unconstrained validation-selected references, appears in \Cref{fig:stability experiment figure}.
                }
            \label{fig:edgepopup-k-sweep}
        \end{figure}
        
    \section{Additional ablation experiments}
    \label{additional ablations appendix}
    
        The main paper reports the ablation study at requested \(90\%\) sparsity.  \Cref{tab:ablation-appendix-all} gives the full ablation table at both \(90\%\) and \(95\%\) requested sparsity.  These results separate the effects of magnitude-based score ranking, auxiliary score competition, auxiliary width, frozen auxiliary coordinates, and final projection to exact original-coordinate sparsity.
        
        \begin{table}[H]
        \centering
        \scriptsize
        \setlength{\tabcolsep}{3.5pt}
        \caption{
        Full ablation results at requested \(90\%\) and \(95\%\) sparsity.
        SingleScore-Abs is the fixed-density \texttt{edge-popup}-style baseline.
        The augmented variants improve strong-ticket accuracy, while the projected-final variant shows that the gain is tied to the effective sparsity induced by auxiliary score competition rather than simply to better original-coordinate saliencies at a fixed density.
        }
        \label{tab:ablation-appendix-all}
        \begin{tabular}{lllllr}
            \toprule
            Requested sparsity & Variant & Test acc. (\%) & Achieved sparsity (\%) & Score params & n \\
            \midrule
            90\% & Random & 8.93 $\pm$ 4.39 & 90.0 $\pm$ 0.0 & 0.0$\times$ & 3 \\
            90\% & SingleScore-Abs & 73.88 $\pm$ 0.31 & 90.0 $\pm$ 0.0 & 1.0$\times$ & 3 \\
            90\% & SingleScore-NoAbs & 72.60 $\pm$ 0.52 & 90.0 $\pm$ 0.0 & 1.0$\times$ & 3 \\
            90\% & Aug-x1-Abs & 80.34 $\pm$ 0.42 & 82.9 $\pm$ 0.2 & 2.0$\times$ & 3 \\
            90\% & Aug-x1-FrozenAux & 79.87 $\pm$ 0.42 & 82.9 $\pm$ 0.0 & 2.0$\times$ & 3 \\
            90\% & Aug-x1-NoAbs & 78.49 $\pm$ 0.20 & 83.3 $\pm$ 0.2 & 2.0$\times$ & 3 \\
            90\% & Aug-x1-Projected & 70.16 $\pm$ 6.48 & 90.0 $\pm$ 0.0 & 2.0$\times$ & 3 \\
            90\% & Aug-x2-Abs & 80.62 $\pm$ 0.26 & 80.8 $\pm$ 0.1 & 3.0$\times$ & 3 \\
            90\% & Aug-x4-Abs & 80.93 $\pm$ 0.45 & 79.1 $\pm$ 0.2 & 5.0$\times$ & 3 \\
            95\% & Random & 9.63 $\pm$ 1.92 & 95.0 $\pm$ 0.0 & 0.0$\times$ & 3 \\
            95\% & SingleScore-Abs & 46.96 $\pm$ 4.67 & 95.0 $\pm$ 0.0 & 1.0$\times$ & 3 \\
            95\% & SingleScore-NoAbs & 44.49 $\pm$ 2.92 & 95.0 $\pm$ 0.0 & 1.0$\times$ & 3 \\
            95\% & Aug-x1-Abs & 73.75 $\pm$ 0.12 & 90.0 $\pm$ 0.0 & 2.0$\times$ & 3 \\
            95\% & Aug-x1-FrozenAux & 73.37 $\pm$ 0.58 & 90.0 $\pm$ 0.0 & 2.0$\times$ & 3 \\
            95\% & Aug-x1-NoAbs & 71.23 $\pm$ 1.94 & 90.1 $\pm$ 0.0 & 2.0$\times$ & 3 \\
            95\% & Aug-x1-Projected & 51.76 $\pm$ 7.94 & 95.0 $\pm$ 0.0 & 2.0$\times$ & 3 \\
            95\% & Aug-x2-Abs & 77.41 $\pm$ 0.30 & 86.0 $\pm$ 0.1 & 3.0$\times$ & 3 \\
            95\% & Aug-x4-Abs & 78.27 $\pm$ 0.86 & 82.8 $\pm$ 0.1 & 5.0$\times$ & 3 \\
            \bottomrule
        \end{tabular}
        \end{table}
        
        \begin{figure}[H]
            \centering
            \includegraphics[width=0.85\linewidth]{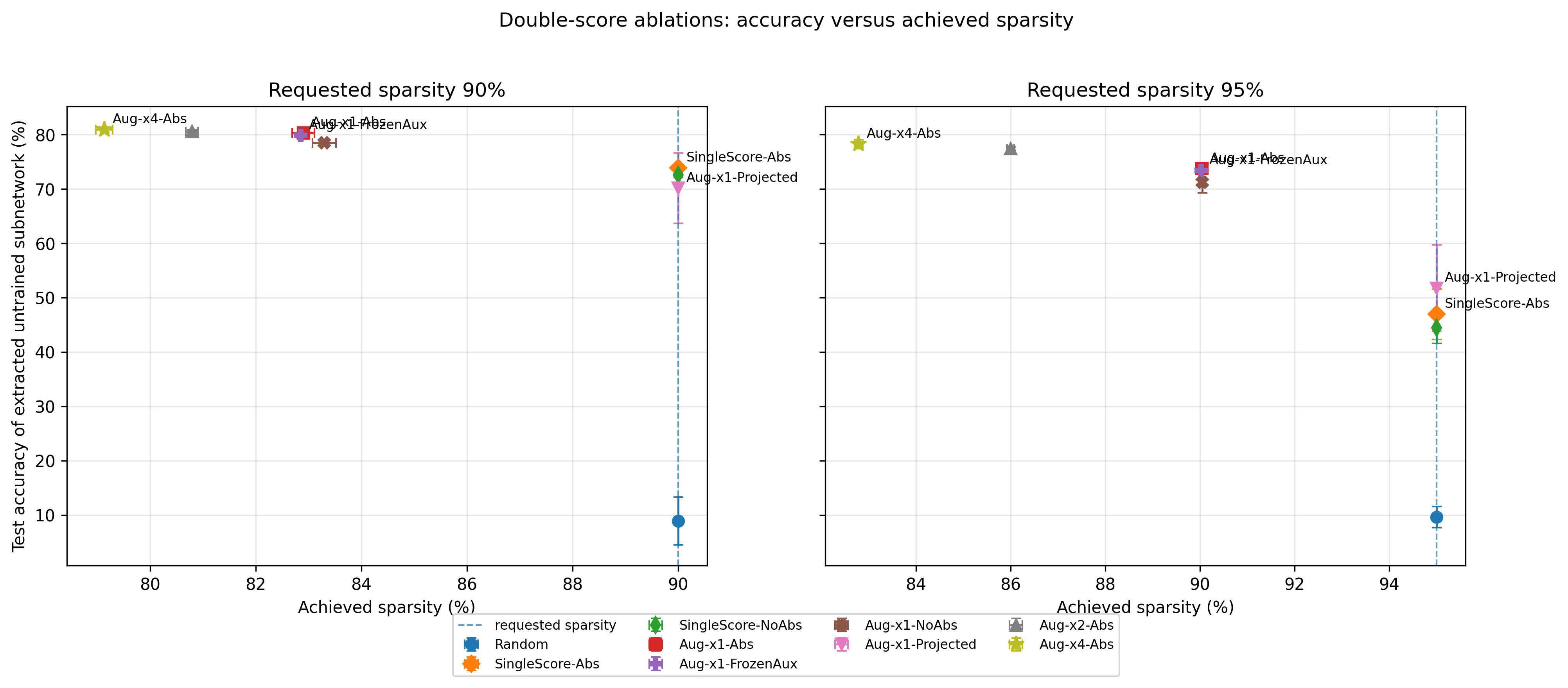}
            \caption{
            Ablation accuracy versus achieved sparsity.  The augmented variants obtain high accuracy while inducing effective sparsities below the requested original-coordinate sparsity.  The projected-final variant enforces the requested sparsity exactly and loses much of the augmented method's advantage.
            }
            \label{fig:ablation-scatter-appendix}
        \end{figure}
        
        \begin{figure}[H]
            \centering
            \includegraphics[width=0.75\linewidth]{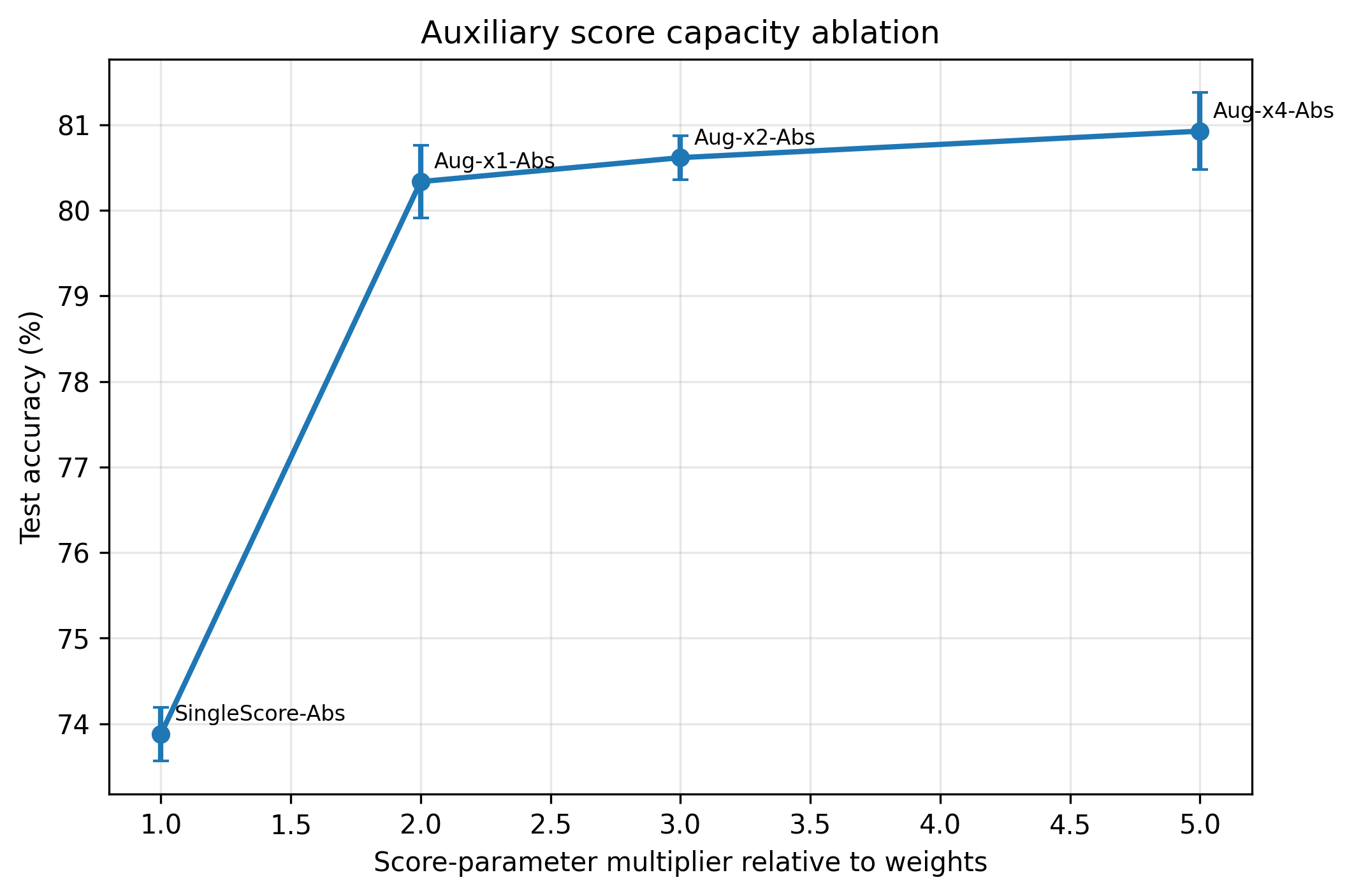}
            \caption{
            Auxiliary score-capacity ablation.  Increasing auxiliary width improves accuracy but also decreases achieved sparsity, indicating that auxiliary score capacity primarily shifts the effective sparsity/accuracy tradeoff rather than simply improving fixed-sparsity saliency learning.
            }
            \label{fig:ablation-capacity}
        \end{figure}
        
        The ablations support three conclusions.  First, magnitude-based score ranking is beneficial: removing the absolute-value parameterization lowers performance for both single-score and augmented variants.  Second, the auxiliary coordinates are not useful merely because they add trainable parameters.  The frozen-auxiliary variant performs nearly identically to the fully trainable Aug-x1-Abs variant, consistent with the interpretation that the auxiliary coordinates act as a competitive reservoir in the top-\(k\) selection.  Third, the projected-final variant shows that the augmented method's advantage is tied to induced effective sparsity.  When the final mask is forced back to exact requested original-coordinate sparsity, performance drops substantially.

    \section{Matched-sparsity CIFAR-10 ConvNet sanity check}
    \label{app:cifar-convnet-matched}
    
        To test whether the main phenomenon is specific to FashionMNIST MLPs, we ran an additional CIFAR-10 experiment using a VGG-style convolutional network without BatchNorm.  Removing BatchNorm avoids the ambiguity that high initial accuracy might come from adapted normalization statistics rather than from the extracted mask.  We use the same \(5{,}000/5{,}000\) train-validation split and the full CIFAR-10 test set.

        \paragraph{Strong ticket extraction.}
        For each seed and nominal target sparsity, we first extract a DoubleScore-Augmented mask and record its achieved sparsity.  All baselines are then run at that matched requested sparsity using the same initialization, data split, and optimizer budget.  Thus this experiment controls for the effective-sparsity difference present in the main FashionMNIST results.

        \begin{table}[H]
            \centering
            \scriptsize
            \setlength{\tabcolsep}{3.5pt}
            \caption{
                Matched-sparsity CIFAR-10 ConvNet strong-ticket extraction.
                For each seed and nominal target sparsity, DoubleScore-Augmented is run first; all other methods are run at the achieved sparsity of the DoubleScore-Augmented mask.  The architecture is a VGG-style ConvNet without BatchNorm.  All masks are evaluated after rewinding weights to the shared random initialization.
            }
            \label{tab:cifar-convnet-matched-strong}
            \begin{tabular}{lll}
                    \toprule
                     & DS nominal 90\%: acc. & DS nominal 95\%: acc. \\
                    method &  &  \\
                    \midrule
                    DoubleScore-Augmented & 46.68 $\pm$ 2.21 & 41.77 $\pm$ 0.51 \\
                    Random & 9.21 $\pm$ 1.06 & 8.97 $\pm$ 0.90 \\
                    EdgePopup-Fixed & 46.16 $\pm$ 3.23 & 42.84 $\pm$ 1.04 \\
                    IMP & 41.69 $\pm$ 2.41 & 23.39 $\pm$ 6.12 \\
                    SNIP & 10.00 $\pm$ 0.00 & 10.00 $\pm$ 0.01 \\
                    GraSP & 10.36 $\pm$ 0.62 & 10.04 $\pm$ 0.07 \\
                    SET & 10.95 $\pm$ 1.74 & 10.01 $\pm$ 0.02 \\
                    RigL & 11.67 $\pm$ 1.50 & 12.50 $\pm$ 2.55 \\
                    Movement & 36.20 $\pm$ 4.22 & 35.30 $\pm$ 1.87 \\
                    \bottomrule
                \end{tabular}
        \end{table}
        The results show that DoubleScore-Augmented remains a nontrivial strong-ticket extractor in this harder no-BatchNorm convolutional setting.  At nominal \(90\%\) sparsity, it obtains \(46.7\pm2.2\%\) test accuracy at \(83.8\%\) achieved sparsity, compared with \(46.2\pm3.2\%\) for matched fixed \texttt{edge-popup}, \(41.7\pm2.4\%\) for IMP, and \(36.2\pm4.2\%\) for movement pruning.  At nominal \(95\%\) sparsity, DoubleScore obtains \(41.8\pm0.5\%\) at \(90.9\%\) achieved sparsity, comparable to matched fixed \texttt{edge-popup} at \(42.8\pm1.0\%\) and above the other evaluated baselines.  Thus the qualitative strong-ticket behavior is not limited to the FashionMNIST MLP setting, although in this ConvNet experiment DoubleScore is best interpreted as competitive with matched \texttt{edge-popup} rather than uniformly superior to it. 

        The CIFAR-10 ConvNet result is intentionally conservative: matched \texttt{edge-popup} is given the sparsity level induced by DoubleScore-Augmented. Thus this experiment should not be read as showing that DoubleScore uniformly dominates \texttt{edge-popup}. Rather, it shows that DoubleScore remains competitive with a retrospectively well-chosen \texttt{edge-popup} density and substantially above several other baselines in a no-BatchNorm convolutional architecture.  In practice, these sparsity values are not known in advance for a given initialization, architecture, and training protocol.  As shown, \texttt{edge-popup} performance varies substantially away from these density choices.  The CIFAR-10 result therefore supports the same qualitative conclusion as the main experiments: DoubleScore does not necessarily dominate a well-chosen \texttt{edge-popup} density, but it provides a mechanism for inducing a useful effective sparsity without performing an explicit sparsity sweep.

        \paragraph{Hyperparameter stability.}
        The matched CIFAR-10 stability experiment tests whether the sparsity-selection issue persists in the no-BatchNorm ConvNet setting.  For each seed, DoubleScore-Augmented is first run at nominal \(90\%\) sparsity, and the resulting achieved sparsity is used as the matched reference sparsity for the \texttt{edge-popup} comparison.  This makes the comparison deliberately conservative: \texttt{edge-popup} is evaluated at, or near, the sparsity level induced by DoubleScore itself.
        
        \begin{table}[H]
            \centering
            \scriptsize
            \setlength{\tabcolsep}{4pt}
            \caption{
            Matched CIFAR-10 ConvNet hyperparameter-stability experiment at nominal \(90\%\) DoubleScore sparsity.
            DoubleScore-Augmented is run first, achieving approximately \(84\%\) sparsity. 
            The fixed matched \texttt{edge-popup} row uses this achieved sparsity directly. 
            The matched-band oracle selects the best validation \texttt{edge-popup} configuration whose achieved sparsity lies near the DoubleScore-induced sparsity, while the unconstrained oracle selects over the full \texttt{edge-popup} sweep.
            The unconstrained oracle attains higher accuracy by selecting much denser masks and is therefore not a same-regime comparison.
            }
            \begin{tabular}{lrrrr}
                \toprule
                Method/group & Test acc. (\%) & Range (\%) & Achieved sparsity (\%) & Runs \\
                \midrule
                DoubleScore-Augmented & 45.73 $\pm$ 1.17 & [44.94, 47.07] & 84.0 $\pm$ 0.4 & 3 \\
                Edge-popup fixed matched & 49.08 $\pm$ 1.66 & [47.58, 50.87] & 84.0 $\pm$ 0.4 & 3 \\
                Edge-popup oracle (matched band) & 50.15 $\pm$ 0.63 & [49.74, 50.87] & 83.9 $\pm$ 1.0 & 3 \\
                Edge-popup oracle (unconstrained) & 57.39 $\pm$ 0.48 & [56.96, 57.91] & 40.0 $\pm$ 8.7 & 3 \\
                Edge-popup sweep & 44.11 $\pm$ 11.61 & [9.64, 57.91] & 73.5 $\pm$ 22.9 & 51 \\
                Random masks & 10.09 $\pm$ 1.10 & [8.02, 13.85] & 84.0 $\pm$ 0.4 & 24 \\
                \bottomrule
            \end{tabular}
            \label{tab:cifar-convnet-matched-stability}
        \end{table}
        
        \begin{figure}[H]
            \centering
            \includegraphics[width=1\linewidth]{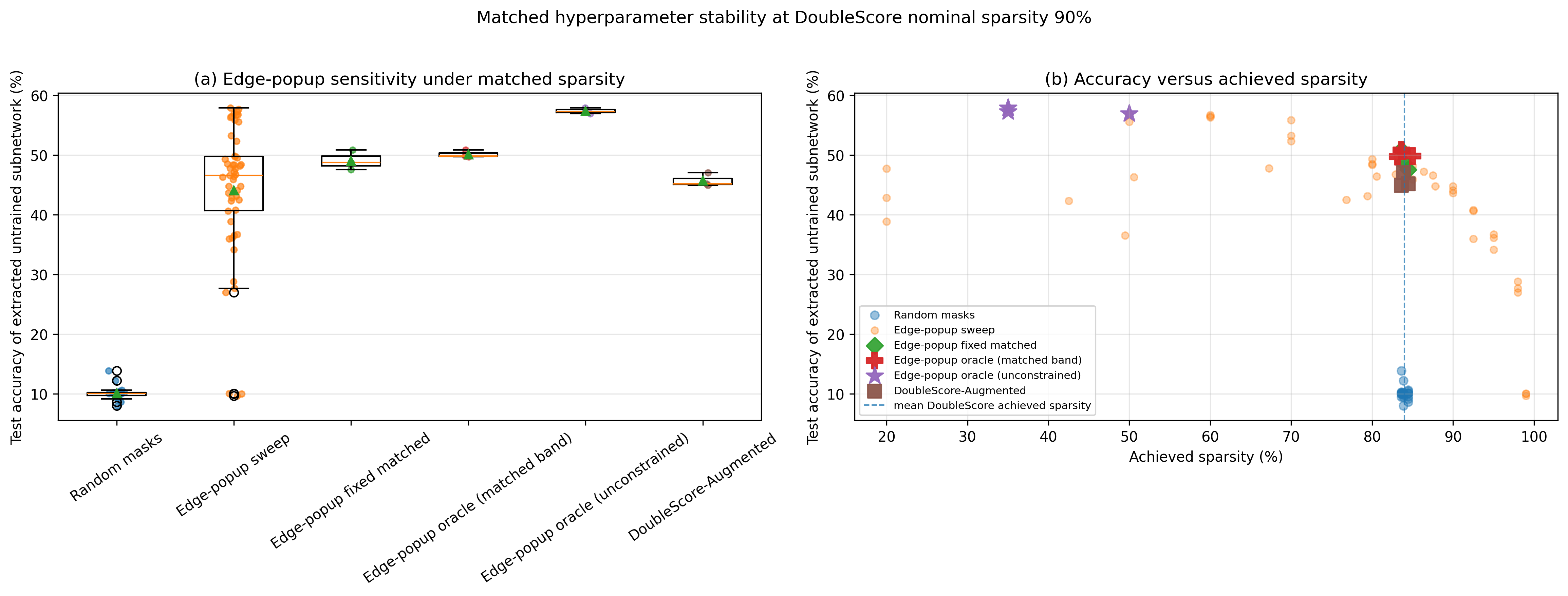}
            \caption{
            Matched CIFAR-10 ConvNet hyperparameter stability.
            Left: \texttt{edge-popup} exhibits substantial variation across scalar and layerwise keep-density choices, while DoubleScore-Augmented produces a nontrivial strong ticket without an explicit density sweep.
            Right: accuracy versus achieved sparsity. 
            Fixed matched \texttt{edge-popup} and the matched-band oracle are evaluated near the DoubleScore-induced sparsity, whereas the unconstrained oracle achieves higher accuracy by selecting substantially denser masks.
            }
            \label{fig:cifar-convnet-matched-stability}
        \end{figure}
        
        Under this retrospective matched comparison, fixed \texttt{edge-popup} reaches \(49.1\pm1.7\%\), slightly above DoubleScore-Augmented at \(45.7\pm1.2\%\), and the matched-band validation oracle reaches \(50.1\pm0.6\%\).  Thus, in this ConvNet setting, DoubleScore should not be interpreted as uniformly superior to a well-chosen \texttt{edge-popup} density.  However, the sweep also shows why the matched comparison is retrospective: \texttt{edge-popup} performance varies from near chance to \(57.9\%\) across the tested configurations, and the unconstrained oracle reaches its best performance by selecting much denser masks, with only \(40.0\pm8.7\%\) achieved sparsity.  These results support the main qualitative claim: DoubleScore induces a useful effective sparsity without performing an explicit sparsity sweep, while \texttt{edge-popup} can be highly competitive when given a good density after the fact.

        \paragraph{Ablations.}
        The matched CIFAR-10 ConvNet ablations are more nuanced than the FashionMNIST MLP ablations and shows that the advantage of DoubleScore over single-score \texttt{edge-popup} is architecture- and regime-dependent.  At nominal \(90\%\), matched SingleScore-Abs slightly outperforms Aug-x1-Abs, while at nominal \(95\%\) Aug-x1-Abs is slightly higher.  We therefore do not claim uniform dominance over a retrospectively matched \texttt{edge-popup} density.  The more robust conclusion is that DoubleScore induces useful effective sparsities without an explicit density search and remains competitive with \texttt{edge-popup} even when \texttt{edge-popup} is given the induced sparsity retrospectively.  Canonical Aug-x1-Abs remains competitive with matched single-score \texttt{edge-popup}, but does not uniformly dominate it: at nominal \(90\%\) sparsity, SingleScore-Abs slightly outperforms Aug-x1-Abs at the same achieved sparsity, while at nominal \(95\%\) Aug-x1-Abs is slightly higher.  The main mechanistic conclusion comes from the projected-final variant.  When augmented training is followed by projection back to the original-coordinate density, performance collapses, reaching only \(21.0\pm2.2\%\) at the \(90\%\) nominal setting and \(15.5\pm5.4\%\) at the \(95\%\) nominal setting.  Thus the useful effect of augmentation is not simply improved original-coordinate saliency at a fixed density; it is the effective sparsity and layerwise mask distribution induced by competition with auxiliary score coordinates.

        \begin{table}[H]
            \centering
            \scriptsize
            \setlength{\tabcolsep}{3.5pt}
            \caption{
                Matched CIFAR-10 ConvNet ablations.
                For each seed and nominal target, canonical Aug-x1-Abs is run first; all other variants are run at the achieved sparsity induced by the canonical variant.  The projected-final variant trains with augmented score competition but extracts its final mask by projecting back to the original score coordinates at the matched original-coordinate sparsity.  Its large performance drop shows that the augmented method's advantage is not merely better original-coordinate saliency learning, but the effective sparsity pattern induced by auxiliary score competition.  Auxiliary-capacity variants induce different achieved sparsities and should be interpreted as tracing a sparsity--accuracy tradeoff rather than as same-sparsity comparisons.
            }
            \label{tab:cifar-convnet-matched-ablations}
            \begin{tabular}{rllllllrr}
                \toprule
                nominal & variant & accuracy & accuracy & method & achieved & score & seconds & n \\
                target &  &  & range & requested & sparsity & params & mean &  \\
                sparsity & &   &  &  sparsity &   &(train/total)  &   &  \\
                \midrule
                0.90 & Aug-x1-Abs & 45.99 $\pm$ 1.13 & [44.69, 46.72] & 90.0 $\pm$ 0.0 & 83.8 $\pm$ 0.1 & 2.0x/2.0x & 37.104471 & 3 \\
                0.90 & Aug-x1-Abs-FrozenAux & 47.22 $\pm$ 2.78 & [44.14, 49.54] & 83.8 $\pm$ 0.1 & 76.5 $\pm$ 0.2 & 1.0x/2.0x & 36.958158 & 3 \\
                0.90 & Aug-x1-Abs-ProjectedFinal & 21.01 $\pm$ 2.21 & [19.01, 23.39] & 83.8 $\pm$ 0.1 & 83.8 $\pm$ 0.1 & 2.0x/2.0x & 36.434561 & 3 \\
                0.90 & Aug-x1-NoAbs & 46.11 $\pm$ 2.32 & [44.02, 48.60] & 83.8 $\pm$ 0.1 & 78.6 $\pm$ 0.0 & 2.0x/2.0x & 35.926176 & 3 \\
                0.90 & Aug-x2-Abs & 46.73 $\pm$ 1.00 & [45.87, 47.83] & 83.8 $\pm$ 0.1 & 73.7 $\pm$ 0.2 & 3.0x/3.0x & 36.813634 & 3 \\
                0.90 & Aug-x4-Abs & 50.60 $\pm$ 0.15 & [50.45, 50.75] & 83.8 $\pm$ 0.1 & 72.1 $\pm$ 0.1 & 5.0x/5.0x & 38.307762 & 3 \\
                0.90 & Random & 10.07 $\pm$ 0.12 & [10.00, 10.21] & 83.8 $\pm$ 0.1 & 83.8 $\pm$ 0.1 & 0.0x/0.0x & 0.847551 & 3 \\
                0.90 & SingleScore-Abs & 49.00 $\pm$ 1.39 & [47.69, 50.46] & 83.8 $\pm$ 0.1 & 83.8 $\pm$ 0.1 & 1.0x/1.0x & 35.396873 & 3 \\
                0.90 & SingleScore-NoAbs & 47.81 $\pm$ 0.55 & [47.18, 48.14] & 83.8 $\pm$ 0.1 & 83.8 $\pm$ 0.1 & 1.0x/1.0x & 34.870252 & 3 \\
                0.95 & Aug-x1-Abs & 43.48 $\pm$ 1.69 & [41.53, 44.62] & 95.0 $\pm$ 0.0 & 90.9 $\pm$ 0.1 & 2.0x/2.0x & 36.728117 & 3 \\
                0.95 & Aug-x1-Abs-FrozenAux & 47.82 $\pm$ 1.26 & [46.99, 49.27] & 90.9 $\pm$ 0.1 & 85.1 $\pm$ 0.1 & 1.0x/2.0x & 36.848718 & 3 \\
                0.95 & Aug-x1-Abs-ProjectedFinal & 15.45 $\pm$ 5.40 & [10.27, 21.05] & 90.9 $\pm$ 0.1 & 90.9 $\pm$ 0.1 & 2.0x/2.0x & 36.788965 & 3 \\
                0.95 & Aug-x1-NoAbs & 45.68 $\pm$ 1.92 & [43.55, 47.27] & 90.9 $\pm$ 0.1 & 86.1 $\pm$ 0.1 & 2.0x/2.0x & 36.180120 & 3 \\
                0.95 & Aug-x2-Abs & 46.58 $\pm$ 2.01 & [44.30, 48.12] & 90.9 $\pm$ 0.1 & 82.4 $\pm$ 0.2 & 3.0x/3.0x & 37.085452 & 3 \\
                0.95 & Aug-x4-Abs & 46.43 $\pm$ 0.43 & [46.10, 46.92] & 90.9 $\pm$ 0.1 & 80.0 $\pm$ 0.5 & 5.0x/5.0x & 38.074876 & 3 \\
                0.95 & Random & 9.73 $\pm$ 0.60 & [9.05, 10.18] & 90.9 $\pm$ 0.1 & 90.9 $\pm$ 0.1 & 0.0x/0.0x & 0.852979 & 3 \\
                0.95 & SingleScore-Abs & 42.73 $\pm$ 1.04 & [41.77, 43.84] & 90.9 $\pm$ 0.1 & 90.9 $\pm$ 0.1 & 1.0x/1.0x & 35.495703 & 3 \\
                0.95 & SingleScore-NoAbs & 40.48 $\pm$ 0.93 & [39.57, 41.42] & 90.9 $\pm$ 0.1 & 90.9 $\pm$ 0.1 & 1.0x/1.0x & 34.999686 & 3 \\
                \bottomrule
            \end{tabular}
        \end{table}
        
        The auxiliary-capacity variants further support this interpretation.  Increasing auxiliary capacity changes the induced achieved sparsity and can improve accuracy, but the resulting variants are denser on the original coordinates (this phenomenon is again noted in \Cref{extra capacity section}).  These rows should therefore be read as tracing an induced sparsity/accuracy tradeoff, not as same-sparsity comparisons.  The frozen-auxiliary variant remains strong, indicating that trainability of the auxiliary coordinates is not essential for them to function as a competitive reservoir.

\section{Additional motivation for theoretical results and proofs}
\label{proofs appendix}

    In the same paper in which Ramanujan et al.~posed the strong lottery ticket hypothesis, they also introduced an algorithm called \texttt{edge-popup} to provide empirical evidence for it. Pseudocode for the procedure is included below in \Cref{original edge popup algorithm}.  Recall that scores are ranked by magnitude; i.e. \(\operatorname{TopKMask}(S;k)\) denotes the mask retaining the largest entries of \(|S|\).
    
    \begin{algorithm}
        \caption{Classical \texttt{edge-popup}}
        \label{original edge popup algorithm}
        \begin{algorithmic}[1]\small
            \Require Frozen randomly initialized weights $\{W_t,b_t\}_{t=1}^{L}$, layerwise densities $\{k_t\}_{t=1}^{L}$, training data, loss function $\mathcal{L}$, learning rate $\eta$
            \State Initialize score tensors $\{S_t,c_t\}_{t=1}^{L}$, where $S_t$ has the same shape as $W_t$ and $c_t$ has the same shape as $b_t$
            \For{each training iteration}
                \For{$t=1,\dots,L$}
                    \State $H_t \gets \operatorname{TopKMask}(S_t;k_t)$
                    \State $h_t \gets \operatorname{TopKMask}(c_t;k_t)$
                    \State $\widetilde{W}_t \gets W_t \odot H_t$
                    \State $\widetilde{b}_t \gets b_t \odot h_t$
                \EndFor
                \State Compute the network output using $\{(\widetilde{W}_t,\widetilde{b}_t)\}_{t=1}^{L}$
                \State Compute loss $\mathcal{L}$
                \State Update the scores by gradient descent with straight-through gradients:
                \[
                S_t \gets S_t - \eta \nabla_{S_t}\mathcal{L},
                \qquad
                c_t \gets c_t - \eta \nabla_{c_t}\mathcal{L}
                \]
            \EndFor
            \State \Return Final masks $\{H_t,h_t\}_{t=1}^{L}$
        \end{algorithmic}
    \end{algorithm}
    
    Ramanujan et al.~proved a useful local monotonicity statement for this procedure: under appropriate smoothness assumptions, if one swaps in a nonzero number of edges in a single layer while keeping the remainder of the network fixed, then the minibatch loss decreases.

    \begin{theorem}[Ramanujan et al.~\cite{ramanujan}]
        \label{decreasing loss edge popup}
        When a nonzero number of edges are swapped in one layer by a straight-through gradients update in \Cref{original edge popup algorithm} and the rest of the network remains fixed, the minibatch loss decreases, provided the loss is sufficiently smooth.
    \end{theorem}
    
    This makes \texttt{edge-popup} quite attractive in principle. The difficulty, as discussed in the main paper, is choosing layerwise densities $k_1,\ldots,k_\ell$.  We illustrate this instability with a toy experiment in \Cref{fig:k-bottleneck}. Performance depends sharply on the density choice, both globally and layerwise, and repeated runs do not suggest any simple universal rule for choosing $k$. This is the bottleneck that prevents \texttt{edge-popup} from fully achieving its intended aim.
    \begin{figure}[t]
        \centering
        \includegraphics[width=0.8\linewidth]{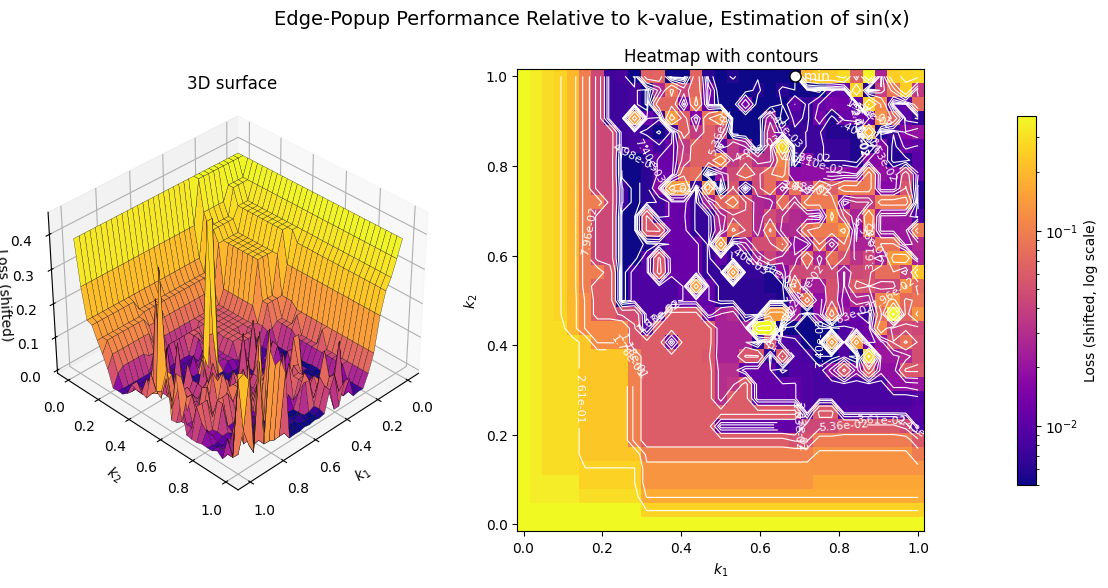}
        \caption{The $k$-selection bottleneck:  A network with a single width 32 hidden layer was trained using \texttt{edge-popup} to approximate a sine function across layerwise density tuples.  Left: The mean-squared errors of the resulting masked networks are plotted above over a two-layer grid of the possible layerwise densities.  Right: A heat map of the same figure with contour lines.  As is apparent, the landscape is jagged, highly non-uniform, and has little to no discernible structure.}
        \label{fig:k-bottleneck}
    \end{figure}
    
    On toy networks, one can brute-force all layerwise density choices and, when one does so, \texttt{edge-popup} often performs very well. \Cref{fig:toy-sine} shows the issue on small, synthetic problems where exhaustive search is actually feasible. Thus, the problem is not that \texttt{edge-popup} is incapable of finding strong masks. Rather, the problem is that the search over $k$ becomes prohibitive extremely quickly.
    
    \begin{figure}[t]
        \centering
        \includegraphics[width=1.0\linewidth]{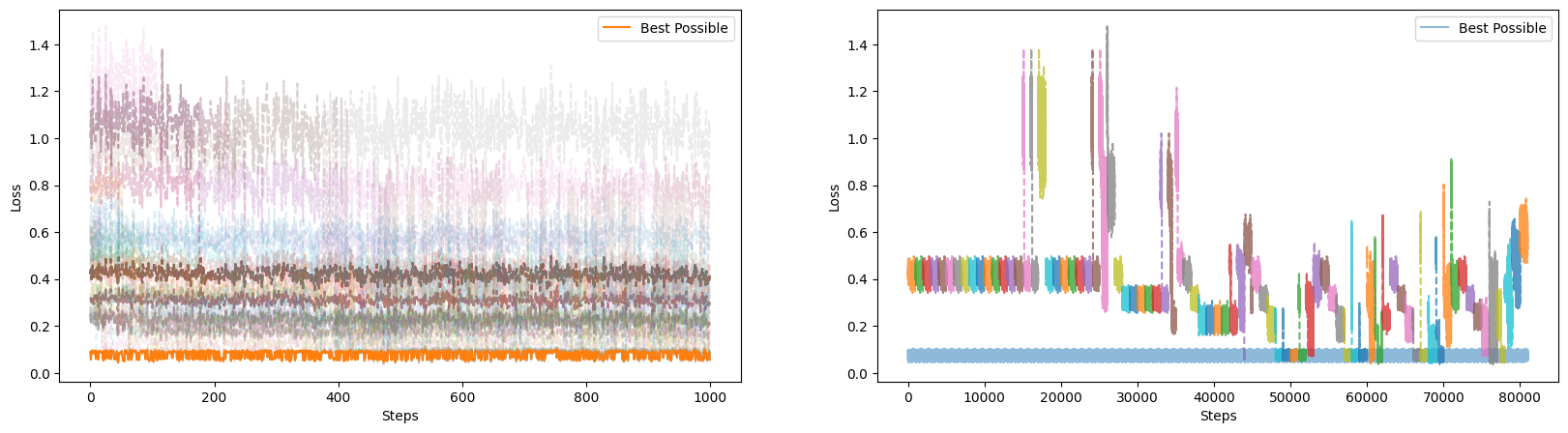}
        \caption{A network with a single width 8 hidden layer was trained using \texttt{edge-popup} to approximate a sine function across all layerwise density tuples.  The loss curves are compared to the loss of the best possible mask, found via an exhaustive search of all possible masks.  In the left chart, all loss curves are overlaid.  In the right, each is plotted sequentially.  This reveals that \texttt{edge-popup} can match the best brute-force mask on a genuinely tiny search space, but also reveals that the dependence on $k$-tuples is highly nonuniform.}
        \label{fig:toy-sine}
    \end{figure}
    
    Despite the difficulties associated with choosing a good value of $k$, there is one important regime in which the sparsity-selection problem largely disappears: the case of already sparse weight tensors. That is, for weights that are already equal to zero, the corresponding mask entry has no effect in terms of the resulting subnetwork.
    
    \begin{example}
        Observe that the following two masks are distinct (as they differ in the entry in the second row and second column) but produce the same result when applied to the weight matrix below:
        $$
            \begin{bmatrix}
                0 & 1  \\
                1 & \textcolor{red}{0} \\
            \end{bmatrix} \odot \begin{bmatrix}
                w_{1,1} & w_{1,2} \\
                w_{2,1} & 0 \\
            \end{bmatrix}
            =\begin{bmatrix}
                0 & w_{1,2} \\
                w_{2,1} & 0 \\
            \end{bmatrix}=\begin{bmatrix}
                0 & 1  \\
                1 & \textcolor{red}{1} \\
            \end{bmatrix}\odot \begin{bmatrix}
                w_{1,1} & w_{1,2} \\
                w_{2,1} & 0 \\
            \end{bmatrix}
        $$
        
        Thus distinct masks can represent the same masked tensor whenever they differ only on locations where the weights are already zero.
    \end{example}
    
    This suggests a natural modification of \texttt{edge-popup}. This is not technically needed for the basic \texttt{double-scoring} construction; but instead it records a quantitative selector-mask version of the same principle.  The core idea is that if many entries of a weight tensor are already zero, then there is no need to assign scores to all coordinates. Rather, one need only score the entries whose mask values actually matter. More generally, there are situations in which one may wish to apply \texttt{edge-popup} only to a specified subset of the weights. For instance, one might wish to preserve a distinguished part of the tensor, or enforce that masking respects some additional structural constraint such as symmetry. This motivates the following selector-mask version of \texttt{edge-popup} (\Cref{alg:network-selective-edge-popup}).

    \begin{algorithm}[ht]
        \caption{Network-level \texttt{edge-popup} with scoring selection}
        \label{alg:network-selective-edge-popup}
        \begin{algorithmic}[1]\small
            \Require Frozen network parameters $\{W_t,b_t\}_{t=1}^{L}$, selector masks $\{Q_t,q_t\}_{t=1}^{L}$, density parameters $\{k_t\}_{t=1}^{L}$, training data, loss function $\mathcal{L}$, learning rate $\eta$
            \For{$t=1,\dots,L$}
                \State Let $I_t = \{ \alpha : (Q_t)_\alpha = 1\}$ and $J_t = \{ \beta : (q_t)_\beta = 1\}$
                \State Initialize score vectors $s_t \in \mathbb{R}^{|I_t|}$ and $c_t \in \mathbb{R}^{|J_t|}$
            \EndFor
            \For{each training iteration}
                \For{$t=1,\dots,L$}
                    \State $u_t \gets \operatorname{TopKMask}(s_t;k_t)$
                    \State $v_t \gets \operatorname{TopKMask}(c_t;k_t)$
                    \State Extend $u_t$ to a full mask $H_t$ on $W_t$ using selector mask $Q_t$
                    \State Extend $v_t$ to a full mask $h_t$ on $b_t$ using selector mask $q_t$
                    \State $\widetilde{W}_t \gets W_t \odot H_t$
                    \State $\widetilde{b}_t \gets b_t \odot h_t$
                \EndFor
                \State Compute the network output using $\{(\widetilde{W}_t,\widetilde{b}_t)\}_{t=1}^{L}$
                \State Compute loss $\mathcal{L}$
                \For{$t=1,\dots,L$}
                    \State Update $s_t$ and $c_t$ by gradient descent with straight-through gradients
                \EndFor
            \EndFor
            \State \Return Final masks $\{H_t,h_t\}_{t=1}^{L}$
        \end{algorithmic}
    \end{algorithm}
    
    The point of \Cref{alg:network-selective-edge-popup} is that it decouples the set of coordinates that receive popup scores from the full support of the weight tensor. In particular, when a layer is already sparse, one may score all nonzero entries and only some of the zero entries, thereby shrinking the effective search space without losing expressive power.
    
    We first record a quantitative version of this principle.
    
    \begin{theorem}
    \label{thm:quantitative-half-density}
        Let $M$ be a weight tensor with $n$ entries and sparsity
        \[
        \frac{1}{2}-\frac{1}{n} \leq s \leq \frac{1}{2}-\frac{2}{n}.
        \]
        Suppose one applies \Cref{alg:network-selective-edge-popup} to $M$ by assigning popup scores to all nonzero entries of $M$ and to any additional $n(1-s)-2$ zero entries of $M$. Then, by fixing the target mask density at $k=\frac{1}{2}$, one can represent every function obtainable by a nontrivial masking of $M$.
    \end{theorem}
    \begin{proof}
        Let $M_0=\{i:M_i=0\}$ and $M_0^c=\{i:M_i\neq 0\}$, so that $|M_0|=sn$ and $|M_0^c|=n(1-s)$.  Suppose that popup scores are assigned to all nonzero entries of $M$ and to $t$ zero entries of $M$. Write $z:=\frac{t}{|M_0|}$, so that $sz=t/n$ is the proportion of tracked zero entries relative to the total size of $M$. The total number of scored entries is therefore
        \[
        n(1-s)+nsz=n\bigl(1-s(1-z)\bigr).
        \]
        
        Fix a density parameter $k$, and suppose the resulting mask retains $m$ of the scored entries. Then
        \[
        m=kn\bigl(1-s(1-z)\bigr).
        \]
        Write $m=m_1+m_2$, where $m_1$ is the number of retained nonzero entries and $m_2$ is the number of retained scored zero entries. If
        \[
        k_1:=\frac{m_1}{n(1-s)}
        \qquad\text{and}\qquad
        k_2:=\frac{m_2}{nsz},
        \]
        then
        \[
        k_1n(1-s)+k_2nsz=kn\bigl(1-s(1-z)\bigr).
        \]
        Since $0\leq k_2\leq 1$, it follows that
        \[
        \frac{k\bigl(1-s(1-z)\bigr)-sz}{1-s}\leq k_1\leq \frac{k\bigl(1-s(1-z)\bigr)}{1-s}.
        \]
        Thus, for fixed $s$, $z$, and $k$, the proportion $k_1$ of retained nonzero entries ranges over an interval of width $\frac{sz}{1-s}$.  Now, choose $k$ so that this interval is centered at $\frac{1}{2}$. One convenient choice is
        \[
        k=\frac{1-s+sz}{2\bigl(1-s(1-z)\bigr)}=\frac{1}{2},
        \]
        which is the midpoint of the admissible interval above. With this choice, $k_1$ ranges over an interval centered at $\frac12$ of width $\frac{sz}{1-s}$.
        
        The nontrivial masking proportions on the nonzero entries are precisely
        \[
        \frac{1}{n(1-s)},\frac{2}{n(1-s)},\dots,\frac{n(1-s)-1}{n(1-s)}.
        \]
        Hence, it is enough for the interval of possible $k_1$ values to contain all of these, which will occur provided its width is $1-\frac{2}{n(1-s)}$.  Thus, we impose
        \[
        \frac{sz}{1-s}=1-\frac{2}{n(1-s)}.
        \]
        Solving for $z$ gives
        \[
        z=\frac{n(1-s)-2}{ns}.
        \]
        Equivalently, the number of zero entries that must be assigned scores is
        \[
        t=|M_0|z=sn\cdot \frac{n(1-s)-2}{ns}=n(1-s)-2.
        \]
        Therefore, by tracking scores on all nonzero entries of $M$ and on any additional $n(1-s)-2$ zero entries, the selector-mask version of \texttt{edge-popup} at fixed density $k=\frac12$ can realize every nontrivial masking pattern on the nonzero entries of $M$. Since the zero entries of $M$ do not affect the represented function, this proves the claim.
    \end{proof}
    
    \begin{remark}
        If one also wishes to include the two trivial masks (the all-zero mask and the all-one mask on the nonzero support), then it is enough to score $n(1-s)$ zero entries rather than $n(1-s)-2$, provided $s\geq \frac{1}{2}$.  Since this is wasteful in practice (one can simply check these trivial masks without any score training), this is likely undesirable.
    \end{remark}
    
    The content of \Cref{thm:quantitative-half-density} is that, on a sufficiently sparse tensor, selector-based scoring allows one to simulate the full family of nontrivial masking levels while keeping the target density fixed at $k=\frac{1}{2}$. Thus, the apparent freedom in the density parameter is, in this regime, largely illusory: one may trade a search over densities for a suitable choice of scored coordinates.  This has a particularly clean consequence when one is interested only in representability rather than in the bookkeeping of scored zero coordinates.  This is the content of \Cref{prop:half-density} in the main paper, the proof of which is included below.

    \paragraph{Proof of \Cref{prop:half-density}}
    \begin{proof} 
        Let $P$ be the set of indices at which $W$ is nonzero, and let $Z$ be the set of indices at which $W$ is zero. Since at least half of the entries of $W$ vanish, we have $|Z|\ge \lceil d/2 \rceil \geq \lfloor d/2 \rfloor=r$, and, hence, $|P|=d-|Z| \le d-\lceil d/2 \rceil=\lfloor d/2 \rfloor=r$.
        
        Let $M \in \{0,1\}^{m\times n}$ and define $M^\ast$ as follows. First, set $M^\ast_{ij}=M_{ij}$ for all $(i,j)\in P$.  This guarantees that $W\odot M$ and $W\odot M^\ast$ agree on every nonzero entry of $W$. Now add ones to $M^\ast$ on enough indices from $Z$ so that $M^\ast$ has exactly $r$ ones in total, and set all remaining entries on $Z$ equal to zero.
        
        This construction is always possible because the number of ones already prescribed on $P$ is at most $|P|\le r$, while there are at least $r$ indices available in $Z$. Moreover, every additional one placed on $Z$ multiplies a zero entry of $W$, and therefore does not change the masked tensor. Hence, $W \odot M = W \odot M^\ast$.
    \end{proof}
    The layerwise version of this then follows in \Cref{cor:network-half-density}.
    
    \paragraph{Proof of \Cref{cor:network-half-density}}
    \begin{proof}
        Apply \Cref{prop:half-density} independently to each layer.
    \end{proof}
    Applied layerwise, \Cref{cor:network-half-density} says that once each layer is at least half sparse, the entire search space of masked subnetworks can be represented using a single fixed density near $\frac12$ in every layer, thereby removing the principal \emph{expressive} obstruction.
    \begin{example}
        Simple sparse-network experiments line up with this picture (\Cref{fig:sparse-clustering}). When the base model is initialized with substantial sparsity, good performance tends to cluster near the half-density regime rather than requiring a finely tuned layerwise search.  The comparison is particularly stark when contrasted with \Cref{fig:k-bottleneck}:  in the sparse regime, the test accuracy landscape is substantially smoother with respect to variations in the layerwise $k$ vector.
        
        \begin{figure}[t]
            \centering
            \includegraphics[width=0.4\linewidth]{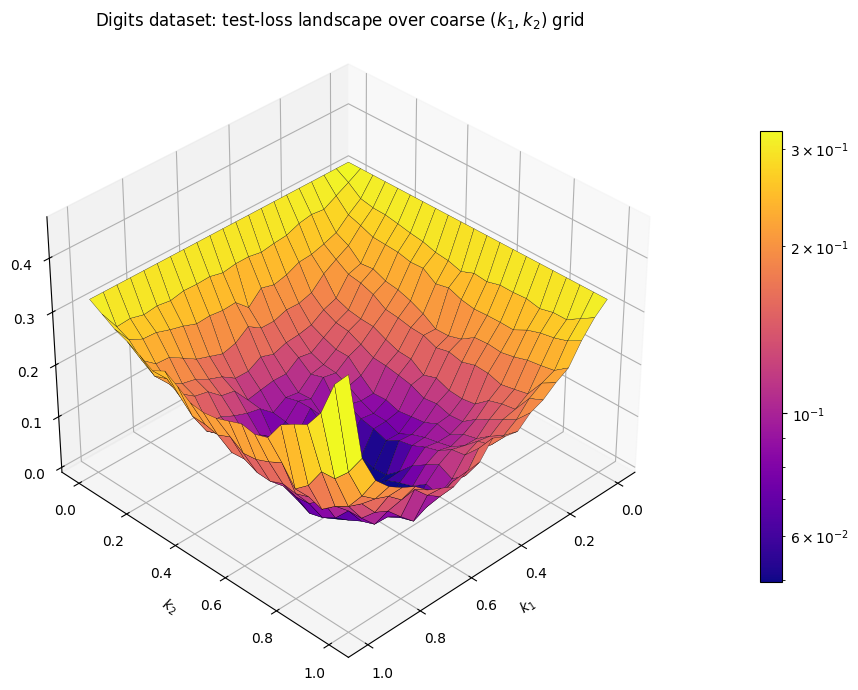}         \includegraphics[width=0.59\linewidth]{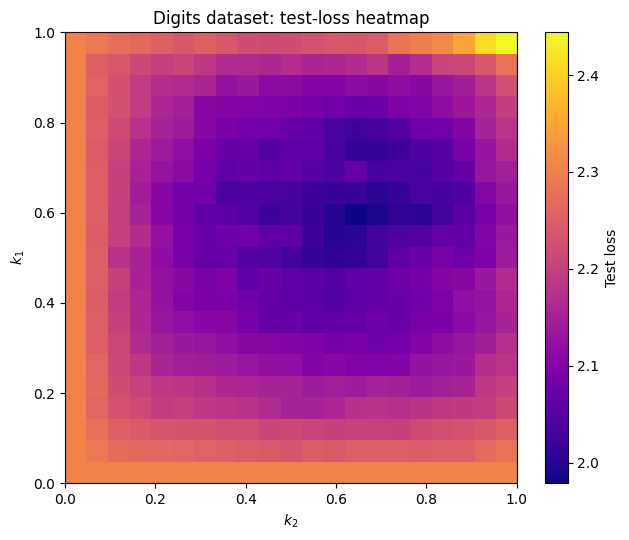}
            \caption{Left: test accuracy as a function of $k$-tuples for a 50\% sparsity two-layer network trained on the scikit-learn digits dataset.  The best performance begins to appear around the center of each $k$ interval.   Right: A heat map of test  losses for the same network across $k$-tuples.  The best-performing region concentrates near the fixed-density regime suggested by \Cref{prop:half-density}, consistent with the claim that the sparse regime reduces the effective burden of layerwise density selection.}
            \label{fig:sparse-clustering}
        \end{figure}
    \end{example}

    These observations lead naturally to our implementation of \texttt{double-scoring}, which treats a network as being synthetically embedded in a larger dense network.  This is enough by \Cref{prop:augmented-space}, the proof of which is given below.

    \paragraph{Proof of \Cref{prop:augmented-space}}
    \begin{proof}
    
        Let $\widehat{m}$ agree with $m$ on the first $d$ coordinates. If $m$ contains $q$ ones, then add exactly $d-q$ ones among the final $d$ coordinates of $\widehat{m}$ and set all remaining auxiliary coordinates to zero. Since the final $d$ coordinates of $\widehat{w}$ are zero, these additional ones do not affect the masked vector. By construction, $\widehat{m}$ has exactly $d$ ones and restricts to $m$ on the original coordinates.
    
    \end{proof}
    
    This is the key representational idea: by enlarging the \emph{score space}, a dense layer can be treated as though it were embedded inside a half-sparse augmented layer, without ever modifying the underlying weights.  This yields \texttt{double-scoring}.  The comparison of \texttt{double-scoring} to the embedded network in \Cref{prop:augmented-space} is not merely representationally equivalent to \texttt{edge-popup} on an augmented half-sparse system; the optimization dynamics agree exactly as well.
    
    \begin{theorem}[\texttt{double-scoring} as augmented \texttt{edge-popup}]
        \label{thm:double-scoring-equivalence}
        Fix a feed-forward network with frozen weights $\{W_t,b_t\}_{t=1}^L$. For each layer $t$, let
        \[
        \widehat{W}_t := (W_t,0)
        \qquad\text{and}\qquad
        \widehat{b}_t := (b_t,0),
        \]
        where the second half consists of zero entries and the concatenation is taken after flattening the tensors. Likewise, let
        \[
        \widehat{S}_t := (S_t,T_t)
        \qquad\text{and}\qquad
        \widehat{f}_t := (f_t,g_t)
        \]
        be the corresponding doubled score tensors.
        
        Consider the following two procedures, both run at density $1/2$:
        \begin{enumerate}
            \item classical \texttt{edge-popup} applied to the augmented tensors $(\widehat{W}_t,\widehat{b}_t)$ with scores $(\widehat{S}_t,\widehat{h}_t)$;
            \item \texttt{double-scoring} applied to the original tensors $(W_t,b_t)$ with score pairs $(S_t,T_t)$ and $(f_t,g_t)$, where the effective masks are obtained by restricting $\operatorname{TopKMask}(\widehat{S}_t;1/2)$ and $\operatorname{TopKMask}(\widehat{f}_t;1/2)$ to the original coordinates.
        \end{enumerate}
        Assume that the backward pass for $\operatorname{TopKMask}$ uses the usual straight-through estimator, i.e. it acts as the identity on gradients.
        
        Then, at every training iteration:
        \begin{enumerate}
            \item the masked parameters used by the two procedures agree on the original coordinates, and hence the network outputs and losses are identical;
            \item the gradients of the original score coordinates agree in the two procedures;
            \item the gradients of the auxiliary score coordinates are zero in both procedures.
        \end{enumerate}
        Consequently, after restricting to the original coordinates, \texttt{double-scoring} produces exactly the same score iterates, masks, network outputs, and losses as classical \texttt{edge-popup} on the augmented zero-padded network.
    \end{theorem}
    \begin{proof}
        We compare gradients first with respect to the magnitude scores \(A_t=|\widehat S_t|\). Since the \texttt{double-scoring} loss depends on \(\widehat H_t\) only through its restriction \(H_t\) to the original coordinates, its gradient with respect to \(\widehat H_t\) has the form
        \[
            \frac{\partial \mathcal L}{\partial \widehat H_t} = \left(\frac{\partial \mathcal L}{\partial H_t},0\right).
        \]
        Under the straight-through estimator for the hard top-\(k\) map, this is also the gradient with respect to the magnitude scores \(A_t\). Thus the auxiliary
        magnitude-score coordinates receive zero gradient.
        
        For classical \texttt{edge-popup} on the augmented zero-padded tensor, the loss depends on \(\widehat H_t\) only through
        \[
            \widehat W_t\odot \widehat H_t=(W_t\odot H_t,0).
        \]
        Consequently,
        \[
            \frac{\partial \mathcal L}{\partial \widehat H_t}=\frac{\partial \mathcal L}{\partial(\widehat W_t\odot \widehat H_t)} \odot \widehat W_t = \left( \frac{\partial \mathcal L}{\partial(W_t\odot H_t)}\odot W_t,0\right),
        \]
        which is the same gradient obtained in the restricted \texttt{double-scoring} formulation. Finally, gradients with respect to the raw scores are obtained by multiplying by \(\operatorname{sign}(\widehat S_t)\), away from zero. Since the raw scores are the same in the two procedures, this sign factor is identical, and the raw-score gradients agree as well.  The argument for the bias terms is analogous.
    \end{proof}
    \begin{remark}
        In the hard restricted-mask formulation above, the auxiliary score coordinates participate in the top-\(k\) competition but receive zero gradient, because the corresponding augmented weights are identically zero. Thus, with identical initialization, deterministic optimizers, and zero score weight decay, freezing the auxiliary scores gives exactly the same iterates as treating them as trainable parameters. The frozen-auxiliary ablation in \Cref{ablations subsection} is consistent with this prediction: although run as a separately initialized variant, it performs nearly identically to the fully parameterized augmented version.
    \end{remark}
    This allows us to directly inherit the local monotonicity result (\Cref{decreasing loss edge popup}) obtained by Ramanujan et al.
    \begin{corollary}
        \label{cor:double-scoring-monotonicity}
        Any local monotonicity result (e.g. \Cref{decreasing loss edge popup}) for classical \texttt{edge-popup} at density $1/2$ applies verbatim to \texttt{double-scoring} via the augmented zero-padded realization of \Cref{thm:double-scoring-equivalence}.
    \end{corollary}
    \begin{proof}
        By \Cref{thm:double-scoring-equivalence}, \texttt{double-scoring} is exactly classical \texttt{edge-popup} on the augmented zero-padded network after restriction to the original coordinates. Hence any one-step loss improvement statement for classical \texttt{edge-popup} transfers immediately.
    \end{proof}
    
    \begin{example}
        To demonstrate with a toy example, the effect of \texttt{double-scoring} is illustrated in \Cref{fig:double-score-compare-1k}. The fixed-density procedure tracks the outcome of an exhaustive search over layerwise densities without requiring separate runs for each choice of $(k_1,\dots,k_L)$. In particular, it recovers performance comparable to the best density configuration while avoiding the associated combinatorial cost.
        
        \begin{figure}[t]
            \centering
            \includegraphics[width=0.5\linewidth]{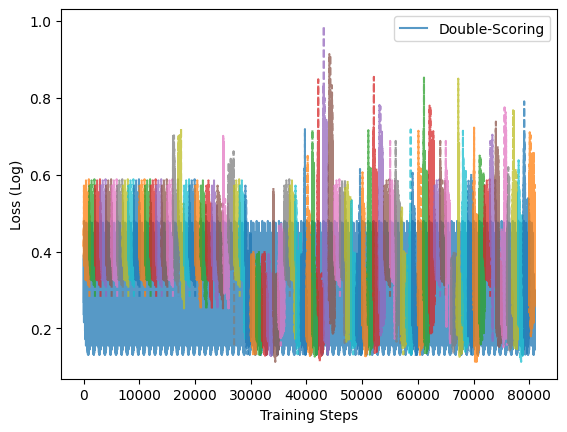}
            \caption{Comparison between the proposed fixed-density method and brute-force layerwise density search, with the loss curves presented sequentially.  The loss curve for \texttt{double-scoring} is repeated for each layerwise density tuple for ease of comparison (albeit at the cost of some visual clarity). As is apparent, the \texttt{double-scoring} procedure rapidly reaches the same level of loss as \texttt{edge-popup} with the best $k$-tuple choices over an exhaustive search.}
            \label{fig:double-score-compare-1k}
        \end{figure}
    \end{example}
    
    \begin{remark}
        It is worth emphasizing that \texttt{double-scoring} builds directly on the \texttt{edge-popup} algorithm of Ramanujan et al.~\cite{ramanujan}. The underlying optimization procedure remains essentially unchanged; the modification consists only in enlarging the score space while keeping the weight tensors fixed. 

        The primary contribution is therefore a reparameterization of the mask-search problem: the underlying \texttt{edge-popup} optimization is left essentially unchanged, while the score space is enlarged so that the effective layerwise sparsity can emerge from the top-\(k\) competition.
    \end{remark}

    \begin{remark}
        We note also that there are a number of obvious modifications to \Cref{double-scoring algo} that produce functionally identical results.  Some of these are minor (e.g., instead of storing the doubled score tensors in two tensors of identical shape to the weight tensor, then concatenating prior to the top-half score selection step to generate the mask, one could store the score tensors in a single tensor of doubled width or extend the dimension of the score tensor by one), which we will not address.  One, however, is worth noting: the choice to use as our mask the slice of the double-width mask corresponding to the initial contiguous block of the same shape as the score tensor (i.e. the first half) is entirely arbitrary--using any appropriately sized portion is enough, irrespective of the correspondence.  That is, at the level of representability, the choice of original-coordinate slice is arbitrary: any injection from the original weight coordinates into the augmented score coordinates gives the same family of representable masked tensors. In all experiments, we use the first half of the augmented coordinates for simplicity, but in principle one is generally free to choose this in whatever way is most convenient (which may be of some use if, e.g., the weight tensor is not stored contiguously in physical memory).
    \end{remark}
    
\section{Scaling heuristics}
\label{scaling heuristics appendix}
    
    Equipped now with the ability to obtain strong lottery tickets, we are able to provide empirical support for the logarithmic overparameterization scaling predicted by Pensia et al.~\cite{pensia2020optimal}. As the width of the network increases, the likelihood of successfully recovering near-optimal subnetworks improves markedly, even when the number of parameters grows only modestly. This observation is consistent with the theoretical prediction that relatively mild overparameterization suffices to ensure the existence (and practical recoverability) of high-quality lottery tickets.
    
    \begin{figure}[h]
        \centering
        \includegraphics[width=\linewidth]{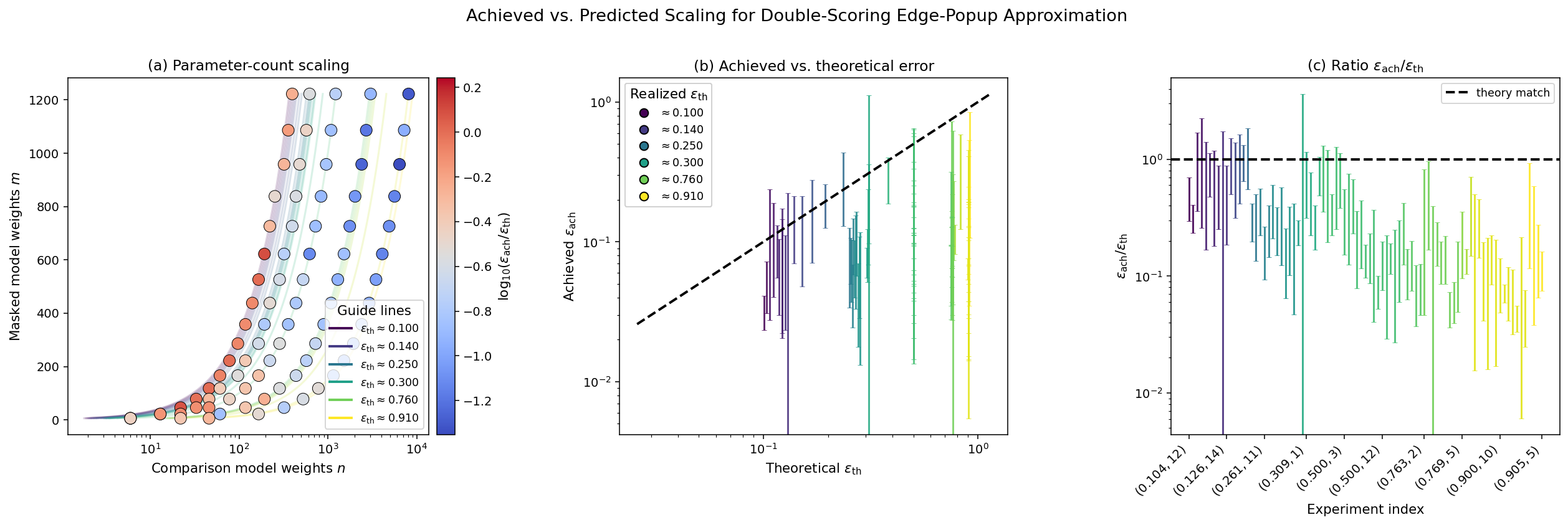}
        \caption{Approximation error trends are broadly consistent with the logarithmic-overparameterization heuristics suggested by \cite{pensia2020optimal} (and, in fact, often more favorable).  We find only very occasional deviation in the small weight regime where initialization noise dominates.}
        \label{fig:scaling}
    \end{figure}

    \Cref{fig:scaling} illustrates the relationship between the theoretical scaling law and the empirical performance of the \texttt{double-scoring} model. The left-most chart plots the parameter counts of the masked model ($m$) against those of the comparison model ($n$), overlaid with the theoretical scaling curves $m = n \log_2(1/\varepsilon)$ for several target errors $\varepsilon$. Each point corresponds to a trained model and is colored according to the logarithmic ratio $\log_{10}\!\left(\frac{\varepsilon_{\mathrm{ach}}}{\varepsilon_{\mathrm{th}}}\right)$,  where $\varepsilon_{\mathrm{th}} = 2^{-m/n}$ is the predicted error and $\varepsilon_{\mathrm{ach}}$ is the observed test error. The concentration of points near or below the theoretical curves indicates that the parameter allocation is correctly bounded by the predicted scaling.  Because the model parameter counts $m$ and $n$ are integers, the realized theoretical values $\varepsilon_{\mathrm{th}} = 2^{-m/n}$ differ slightly from the nominal target values; accordingly, the guide lines in this graph correspond to the realized theoretical $\varepsilon$ values rather than the nominal targets.
    
    The center chart provides a direct comparison between achieved and predicted errors on a log--log scale. Each point represents a single experiment, and the diagonal $y = x$ corresponds to perfect agreement between theory and practice. Across a wide range of configurations, the points lie close to or below this diagonal, demonstrating that the empirical error tracks the theoretical prediction. The deviations are relatively small and do not exhibit a systematic bias, suggesting that the scaling law captures the dominant behavior.
    
    The right-most chart summarizes this agreement by plotting the ratio $\frac{\varepsilon_{\mathrm{ach}}}{\varepsilon_{\mathrm{th}}}$ for each experiment. Values near or below $1$ indicate strong agreement with theory. We observe that this ratio remains approximately bounded by unity across all tested configurations, with only modest variance. Taken together, these results provide strong empirical evidence that the approximation error achieved by the masked model scales as $\varepsilon \approx 2^{-m/n}$, in accordance with the theoretical prediction.

\section{Brute-force optimal subnetwork search and limitations}

    To better understand the behavior of \texttt{double-scoring}, we conduct experiments on small networks where exhaustive search over all possible masks is computationally feasible. This allows us to directly compare the subnetworks found by \texttt{double-scoring} against the true optimal subnetworks.

        \begin{figure}[t]
            \centering
            \includegraphics[width=0.5\linewidth]{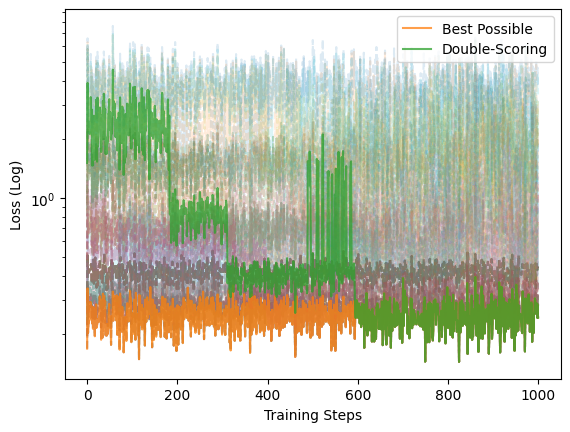}
            \caption{Loss curves for training a network with one hidden layer of width 8 trained to approximate a sine curve using an exhaustive layerwise density search (dashed lines), the \texttt{double-scoring} algorithm (solid orange line), compared to the loss values for the overall best mask found via a brute-force search of all possible masks (solid green line).  In this toy example, the subnetwork produced by \texttt{double-scoring} matches or exceeds the performance of the subnetworks produced by the classical \texttt{edge-popup} algorithm at all layerwise density tuples and nearly reaches the performance of the optimal subnetwork.}
            \label{good but not optimal fig}
        \end{figure}

        In general, \texttt{double-scoring} typically recovers subnetworks whose performance approaches that of the optimal mask in the toy case. However, it does not always recover the optimal mask: as \Cref{good but not optimal fig} illustrates, in some instances, the procedure converges to high-performing subnetworks that are, all the same, suboptimal relative to the global optimum.  This behavior is consistent with the optimization landscape of classical training (as illustrated in \Cref{same loss curves classical vs double-scoring figure}). That is, while the search is performed over score variables rather than weights, the method is subject to the same types of limitations as classical neural network training, including sensitivity to initialization and the presence of local minima. In this sense, the difficulty lies not in the expressivity of the search space (which, as shown in previous sections, is sufficiently rich) but in the optimization dynamics required to navigate it.  Taken together, these results suggest that the primary obstacle in lottery ticket extraction is not representational, but algorithmic: even when the correct subnetwork exists and is accessible within the search space, finding it remains an optimization problem with the same type of inherent difficulties as classical training.  This may be possible to ameliorate by taking into account the sparsity dynamics of the obtained mask (see \Cref{extra capacity section} for further details), though this is purely speculation at this point.

        \begin{figure}[t]
            \centering
            \includegraphics[width=0.5\linewidth]{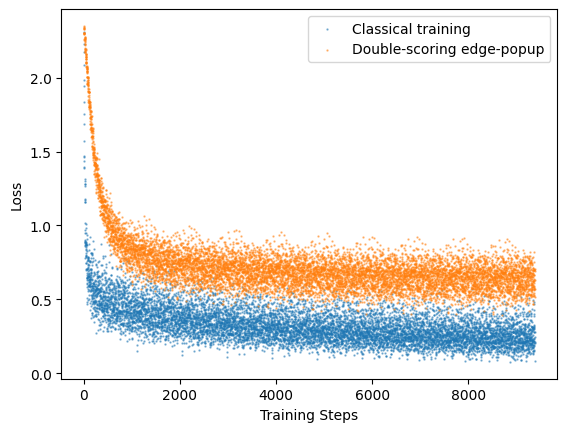}
            \caption{Two identical networks were trained for 20 epochs on the MNIST Fashion data set.  The loss curve during training of with \texttt{double-scoring} (blue) mirrors the loss curve of a network trained classically (orange), albeit with a somewhat slower initial rate of convergence.}
            \label{same loss curves classical vs double-scoring figure}
        \end{figure}
    
\section{Strong tickets as weak tickets}
\label{strong tickets as weak tickets appendix}

    A commonly held view in the lottery ticket literature is that strong and weak lottery tickets represent fundamentally different phenomena. In particular, prior work (e.g., \cite{ramanujan}) suggests that subnetworks which perform well without training need not retain this advantage after training, and vice versa.
    
    In contrast, our experiments indicate that, in the setting considered here, the distinction between strong and weak lottery tickets appears to vanish. Specifically, the subnetworks identified by the proposed method not only achieve high performance without training (as strong tickets), but also train effectively when optimized in isolation, achieving accuracy comparable to the unmasked network at essentially exactly the same rate during training.

    \begin{figure}[t]
        \centering
        \includegraphics[width=\linewidth]{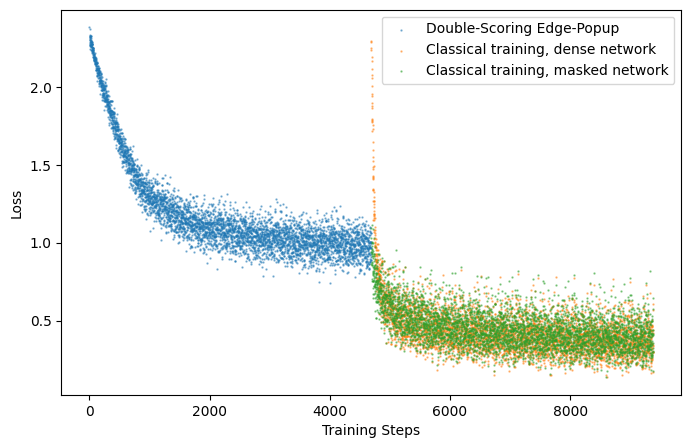}
        \caption{Observationally, it often appears to be the case that the obtained strong lottery tickets are also weak lottery tickets.  Here, the left-most blue curve is the loss curve of a network during double-scored \texttt{edge-popup} training, the right green curve is the masked network trained classically, and the orange curve is an identical unmasked network trained classically.  As is readily apparent, the training dynamics of the masked network and the unmasked network are virtually identical.}
        \label{fig:sltsareweaklts}
    \end{figure}
    
    Empirically, this behavior is consistent with that displayed in \Cref{sparse training and weak tickets subsection}. In particular, subnetworks extracted via \texttt{double-scoring} exhibit both strong initial performance and robust trainability, suggesting that the properties defining strong and weak lottery tickets may, in practice, coincide under appropriate extraction procedures.  That is, the apparent gap between strong and weak lottery tickets may not be intrinsic to the networks themselves, but rather to the methods previously used to identify them.   Furthermore, as evidenced in \Cref{strong ticket extraction subsection} movement pruning appears to produce weak tickets that perform \textit{better} after training than the original frozen network after a similar number of iterations, while \texttt{double-scoring} appears to produce subnetworks that more closely match the original network after training.  While speculative, this suggests that it is possible for strong and weak lottery tickets to coincide, yet deviate from some form of as-of-yet unidentified 'super' weak tickets that movement pruning is able to extract.  We leave an investigation of this possibility for future work, and claim here only that these results support the interpretation that \texttt{double-scoring} substantially narrows the gap between strong and weak tickets in this setting: the extracted subnetworks require little additional optimization to reach their trained sparse accuracy.
    
    These findings do not contradict existing theoretical results, which establish existence but do not characterize the typical behavior of extracted subnetworks. However, they do suggest that the commonly assumed separation between strong and weak lottery tickets may be, at least in part, an artifact of the extraction methods used. By removing the sparsity-selection bottleneck, the proposed approach appears to recover subnetworks that are simultaneously strong and trainable.
    
    One possible explanation for this phenomenon is that enlarging the score space allows the optimization procedure to more effectively explore subnetworks that are both well-aligned with the target function at initialization and stable under subsequent training. In this sense, the method may be implicitly biasing the search toward subnetworks that satisfy both criteria simultaneously.  Whether this is actually the case, let alone why this may occur, is not clear.

\section{Broader impacts}
\label{broader impacts appendix}
    This work is foundational and algorithmic: it proposes a method for extracting sparse subnetworks from randomly initialized neural networks and evaluates it on standard public benchmarks.  Potential positive impacts include improved computational efficiency, reduced training or inference costs, and increased accessibility of sparse neural network methods for researchers with limited compute.  Potential negative impacts are indirect.  As with other improvements in neural network efficiency, better sparse-model extraction could reduce the cost of deploying models in downstream applications, including applications with fairness, privacy, security, or misuse risks.  These risks depend on the downstream system and deployment context rather than on the mask-extraction method itself.

\subsection{Declaration of LLM usage}
    LLM assistance was used during manuscript preparation and experimental-code development, including drafting, editing, formatting, and generating or revising implementation code. LLMs were not used as part of the proposed algorithm, as experimental data, as model components, as baselines, or as evaluation signals. The authors reviewed, modified, debugged, and validated the final code and are responsible for the methodology, experiments, results, and claims.

\section{Speculative extensions}
\label{speculative extensions appendix}

    This appendix reports exploratory experiments that are not used to support the main claims. They are included to document possible directions for future work.  Many are, admittedly, largely toy configurations and proof-of-concept tests; their purpose is not to rigorously and empirically demonstrate an established hypothesis, but to illustrate potential avenues for further study and describe some phenomena that was observed in testing which may be of interest.

    \subsection{Extra capacity in the score space \& sparsity control}\label{extra capacity section}

        A natural question is whether the performance of the proposed method can be further improved by enlarging the score space beyond the doubling used in the \texttt{double-scoring} construction. To investigate this, we consider variants in which each weight tensor is associated with more than two score tensors, thereby increasing the dimensionality of the score space while keeping the weight tensors fixed.
        
        Across a number of toy experiments (\Cref{additional scores do noting}), we observe that this additional capacity tends to yield no similarly meaningful improvement in performance; the primary gain comes from the first doubling. In particular, increasing the number of auxiliary score tensors or their size does not lead to substantially higher test accuracy, faster convergence, or more reliable recovery of high-quality subnetworks. In general, the results tend to be effectively indistinguishable from those obtained with standard \texttt{double-scoring}.
        
        \begin{figure}[t]
            \centering
            \includegraphics[width=0.5\linewidth]{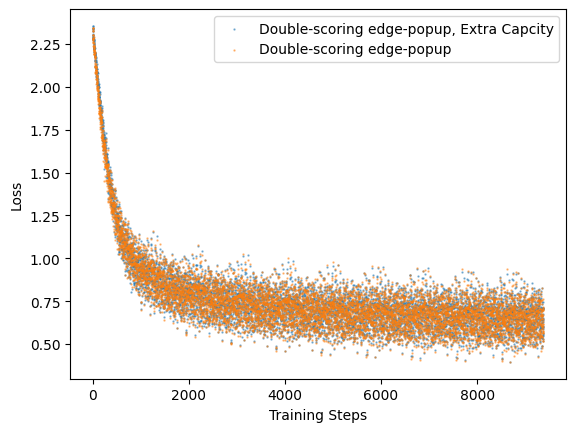}
            \caption{Two identical networks were trained for 20 epochs on the MNIST Fashion dataset with the \texttt{double-scoring}, with one network receiving an a large number of additional trainable score parameters.  As is readily apparent, this produces virtually no change in the training dynamics.}
            \label{additional scores do noting}
        \end{figure}
        
        This behavior suggests that the benefit of \texttt{double-scoring} is not simply due to an increase in the number of trainable parameters, but rather to the specific structural effect of embedding the original network into a half-sparse augmented representation. Once this embedding is achieved, further enlarging the score space appears to provide negligible, if any, returns.  From this perspective, \texttt{double-scoring} may be viewed as a minimal modification that removes the sparsity-selection bottleneck without introducing unnecessary additional capacity.  This further supports the interpretation that the key role of the augmented score space is to remove a structural constraint, rather than to provide additional expressive capacity.
        
        However, it \textit{is not the case} that increasing the score dimension beyond this point yields no further expansion to the set of practically reachable subnetworks.  Certainly, it does not improve their \textit{performance}, but instead it appears to act as a fine-grained control knob on the \textit{sparsity} of the obtained subnetwork.  This makes some sense intuitively:  From a given initialization of scores and a specified $k$-value, the size of the score tensor controls the proportion of selected scores that will be used in a mask.  As a general rule, it appears that setting the $k$-value determines the starting point of the mask's sparsity, while the scale of extra capacity in the score space determines the rate at which the obtained mask tends to deviate in sparsity from the initial $k$-value.  Interestingly, it appears to be the case that, often, the best performing mask seems to occur around a sparsity of $\frac{1}{2}$ and all other masks, given sufficient extra capacity, seem to pull towards this value irrespective of their initial $k$-value, as illustrated in \Cref{fig:tuning the sparsity}. 

        \begin{figure}[t]
            \centering
            \includegraphics[width=\linewidth]{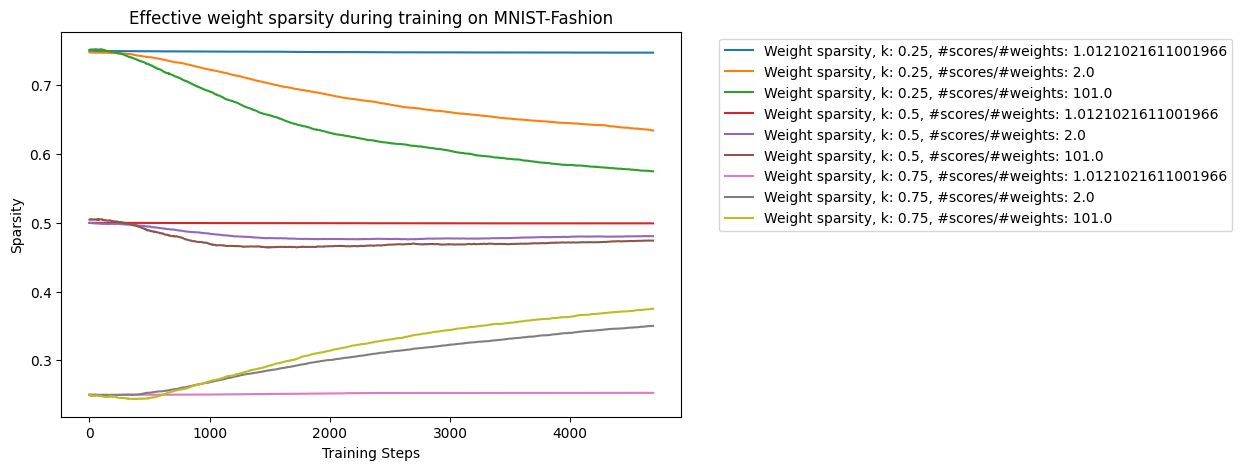}
            \caption{By changing the initial $k$-value and additional capacity in the score tensors, it is possible to fine-tune the sparsity of the mask obtained via \texttt{double-scoring}.  The initial $k$-value determines the sparsity of the mask at initialization, and the ratio of the number of parameters in the score tensors and weight tensors controls the speed of its deviation from this value.  Though it is unclear why, it appears that sufficiently large score tensors tend to eventually produce masks with a sparsity of 0.5, irrespective of the initial $k$-value. }
            \label{fig:tuning the sparsity}
        \end{figure}

    Additionally, it should be noted that it is possible to repeatedly apply \Cref{double-scoring algo}, mask the weights of the network, then repeat this process.  Empirically, this appears to allow  one to rapidly increase the sparsity of the obtained subnetwork (on a scale of $2^{-n}$ initially, where $n$ is the number of repetitions of this process) with only mild degradation to performance for small increases in $n$. This is depicted in \Cref{fig:repeated runs}.  

    \begin{figure}[t]
        \centering
        \includegraphics[width=0.5\linewidth]{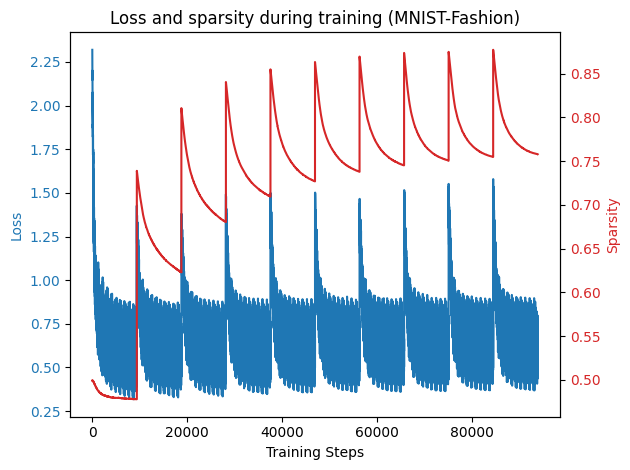}
        \caption{By repeatedly performing \cref{double-scoring algo} and applying the obtained mask to the weights of the model between each repetition, it is possible to achieve high sparsity masks with only a modest effect on the performance of the model.}
        \label{fig:repeated runs}
    \end{figure}

    Finally, we observe that the above suggests that it may be possible to strengthen the strong lottery ticket hypothesis \textit{even further yet still}.  That is, it may not just be the case that all sufficiently large networks contain, with high probability, a subnetwork approximating whatever continuous function one might choose to whatever degree of precision on a compact set they might desire, but they might contain \textit{extraordinarily many} such subnetworks \textit{at virtually all levels of sparsity}.  More compactly, we make the following conjecture:
    
    \begin{conjecture}
        Given any feed-forward network $g$, a compact subset $K$ of the input space, any $\epsilon>0$, and any $\delta \in (0,1)$, there exists $N$ such that a randomly initialized network  with equivalent input and output dimensions and widths $\min \{n_i : 1 \leq i \leq \ell-1
        \} \geq N$ contains, with probability at least $\delta$, a mask $H$ with sparsity $s$ satisfying $\|f_{W \odot H} - g\|_{K,\infty} < \epsilon$ for all realizable sparsities $s \in (0.5-\eta,0.5+\eta)$, where $\eta \rightarrow 0.5$ as $N \rightarrow \infty$.
    \end{conjecture}

    Notice that this conjecture, if true, would provide useful intuition for why \texttt{double-scoring} is effective at uncovering strong lottery tickets. A fixed sparsity constraint may be overly rigid: even when a randomly initialized network contains a strong ticket, that ticket need not occur at the particular sparsity level specified in advance. Instead, the initialization may contain strong subnetworks at \textit{nearby} sparsity levels. Standard \texttt{edge-popup} must search within a single prescribed mask density, whereas \texttt{double-scoring} gives the optimization procedure enough flexibility to move to a nearby realizable sparsity level where a strong ticket is present.


    \subsection{Simultaneous training and masking}

        A natural extension of the proposed method is to combine mask selection and weight training within a single procedure. To this end, we consider a hybrid approach in which both the score tensors (via \texttt{double-scoring}) and the network weights are updated simultaneously during training.
        
        Empirically in toy examples (\Cref{fig: simultaneous masking and training figure}), this approach does not yield meaningful improvements to \textit{performance} over the pure masking procedure, though it does yield a desirable level of sparsity (tunable via \Cref{extra capacity section}) rapidly in a single stage process. In most experiments, the mask stabilizes very early in training, after which the method effectively reduces to standard training of a fixed subnetwork, as shown in \Cref{fig: simultaneous masking and training figure}.
        \begin{figure}[htbp]
            \centering
            \includegraphics[width=\linewidth]{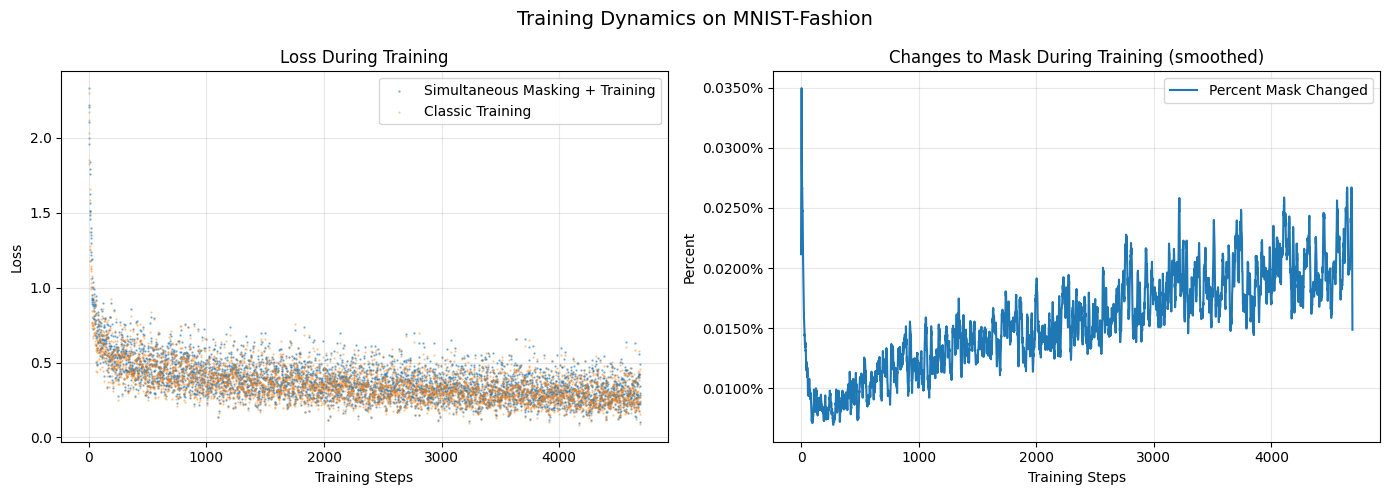}
            \caption{Training with simultaneous masking and weight modification produces no meaningful difference when compared to classical training (save for a fixed proportion parameter reduction), with the mask stabilizing almost immediately, thereby producing--essentially--a classical model of approximately half the number of trainable parameters (for a doubled score space and $k=0.5$).}
            \label{fig: simultaneous masking and training figure}
        \end{figure}
                
        This behavior suggests an inherent imbalance between the two optimization processes. The score updates, driven by the straight-through estimator and the top-$k$ selection mechanism, tend to induce rapid, discrete changes in the mask, while weight updates proceed more gradually. As a result, the mask selection phase effectively terminates before the weights have sufficiently adapted, leading to early commitment to a subnetwork that may be suboptimal.  This suggests that separating mask selection from weight training may be an essential feature of effective lottery ticket extraction, rather than a limitation of existing methods.
        
        We emphasize, however, that this negative result should be interpreted cautiously. The interaction between mask optimization and weight training is delicate, and the simple joint training scheme considered here may not adequately balance the two objectives. It is plausible that alternative strategies (scheduling, regularization, temperature-based relaxations of the masking operation, etc.) could yield improved performance. We leave a more systematic investigation of such approaches to future work.

    \subsection{Lottery tickets in other models}
            In this section, we note that the procedure in \texttt{double-scoring} does not appear to be intrinsically tied to standard feedforward neural networks, but rather exploits a more general structural feature: the representation of model components as linear combinations of basis elements. In the classical setting, a weight matrix can be viewed as a linear combination of elementary matrices, which form an orthonormal basis for the space of linear operators with an appropriate inner product. Masking, in this context, corresponds to an orthogonal projection onto a subspace spanned by a subset of these basis elements. The \texttt{double-scoring} algorithm effectively learns such projections, enabling the extraction of subnetworks that function as strong lottery tickets, and hence also as weak lottery tickets.

            This perspective suggests a natural generalization: any model composed of layers that admit a representation as linear combinations of basis functions interleaved with nonlinearities probably, in principle, admits an analogous masking procedure. To provide preliminary evidence for this claim, we construct an alternative toy model based on compositions of polynomial expansions, where coefficients are taken with respect to standard bases (e.g., monomial, Legendre, or Chebyshev).  To emphasize, this is not a well-performing class of models for most purposes--rather, we use it here purely to expose the generality of lottery ticket phenomenon in a toy model.  Applying the same \texttt{double-scoring} mechanism to mask coefficients, we observe behavior (\Cref{polynomial lottery fig}) qualitatively similar to that seen in neural networks (c.f. \Cref{fig:sltsareweaklts}): the procedure successfully identifies sparse substructures that perform comparably to dense models, and these substructures can be further trained to yield weak lottery tickets.

            \begin{figure}[t]
                \centering
                \includegraphics[width=0.49\linewidth]{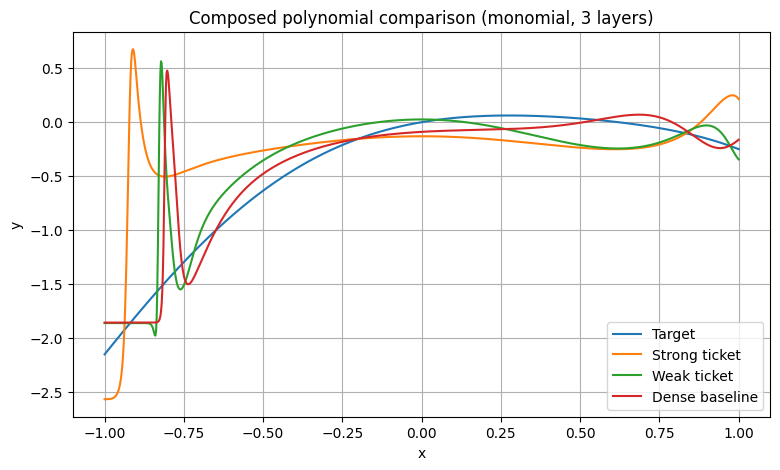}
                \includegraphics[width=0.49\linewidth]{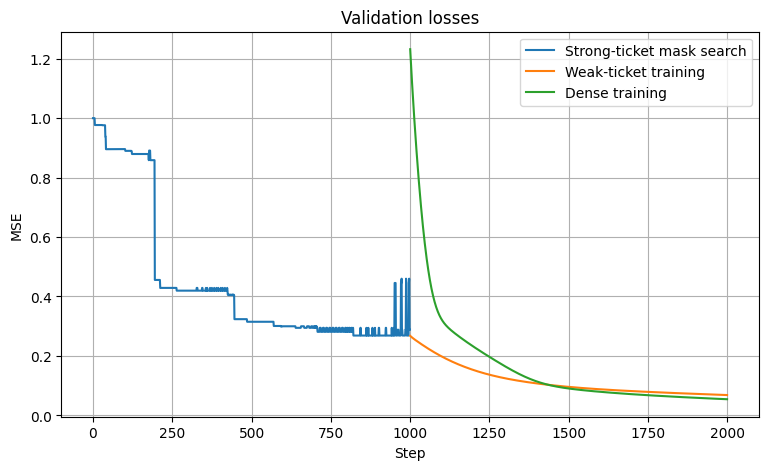}
                \caption{Function approximation using masked polynomial compositions. The \texttt{double-scoring} procedure identifies a sparse subset of basis coefficients (strong lottery ticket) whose induced model closely matches both a retrained weak ticket and a fully dense model, producing effectively the same picture as \Cref{fig:sltsareweaklts}. This demonstrates that the lottery ticket phenomenon, interpreted as projection onto a learned subspace of basis functions, persists beyond standard neural network architectures.}
                \label{polynomial lottery fig}
            \end{figure}
            
            We do not pursue a systematic study of this phenomenon here. However, these results suggest that the effectiveness of double scoring may extend well beyond fully connected architectures, and we conjecture that it applies broadly to models admitting basis decompositions, where masking can be interpreted as projection onto learned subspaces.
            
    \subsection{Training parallelization}\label{training parallelization}

        A distinguishing feature of score-based strong lottery ticket extraction algorithms, like \texttt{edge-popup} and \texttt{double-scoring} is that the underlying network weights remain fixed throughout the mask selection process. All task-specific adaptation is carried out exclusively through the score tensors, which determine the active subnetwork. This separation between a shared, frozen parameterization and task-dependent masking yields the possibility of parallelization in training.  Specifically, a single randomly initialized network may be viewed as a shared computational substrate capable of supporting many distinct tasks simultaneously. Different tasks correspond to different masks, each selecting a subnetwork adapted to a particular objective. Since the weights are never modified, multiple mask optimization procedures can be run in parallel on the same underlying model without interference. The resulting masks can be interpreted as task-specific specializations of a common architecture. Moreover, because sufficiently sparse masks are low-dimensional relative to the full parameter space, they can be efficiently stored, transmitted, and aggregated.
        
        An additional possibility is to combine information across parallel mask searches. For instance, one may aggregate score tensors (or induced masks) obtained from different tasks or random initializations and use them to bias subsequent searches toward consistently useful subnetworks (this is explored further in \Cref{implicit transfer}). In this sense, the score space provides a natural medium for sharing information across distributed optimization processes. This raises the possibility of collaborative or federated mask discovery, in which multiple agents explore the space of subnetworks and exchange information to accelerate convergence.

    \subsection{Implicit transfer learning for multi-task models}\label{implicit transfer}
    
       Biological brains exhibit strong resource efficiency compared to artificial neural networks.  For instance, brains do not globally optimize synapses for every new task but instead reuse pre-configured circuits (e.g., olfactory information as a memory cue).  Using strong lottery tickets and masking the same frozen base model for multiple tasks, we observe some vaugely similar phenomenon.  Loosely speaking, there is an \textit{implicit} analog to a `mixture of experts' model \cite{chen2022towards,gormley2019mixture,nguyen2018practical,cai2024survey,yuksel2012twenty,masoudnia2014mixture} inside every network, wherein the subnetworks are each `experts' at different tasks.  Looking at how these `experts' relate to each other reveals the possibility of an interesting \textit{implicit form of transfer learning}\cite{weiss2016survey,zhuang2020comprehensive,hosna2022transfer} through mask similarity.  In short, we observe that, with a fixed base model, the trained masks for related tasks often have a high degree of overlap in their active parameters.  As a slogan, think: ``\textit{similar tasks, similar masks}".  This mimics the biological case to some extent, reflecting the presence of something like conserved computational motifs within the network.  
       
       To show preliminary evidence of this with a toy model, we use the MedMNIST datasets \cite{yang2021medmnist,yang2023medmnist} and mask a fixed model repeatedly with \Cref{double-scoring algo}.  Looking at the overlap of the obtained strong lottery ticket masks, they appear to have a degree of similarity directly proportional to the human-intuitive degree of similarity between the tasks.  That is, breast cancer diagnosis from ultrasound data and skin lesion classification from dermatascope imaging have a large amount of overlap and involve the utilization of similar parts of the frozen base model, while less similar tasks, like breast cancer diagnosis from ultrasound data and retinal disease classification from retinal optical coherence tomography data, utilize different portions. Maintaining the biological analogy, similarities across masks allow for something like developmental plasticity, wherein organisms repurpose evolutionarily ancient circuits (e.g., basic motion detection) for novel tasks (e.g., recognizing complex gestures).  

            \begin{figure}[t]
                \centering
                \includegraphics[width=0.75\linewidth]{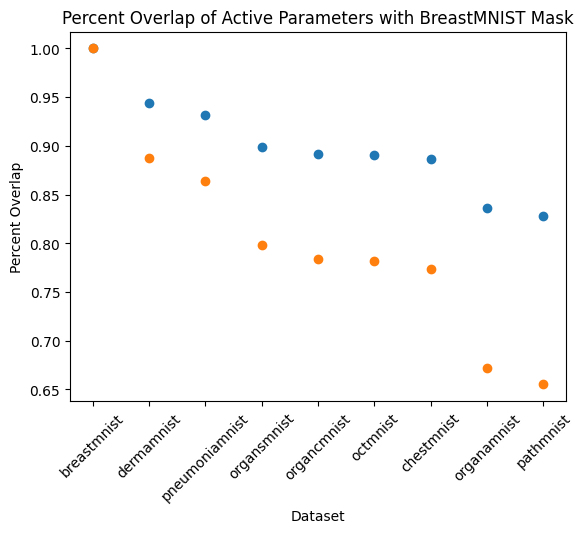}
                \caption{Similar Tasks Have Similar Masks: Well-performing masks are identified in a small feed-forward network using \Cref{double-scoring algo} with $k=0.5$ to perform various binary or multi-label classification tasks using image data in the MEDMNIST dataset \cite{yang2021medmnist,yang2023medmnist}.  The resulting masks are compared for overlap (both in all parameters and in active parameters only) using the mask from the BreastMNIST dataset as a reference, clearly showing that similar tasks (i.e. breast cancer diagnosis from ultrasound data and skin lesion classification from dermatascope imaging) involve the utilization of similar parts of the frozen base model, while less similar tasks (i.e. breast cancer diagnosis from ultrasound data and retinal disease classification from retinal optical coherence tomography data) utilize different portions, as evidenced by their excess overlap beyond the expected 50\% overlap one would expect between two independent random masks with a keep-density of 50\%}
                \label{conserved motif}
            \end{figure}

    \subsection{Dynamic subnetwork prediction and generalization}

        The framework of \texttt{double-scoring} also suggests a natural meta-learning problem. Given a dataset $\mathcal{D}$ and a fixed frozen network, the \texttt{double-scoring} procedure produces a task-dependent mask (or score tensor) adapted to $\mathcal{D}$. Repeating this process across many related tasks therefore yields a collection of input--output pairs of the form
        \[
        \mathcal{D} \longmapsto H_{\mathcal{D}},
        \]
        where $H_{\mathcal{D}}$ denotes a mask extracted from the frozen base model using the data in $\mathcal{D}$.
        
        This makes it possible to train an auxiliary predictor, or \emph{router}, that takes as input a representation of the dataset and outputs a predicted mask (or some lower dimensional representation thereof) directly. In this way, iterative mask optimization may be replaced by a learned one-shot approximation. If the new task is sufficiently similar to those seen during training, one may hope that the predicted mask is already close to a high-performing subnetwork, thereby dramatically reducing or even eliminating the need for task-specific mask search.
        
        Conceptually, this turns mask extraction into a supervised learning problem at the meta-level: the original optimization procedure is used to generate labels, and a second model is trained to imitate its outputs. The resulting mechanism may be viewed as an amortized version of strong lottery ticket extraction, in which experience on previous tasks is distilled into a fast predictor of subnetworks.  This perspective is especially appealing in settings where many related tasks must be solved on a shared frozen architecture. Rather than rerunning \texttt{double-scoring} independently for each new dataset, one may instead learn a mapping from task statistics, support examples, or low-dimensional summaries of the data to the corresponding mask. In principle, this would provide a form of one-shot or few-shot subnetwork selection driven directly by the dataset itself.
        
        In line with the observations of \Cref{implicit transfer}, for a fixed randomly initialized network, the masks obtained via the \texttt{double-scoring} procedure exhibit substantial structural overlap across related tasks. This suggests that the space of effective masks is highly redundant: many masks yield comparable input-output behavior, and the set of “good” masks occupies a low-dimensional manifold within the ambient high-dimensional binary space. To exploit this structure, we show in \Cref{dynamic masking} that this 'mask learning' is possible via an experiment with two-stages. First, we generate a dataset of masks by applying the \texttt{double-scoring} algorithm to a collection of tasks drawn from a common distribution. We then train a mask autoencoder to compress these binary masks into a low-dimensional latent representation and reconstruct them with high fidelity. Empirically, we observe that relatively small latent dimensions suffice to capture the salient structure of the mask distribution, indicating that the effective degrees of freedom are far smaller than the total number of parameters.  In the second stage, we train a \emph{router} network that maps a representation of a task (in our experiments, simple summary statistics of the dataset) directly to the latent code of a corresponding mask. The decoder from the autoencoder is then used to lift this latent code back into a full mask, which is applied to the frozen network. This enables a one-shot prediction of a high-quality subnetwork without performing iterative score optimization. A key feature of this approach is that it avoids directly predicting masks in the original parameter space, which is prohibitively high-dimensional. Instead, the router operates in the learned latent space, where the geometry of the mask distribution is significantly simpler. This not only reduces the complexity of the prediction problem, but also implicitly captures equivalence classes of masks that induce similar functions.

        \begin{figure}[h]
            \centering
            \includegraphics[width=0.9\linewidth]{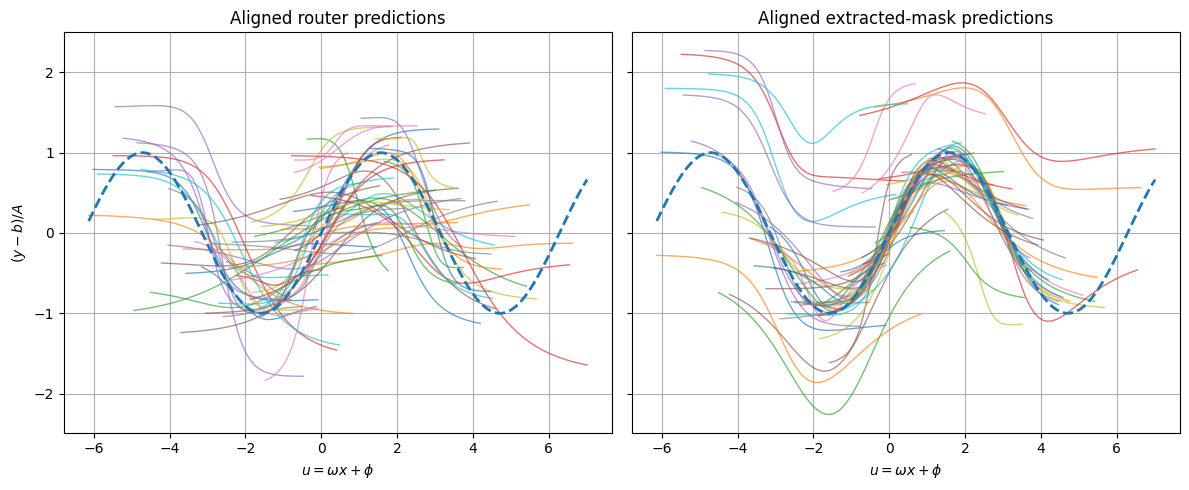}
            \includegraphics[width=0.45\linewidth]{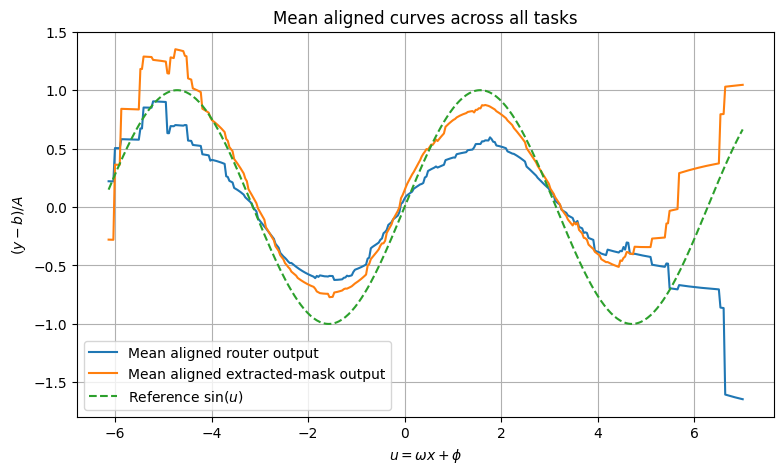}\includegraphics[width=0.45\linewidth]{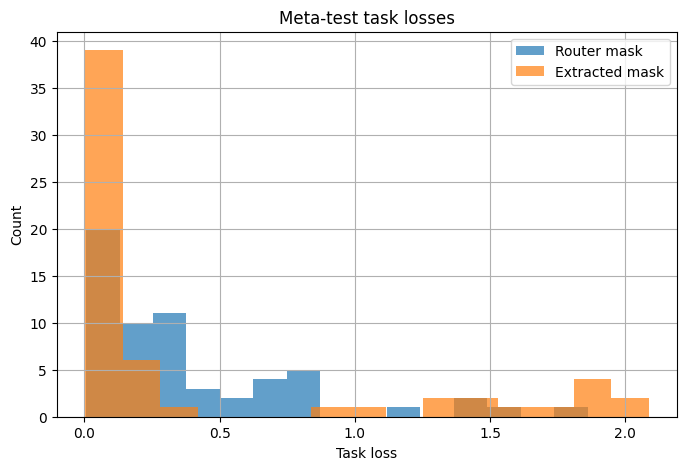}
            \caption{A data set of doubled mask popup scores was generated for a very small test network using \texttt{double-scoring} for a variety of scaled sine curve approximations.  Using a router and the decode head of an autoencoder trained on the data set of scores, a model whose popup scores were predicted using this router was tested on a variety of new scaled sine curve approximations, with the results compared to the performance of masks obtained via \texttt{double-scoring}.  Top:  Aligning the results to a common sine curve, we show the routed model's approximations on the left and the \texttt{double-scoring} approximations on the right.  The approximations are only shown over their native domain after alignment, which occasionally produces some clipping. Bottom Left:  The average over all commonly-aligned approximations for both the routed model and the extracted \texttt{double-scoring} masking model, with interpolation for out-of-domain values in the mean calculation.  Bottom right: A histogram of the routed model's test loss and the \texttt{double-scoring} masking models' test loss.  Throughout, it can be seen that score routers allow for generalization to unseen tasks, albeit at the cost of some degradation to performance.}
            \label{dynamic masking}
        \end{figure}

        In this toy example, we find that the router is able to produce masks whose performance is generally competitive with those obtained via full \texttt{edge-popup} optimization, despite requiring no iterative training at test time. This suggests, as mentioned above, that mask extraction admits a form of amortization: once a sufficient collection of tasks has been processed, the cost of finding a good mask for a new task can be reduced to a single forward pass through the router.  This perspective opens several directions. First, it provides a natural mechanism for transfer learning: knowledge of previously solved tasks is distilled into the latent representation of masks and reused for new tasks. Second, it suggests the possibility of large-scale parallelization, where many agents independently generate masks for related tasks and contribute to a shared latent model. Finally, it raises the question of whether the latent space admits further structure.

    \subsection{Iterated masking and implicit depth}\label{iterated masking}

        We explore an extension of the \texttt{double-scoring} framework motivated by the observation that, for layers with square weight matrices, one can simulate increased depth by repeatedly applying a masked linear operator interleaved with nonlinearities. Concretely, given $W \in \mathbb{R}^{d\times d}$ and masks $H^{(1)},\dots,H^{(T)}$, this yields the composition
        \[
        x \mapsto \sigma\big((W \odot H^{(T)}) \, \sigma\big(\cdots \sigma\big((W \odot H^{(1)}) x\big)\cdots\big)\big),
        \]
        allowing one to realize a depth-$T$ nonlinear network using a single underlying weight tensor.
        
        A naive implementation assigns an independent mask to each virtual layer. While this increases expressivity in principle, it inherits the standard optimization challenges of deep compositions—most notably, the degradation of gradients across repeated nonlinear transformations. In preliminary testing, we find that simply increasing this \emph{implicit depth} yields only modest gains. This suggests that the primary difficulty lies not in depth itself, but in \emph{mask sequence selection}: determining which subnetworks should be applied, and in what order.
        
        This observation motivates a shift from fixed to \emph{adaptive} compositions. Rather than assigning masks to predetermined depths, we introduce a finite \emph{vocabulary} of masked operators
        \[
        W^{(m)} = W \odot H^{(m)}, \quad m = 1,\dots,M,
        \]
        and view each masked transformation as a reusable computational primitive. The learning problem then becomes one of dynamically selecting and composing these primitives (a perspective generally in line with the spirit of the Kolmogorov-Arnold representation theorem, in which functions are represented by sums of compositions of a finite vocabulary of primitives). While one could consider finer-grained vocabularies (e.g., blockwise masks or low-dimensional mask parameterizations), we focus here on full-mask primitives and leave such extensions to future work.
        
        This viewpoint admits a natural analogy to spiking neural networks, where neuronal activity is inherently event-driven: whether a neuron fires depends on its prior activation history and the temporal structure of incoming signals. From this perspective, spiking dynamics can be interpreted as inducing a time-varying mask over the network, with the active subnetwork evolving as a function of past activity. In this sense, spiking systems implicitly perform history-dependent mask selection, suggesting that autoregressive mask composition may provide a useful abstraction for adaptive, input-dependent computation.
        
        \subsubsection{Autoregressive mask selection via a transformer}
        
        We formalize mask selection as an autoregressive sequence modeling problem over this vocabulary. Fix weight and bias mask pairs $(H^{(1)},h^{(1)}),\dots,(H^{(M)},h^{(m)})$ together with a distinguished termination token \texttt{<END>}. A computation is represented by a sequence
        \[
        m_1, m_2, \dots, m_T, \texttt{<END>}, \quad m_t \in \{1,\dots,M\}.
        \]
        At each step, a transformer predicts the next token conditioned on the previously selected masks (and optionally the input). If \texttt{<END>} is emitted, computation terminates; otherwise, selecting $m_i$ updates the hidden state via
        \[
        h \mapsto \sigma\big((W \odot H^{(m_i)})h + b \odot h^{(m_i)}\big).
        \]
        
        This formulation explicitly separates \emph{what} computations are available (the mask vocabulary) from \emph{how} they are composed (the autoregressive policy). Unlike naive iterated masking, the next operation depends on the entire history of prior selections, enabling richer and more flexible control over computation. The resulting architecture can be interpreted as a sequence model over internal computational steps, where each token corresponds to the application of a masked subnetwork.
        
        A key consequence of this formulation is a natural mechanism for \emph{adaptive depth}: computation proceeds until the model emits \texttt{<END>}, allowing termination to depend on the input. More broadly, the model can be viewed as generating a sequence of internal ``thoughts,'' where each step refines the representation via a learned computational primitive. In essence, we create a vocabulary of computational primitives and learn a semantic and syntactic structure of this internal 'language of thought' adapted to the problem at hand.

        We evaluate the proposed transformer-based mask-program model against a classical multilayer perceptron on the \emph{scikit-learn digits} classification task (\Cref{fig:mask transformer fig 1}). To ensure a fair comparison, we match model capacity by setting the depth of the classical network equal to the average realized program length of the learned model.  The program model applies a sequence of masked linear operators drawn from the learned vocabulary, together with a differentiable halting mechanism that determines when to terminate computation. Training dynamics reveal that both models achieve comparable predictive performance; however, the program model exhibits substantially richer internal behavior. In particular, the learned halting mechanism induces a nontrivial distribution over computation lengths, as evidenced by the divergence between the hard (realized) and soft (expected) program lengths and the evolution of the halting probability during training.  Additionally, our model does not seem to suffer from overfitting in the same way as a classical model when overtrained, as evidenced by a comparison of the validation loss.  We believe this may be due to the fact that the model is not only learning the classification task, but also a much more difficult high-dimensional task of navigating in the 'thought space' as it generates programs.

        \begin{figure}[t]
            \centering
            \includegraphics[width=0.49\linewidth]{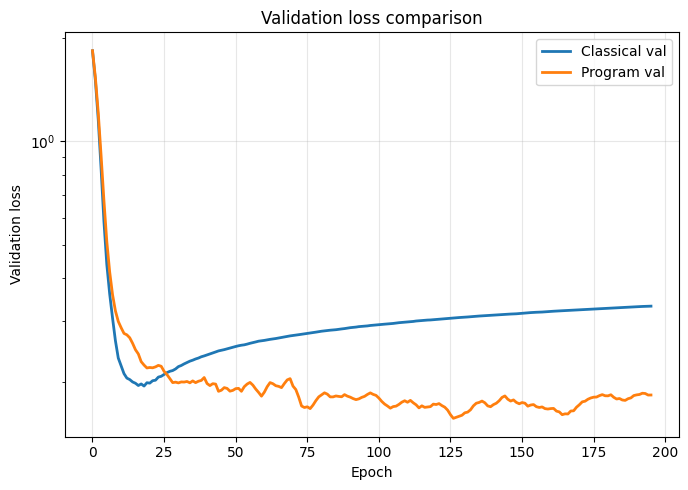}
            \includegraphics[width=0.49\linewidth]{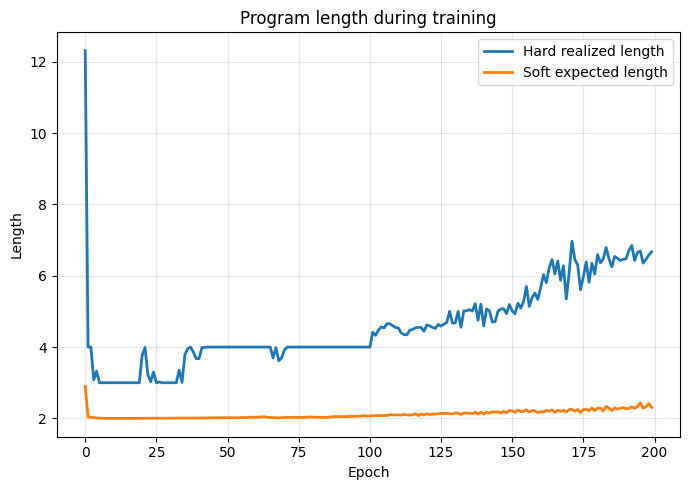}
            \caption{We train a transformer-based policy that generates sequences of mask tokens (“programs”) applied to a fixed randomly initialized network via \texttt{double-scoring}. For each input from the SKLearn digits dataset, the model autoregressively produces a variable-length program which transforms a shared hidden representation before classification.  Unlike a standard model of the same capacity, this approach does not seem to suffer from overfitting (left).  Over time, the length of programs increases (right).}
            \label{fig:mask transformer fig 1}
        \end{figure}
        
        Further diagnostics (\Cref{fig:mask transformer fig 2}) show that the model utilizes a diverse set of computational primitives: the token usage distribution and associated entropy indicate that multiple masks are actively employed, rather than collapsing to a single dominant operation. Finally, the distribution of programs across true class labels highlights that the model allocates computational strategies adaptively across inputs, providing empirical support for interpreting the architecture as a learned, input-dependent computation process rather than a fixed-depth network.
        
        \begin{figure}[t]
            \centering
            \includegraphics[width=0.49\linewidth]{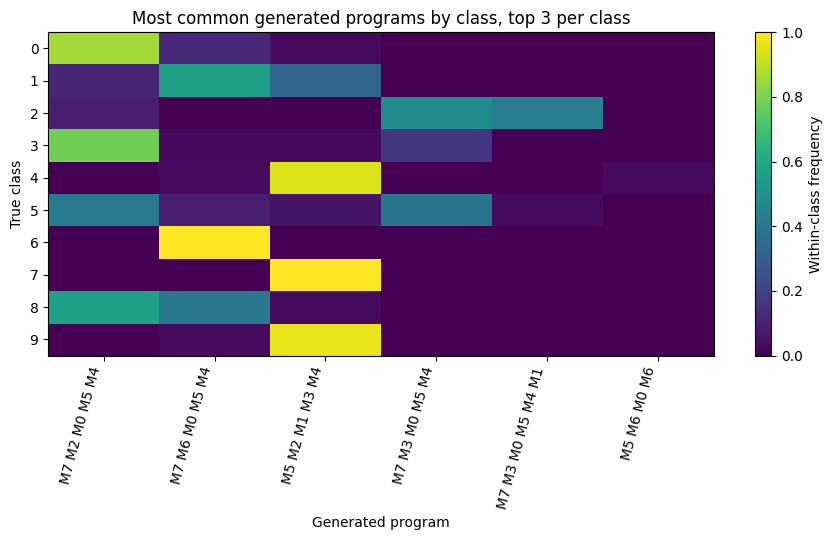}
            \includegraphics[width=0.49\linewidth]{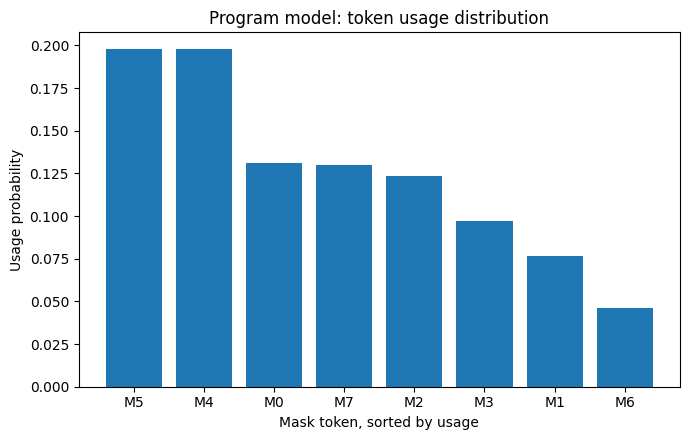}
            \caption{For each true class, the most frequently generated programs (top-$k$ per class) are shown with color indicating the within-class frequency of each program. Distinct classes exhibit concentrated mass on different subsets of programs, indicating that the model learns class-dependent computational routines rather than relying on a single shared transformation. At the same time, partial overlap across classes reveals reuse of subroutines, suggesting that the learned mask vocabulary supports compositional structure across tasks.}
            \label{fig:mask transformer fig 2}
        \end{figure}
    
    \subsection{Learning initialization distributions and neonatal locomotion in ungulates}\label{initialization learning}

        Organisms exhibit innate, complex behaviors (e.g., neonatal locomotion in ungulates) despite genomic limitations. The horse (Equus caballus) genome ($\approx2.7$ Gb diploid) encodes far less information than required to explicitly specify trillions of neural connections. Yet, foals achieve coordinated sensorimotor control within minutes of birth, integrating proprioceptive, cardiovascular, and environmental inputs. This implies evolutionary optimization of developmental rules and not static parameters to generate brain architectures predisposed to critical functions. 
        
        Such meta-learning mechanisms, whereby genetic information biases connection probability distributions, implies an  \textit{evolutionary-developmental paradigm}, wherein biological learning integrates evolutionary optimization of developmental rules with task-specific synaptic refinement. Unlike the training of ANNs, which globally optimize parameters, biological systems employ sparse, context-dependent activation patterns (e.g., localized neural ensembles vs. the epileptiform “seizure-like” global activity of ANNs at inference). Evolutionary processes may select for optimized initialization rules; essentially, for genetic programs that bias neural development toward configurations pre-disposed to critical functions (e.g., breathing, locomotion), thereby producing redundant, task-competent neural components enabling both efficiency and robustness. Translating this principle to machine learning, a framework centered on strong lottery tickets seems automatic: One should learn initialization distributions via evolutionary mechanisms that produce many easy-to-find copies of subnetworks that, at initialization, already encode the task-specific capabilities one desires in the model.

        Having now the ability to extract strong lottery tickets reliably, this framework becomes possible to explore. Specifically, we propose to replace the standard paradigm of fixed, hand-designed initialization schemes with a learned family of initialization distributions, optimized to generate many easily-discoverable subnetworks accomplishing a family of desired tasks.  To this end, we consider a simple parameterization of layerwise initialization distributions. Rather than directly initializing network weights from a fixed distribution (e.g., Gaussian or Kaiming), we define a family of spatially-varying Gaussian fields over each weight matrix. Concretely, for each layer, the mean and log-variance of the weight distribution are parameterized as low-dimensional functions of normalized neuron indices $(i,j)$, yielding position-dependent initialization statistics. These parameters collectively define an \emph{initialization genome} for the network.

        We then place this genome under evolutionary optimization. Each individual in the population corresponds to a distinct initialization distribution. Given such an individual, a network is instantiated by sampling weights from the induced distribution, after which only a small number of iterations of \texttt{double-scoring} are performed on the task. The fitness of the individual is defined as test loss after this short training trajectory. In this way, evolution does not optimize weights directly, but instead selects for initialization distributions that reliably produce subnetworks \texttt{double-scoring} is capable of rapidly discovering for the task.  
        
        Our numerical experiment is designed to probe this phenomenon in a controlled setting. We consider a simple handwritten digit identification task, using the SK-Learn digits dataset and a shallow fully-connected network.  The evolutionary search is conducted over a population of initialization genomes, each encoding layerwise distribution parameters as described above. At each generation, individuals are evaluated by instantiating networks from their corresponding distributions, \Cref{double-scoring algo} is performed for a fixed and small number of steps, and the resulting losses are recorded. Selection, crossover, and mutation are then applied to produce the next generation. Importantly, mutation operates directly on the parameters of the initialization distribution, thereby exploring the space of developmental rules rather than trained weights.  This provides preliminary empirical support for the evolutionary-developmental paradigm.

        \begin{figure}[t]
            \centering
            \includegraphics[width=\linewidth]{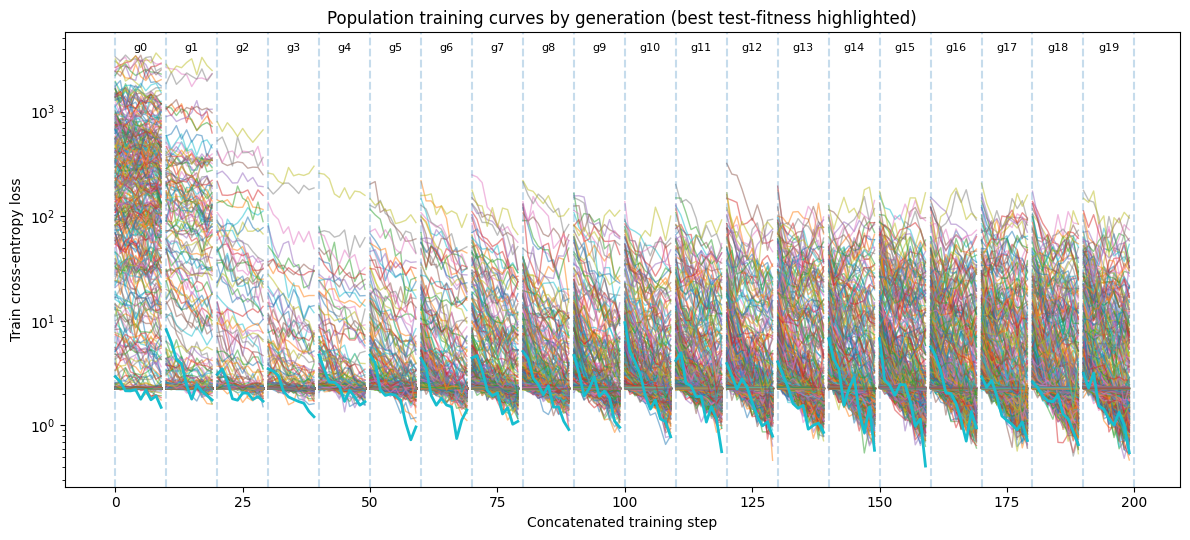}
            \includegraphics[width=0.49\linewidth]{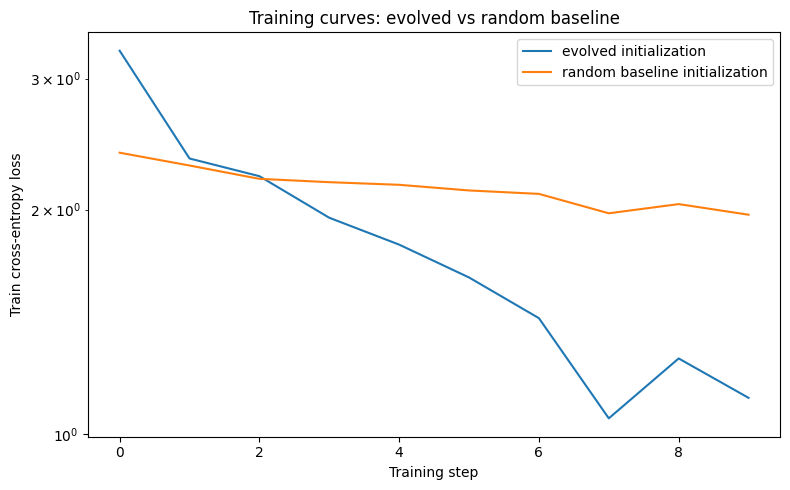}
            \includegraphics[width=0.49\linewidth]{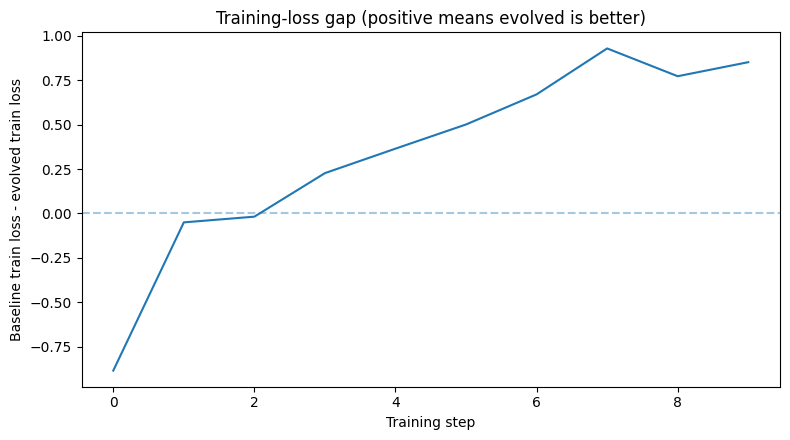}
            \includegraphics[width=\linewidth]{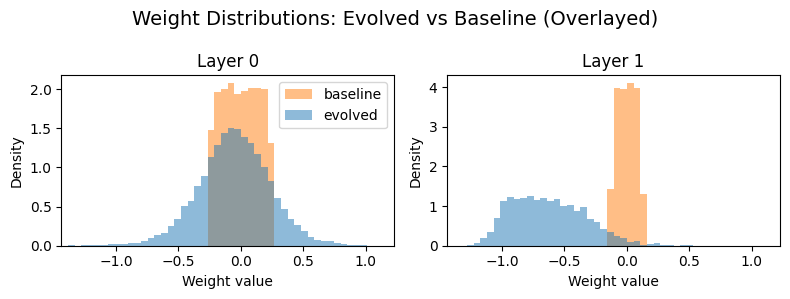}
            \caption{Top: The best-achieved loss over generations and the distribution of training trajectories induced by each initialization genome.  Center:  Training dynamics with \texttt{double-scoring} for an individual initialized from the best genome in the final generation relative to a random baseline.  Bottom:  The qualitative properties of the resulting weight distributions.}
            \label{evolution fig 1}
        \end{figure}
        
        This procedure can be interpreted as a form of zero-shot architectural biasing: rather than learning weights that solve the task, we learn distributions that make solutions \emph{easy to find} via minimal application of \texttt{double-scoring} (\Cref{double-scoring algo}). This procedure effectively shifts learning from weight space to distribution space: we do not learn \textit{solutions}, but rather \textit{laws that make solutions abundant}.

    \subsection{Zero data training with `dreams'}
        
        A central advantage of the \texttt{double-scoring} framework is that it enables a form of data-independent pretraining. In contrast to standard learning paradigms, where performance is tightly coupled to the availability of labeled data, the present approach allows the model to improve its performance without ever observing samples from the target task distribution.

        The key mechanism underlying this phenomenon is what we refer to as dreaming. Rather than training on external data, the model instead samples from its own hypothesis class. Concretely, given a randomly initialized network with frozen weights, we generate behaviors by selecting random subnetworks (via masking) and evaluating them on randomly chosen probe inputs. These input–output pairs form internally generated “dreams,” which are guaranteed to lie within the representational capacity of the model. Crucially, we do not store these subnetworks directly. Instead, we use the \texttt{double-scoring} procedure to relearn masks that approximate the generated behaviors (as the randomly induced masks may not be easily reachable from a random score initialization), thereby constructing a library of subnetworks indexed by their functional behavior on probe sets.
        
        This process yields a collection of input–output examples paired with masks, which can be used to train a routing mechanism. The router learns to map small sets of probe examples to subnetworks capable of reproducing similar behavior. Importantly, this entire pipeline is constructed without reference to any external task distribution. The model is, in effect, learning to organize and index its own functional capabilities.
        
        To evaluate this framework, we consider (\Cref{fig:dreaming}) a toy family of downstream tasks (e.g., parameterized sine functions) that are never observed during the dream-generation or router-training phases. At test time, the model is provided with a small probe set from the target task, and must select an appropriate subnetwork via retrieval or routing. We compare this performance against a baseline consisting of randomly initialized classical neural networks of comparable size, evaluated without task-specific training.

        \begin{figure}[t]
            \centering
            \includegraphics[width=\linewidth]{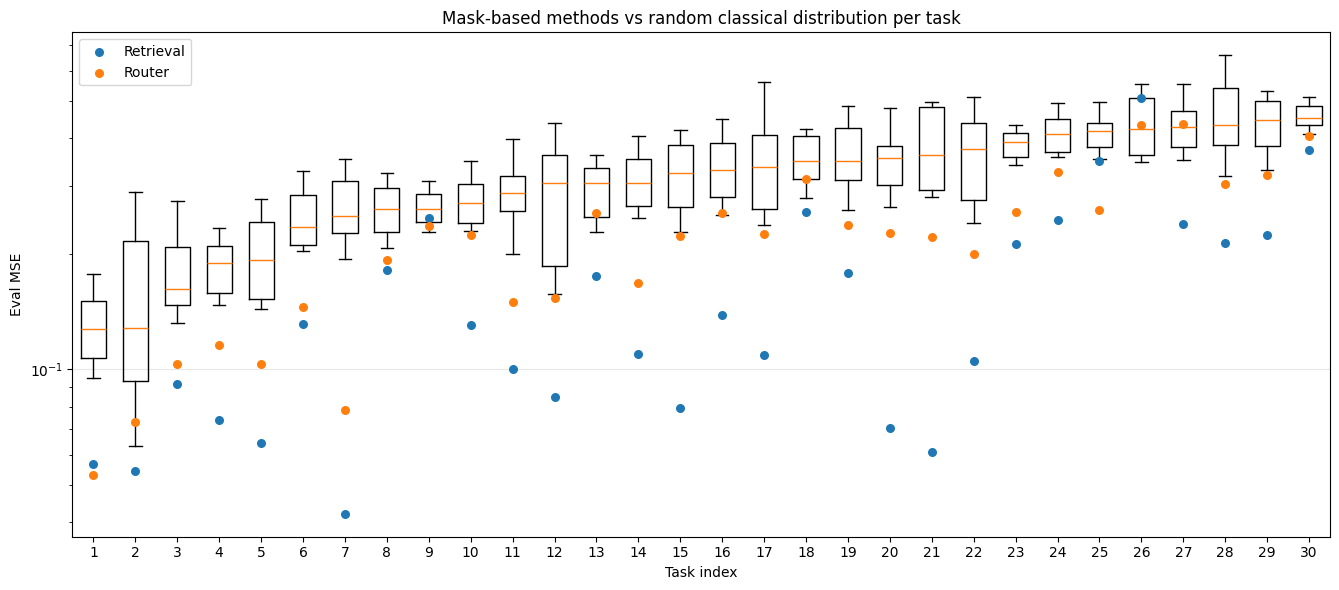}
            \caption{Performance of dream-based subnet selection compared to a random classical baseline. For each task, the boxplot shows the distribution of evaluation errors across randomly initialized dense networks of matching scale, while points denote the performance of subnetworks selected via retrieval and routing from the dream library. In a substantial fraction of tasks, the dream-based methods lie below the bulk of the random baseline distribution, often outperforming even strong random draws. This demonstrates that the model, despite receiving no task-specific training data, is able to systematically select subnetworks with above-random performance. The figure highlights the central phenomenon: dream-based pretraining organizes the model’s internal function space in a way that enables meaningful generalization without exposure to the target task.}
            \label{fig:dreaming}
        \end{figure}
        
        The results consistently show that the dream-based approach outperforms the random baseline across tasks. That is, despite never having been trained on data from the target task (or even from any related distribution) the model achieves lower evaluation error than what one would expect from a randomly sampled network of the same capacity. This demonstrates a clear separation between data and performance that is, to us, genuinely surprising: The model’s ability to solve new tasks is not derived from exposure to those tasks, but rather \textit{from an internal organization of its own representational structure}.  In essence, the model need only know about itself on some level.  We suspect that this would work particularly well in a task-dependent way when combined with the genetic initialization distribution selection process of \Cref{initialization learning}, allowing for the emergence of untrained and effectively `instinctual' capabilities in a model (analogously to how a beaver in captivity will begin to build dam-like structures entirely unprompted and with no prior examples of such behavior), but leave investigation of this for future work.
        
        From a conceptual standpoint, this suggests that the primary role of data in traditional training may be possible to circumvent to a small degree. The \texttt{double-scoring} framework offers this alternative: by learning from a synthetic dataset generated via self-sampling, one can partially bypass some of the need for task-specific data.  The manner in which the gains to performance from such a procedure scale is unclear to the authors.

\subsection{Speculation and outlook}

    The results presented in \Cref{speculative extensions appendix} suggest a perspective that, while highly speculative, may help frame a number of persistent discrepancies between artificial and biological learning systems.
    
    Modern machine learning methods are, at their core, global optimization procedures: parameters are adjusted directly, typically via gradient-based methods, and the computational and energetic cost of training scales accordingly with model size. In contrast, biological systems appear to operate under dramatically different constraints. The human brain, for instance, contains on the order of $10^{14}$ synaptic connections, yet develops and learns under an energy budget that is negligible compared to what would be required to train an artificial system of comparable scale using conventional techniques.  That is, to train a large language model in any familiar architecture with a biologically-comparable number parameters in its feed-forward layers would require an energy budget that--crudely measured--would cost something in the neighborhood of the global GDP while, from birth to high school graduation, the total energy consumption of an average human brain is around what is used by a standard residential air conditioner in a single summer.  
    
    Moreover, biological systems exhibit a degree of robustness that is difficult to reconcile with standard artificial models: localized damage, noise, or loss of connections often results in minimal functional degradation, whereas analogous perturbations in trained neural networks typically lead to significant performance loss.  In fact, many large language models contain single parameters (so-called 'superweights') that, if deleted, render the model unable to output coherent language entirely \cite{yu2024super}.  Biologically, this would be the equivalent of bumping your head or sneezing somewhat forcefully and permanently losing the ability to communicate. 
    
    A further discrepancy arises in the apparent efficiency of biological initialization. Many organisms exhibit complex behaviors with little to no postnatal training. The canonical example is that of a newborn ungulate, which is capable of coordinated locomotion within minutes of birth. It is difficult to attribute such capabilities to a direct encoding of a fully specified control policy in the genome, given the severe informational constraints. Instead, it suggests that what is inherited is not a single model, but rather a structured distribution over possible models, heavily biased toward those that are functionally effective.
    
    Taken together, these observations motivate the following hypothesis: biological learning after birth may rely less on direct parameter optimization and more on search within a highly structured space of pre-configured subnetworks (which, we note, combines both of the scalable approaches detailed in Sutton's \textit{Bitter Lesson} \cite{sutton2019bitter}: Learning is conducted on a population's initialization distribution and search is utilized within an individual after birth). In this view, traditional learning via global parameter optimization is more akin to evolution, which acts as the primary optimizer to shape distributions over connectivity and local structure so that, at birth or after the final development of the brain post-birth, the system already contains a vast number of viable "functional fragments." Learning, then, consists primarily of searching through this space and discovering how to select and activate appropriate subsets of this structure in response to input, rather than constructing functionality from scratch.
    
    From this perspective, several qualitative features of biological systems become more natural. Robustness arises from redundancy: many distinct subnetworks can perform similar functions, so damage to any particular subset has limited effect. Efficiency arises because the search is constrained to a highly favorable region of the space. Rapid acquisition of behavior is possible because the relevant structures are already present, requiring only selection rather than synthesis.

    By contrast, much of contemporary machine learning can be interpreted as optimizing what is, in effect, a “seizure mode” of computation.  That is, biological brains almost never engage anything near their full parameter set simultaneously (barring, as noted, seizures) and prefer highly localized refinements after maturity. While the existing paradigm has proven extraordinarily powerful, it may also represent a particularly inefficient point in the design space.  That is, speaking broadly, what we have done so far in the study of artificial intelligence is focus essentially on the evolutionary side of things by selecting the distribution of parameters for our models with training.  Training with a particular data set is effectively just a roundabout way of accessing a probability distribution on the model's parameters; if one initialized a model according to this distribution directly without training, it would be fundamentally the same.  We propose that this is effectively what evolution does, i.e., evolutionary processes slowly optimize brain development-related parameters that encode a distributions of likely models.  What we have essentially done \textit{in silico} is show that, under the right conditions, an analog of this evolutionary force is enough to make 'brains' whose seizures are capable of performing a wide range of tasks. All the same, optimizing seizures is also probably one of the \textit{hardest possible ways we could have chosen to do things}.  Nature doesn't optimize seizures to perform well on many different tasks; nature is much more fine grained than that in its optimizations.  Biological brains use different chunks for different things.  Nature, essentially, ignores the seizure mode of operation entirely.
    
    The framework explored in \Cref{speculative extensions appendix} can be viewed as a suggestion for an alternative: rather than learning a single set of parameters, we instead study systems in which functionality emerges from selecting among many latent configurations packaged densely together. 


\end{document}